%% file: main.tex
\documentclass[11pt,a4paper,svgnames,dvipsnames]{article}

\usepackage[a4paper, margin=1in]{geometry}

\usepackage{amsmath, amsthm, amssymb}
\usepackage{mathtools}
\usepackage[mathscr]{euscript}
\usepackage{bbm}
\usepackage{enumitem}
\usepackage{geometry}
\usepackage{fancyhdr}
\usepackage[super]{nth}
\usepackage{caption}
\fancypagestyle{firstpage}{
  \fancyhf{}

  \fancyfoot[L]{\footnotesize \textit{Working paper. First released on \nth{26} February 2026.}}
  \fancyhead[R]{\includegraphics[height=1.2cm]{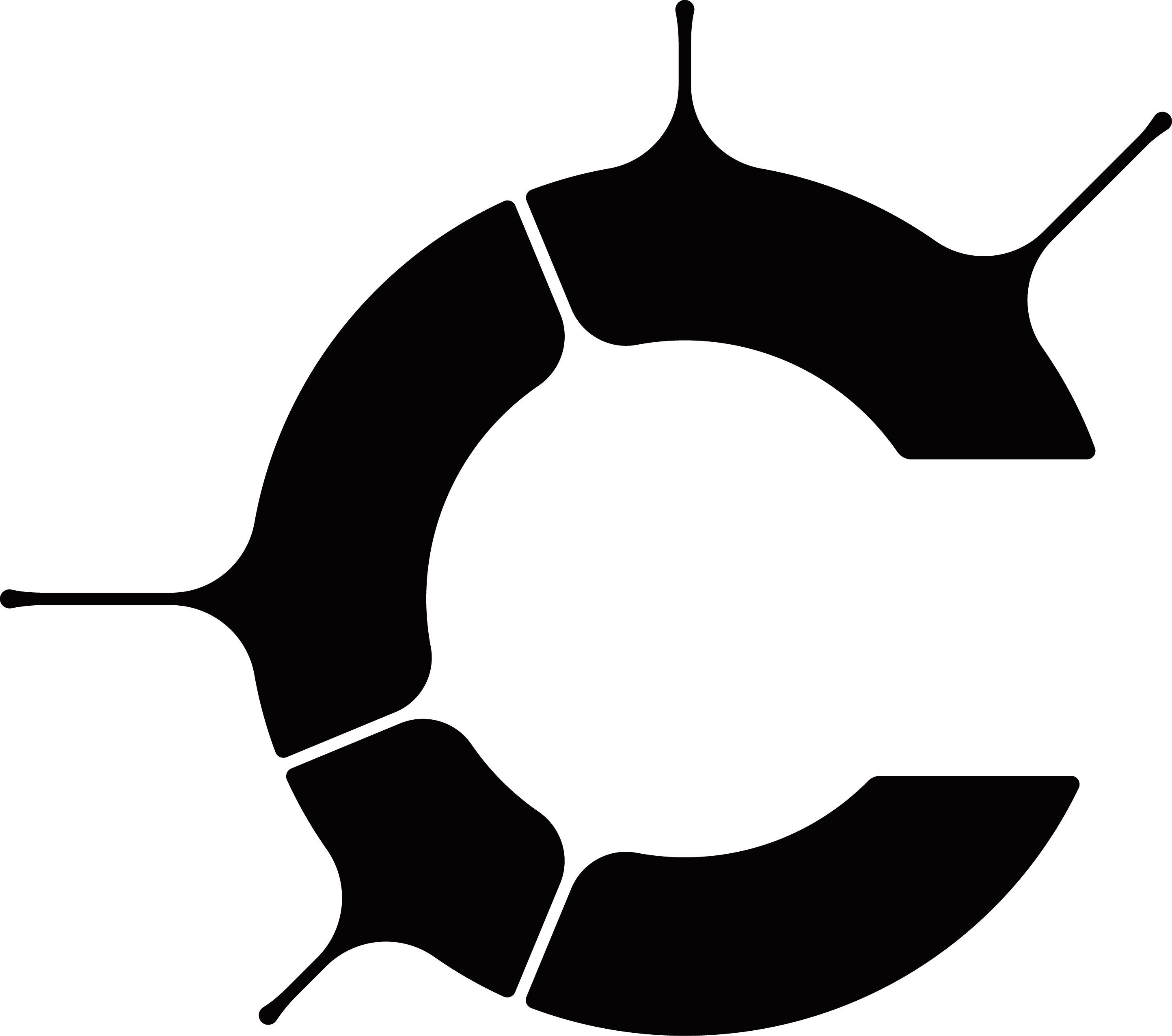}}
}

\usepackage{indentfirst}

\usepackage{graphicx}
\usepackage{tikz}
\usepackage{subcaption}
\usepackage{comment}
\usepackage{threeparttable}
\usepackage{tabularray}

\usetikzlibrary{bayesnet}
\usetikzlibrary{external}

\usepackage{array, booktabs}

\usepackage[colorlinks=true, allcolors=black]{hyperref}
\usepackage{cleveref}  
\Crefname{figure}{Figure}{Figures}
\crefname{figure}{Figure}{Figures}

\usepackage{natbib}
\usepackage{tikz-3dplot}
\usepackage{stmaryrd}
\usepackage[T1]{fontenc}

\usepackage{dsfont}

\usepackage{float}

\theoremstyle{plain}
\newtheorem{theorem}{Theorem}[section]

\newtheorem{proposition}[theorem]{Proposition}

\theoremstyle{definition}
\newtheorem{definition}[theorem]{Definition}
\newtheorem{example}[theorem]{Example}

\theoremstyle{remark}
\newtheorem{remark}[theorem]{Remark}

\newcommand{\TMHP}{\textit{The Principle of Maximum Heterogeneity }}

\newcommand{\R}{\mathbb{R}} 
\newcommand{\N}{\mathbb{N}} 
\newcommand{\Z}{\mathbb{Z}}

\newcommand{\T}{\mathbb{T}}

\newif\ifdraft 
    \drafttrue

\title{\textbf{The Principle of Maximum Heterogeneity} Optimises Productivity in \textit{Distributed Production Systems} Across Biology, Economics, and Computing}

\usepackage{authblk}

\author[1]{Guillhem Artis}
\author[1,2,3]{Danyal Akarca}
\author[1,4]{Jascha Achterberg}

\affil[1]{Callosum, London, United Kingdom (\href{https://www.callosum.com}{www.callosum.com})}
\affil[2]{Imperial College London, London, United Kingdom}
\affil[3]{University of Cambridge, Cambridge, United Kingdom}

\affil[4]{University of Oxford, Oxford, United Kingdom}
\date{}
\begin{document}
\maketitle
\thispagestyle{firstpage}

\begin{abstract}

The world is full of systems of distributed agents, collaborating and competing in complex ways: firms and workers specialise within economies, neurons adapt their tuning across brain circuits, and species compete and coexist within ecosystems. In order to understand the principles governing their dynamics, individual research fields build theories and theoretical models that capture their internal dynamics, explaining how comparative advantage drives trade specialisation, how balanced neural representations emerge from sensory coding, or how biodiversity sustains ecological productivity. Here we propose that many of the well understood findings across fields can be captured in one simple joint cross-disciplinary model, which we call the Distributed Production System. It captures how agent heterogeneity, resource constraints, communication topology, and task structure jointly determine the productivity, efficiency, and robustness of distributed systems across biology, economics, neuroscience, and computing. This cross-disciplinary model reveals that a small set of underlying laws generate the complex dynamics we observe across the distributed production systems of biology, economics, and computing. These can be summarised in our \textit{Principle of Maximum Heterogeneity}: any distributed production system optimising for performance will converge on an increasingly heterogeneous configuration; environmental demands place an upper bound on the degree of heterogeneity required; and the communication topology of the system determines the spatial scale over which heterogeneity spreads, with this principle applying recursively across all layers of nested production systems.

These principles show how any distributed system optimises its production and efficiency, and hence can not only be used to explain already existing systems through a simple joint mechanism, but can also act as a blueprint for how ideal distributed systems ought to be constructed. To demonstrate this, we here use the discovered principles to suggest specific ideal redesigns for compute systems used to execute large scale AI systems. \textit{The Principle of Maximum Heterogeneity} shows a unique convergence of complex phenomena observed across fields onto a simple set of underlying design principles with important predictive value for the ideal design of future distributed production systems.

\end{abstract}
\newpage

{
\setlength{\parskip}{0pt} 
\tableofcontents
}

\newpage
\section{Introduction}
\input{Text/1_introduction}

\newpage
\section{Our model: \textit{The Distributed Production System}}
\input{Text/2_model}
\label{sec:model}

\newpage
\section{Our model captures distributed production systems across biology, economics, and computing}
\label{sec:findings}
\input{Text/3_findings}

\newpage
\section{The principles behind production in any distributed system}
\label{sec:principles}
\input{Text/4_principles}

\newpage
\section{Trying to force homogeneity into distributed production systems}
\label{sec:forcing_homogeneity}
\input{Text/5_forcing}

\newpage
\section{What this means for large scale AI and computing systems}
\label{sec:compute_systems}
\input{Text/6_compute_systems}

\newpage
\section{Discussion}
\label{sec:discussion}
\input{Text/7_discussion}

\section{Acknowledgements}

We thank Lukas B. Freund, Fleur Zeldenrust, Jack Gartside, and Jannik Reichert for helpful discussion about heterogeneity during the early stages of this project. We thank Dan Goodman for ongoing discussions and advice.

\newpage 
\section*{Appendix}
\addcontentsline{toc}{section}{Appendix}
\appendix

\section{Heterogeneity measure}
\label{app:heteroneity_metric}
\label{app:heterogeneity_metric}
\input{Text/heterogeneity_measure}

\section{Supplementary details on model definition}
\input{Text/model_design_discussion}

\section{Technical verifications on the optimisation procedure}
\input{Text/technical_verifications}
\label{app:technical_verifications}

\section{Parameters used for simulations}
\label{app:experiment_parameters}
\input{Text/experiment_parameters}

\section{Networks effects}
\label{app:network_effect}
\input{Text/network_effect}

\section{Supplementary analyses}
\label{app:supplementary_findings}
\label{app:supplementary_principles}
\input{Text/supplementary_principles}
\label{app:supplementary_compute}

\newpage
\addcontentsline{toc}{section}{References}
\bibliography{references}

\vfill
\noindent
\begin{tikzpicture}
  \draw[line width=1.08pt] (-0.5pt,0) -- (\textwidth+0.5pt,0);
  \fill (0.31pt, 0.54pt)
    .. controls (-0.81pt, 0.54pt) and (-1.92pt, 0.69pt) .. (-2.92pt, 1.15pt)
    .. controls (-3.08pt, 1.23pt) and (-3.23pt, 1.27pt) .. (-3.38pt, 1.27pt)
    .. controls (-4.08pt, 1.31pt) and (-4.73pt, 0.73pt) .. (-4.77pt, 0.00pt)
    .. controls (-4.77pt, -0.73pt) and (-4.19pt, -1.31pt) .. (-3.50pt, -1.31pt)
    .. controls (-2.81pt, -1.31pt) and (-3.15pt, -1.31pt) .. (-3.00pt, -1.19pt)
    .. controls (-2.08pt, -0.81pt) and (-1.08pt, -0.54pt) .. (-0.08pt, -0.54pt)
    -- cycle;
  \fill (\textwidth-0.31pt, 0.54pt)
    .. controls (\textwidth+0.81pt, 0.54pt) and (\textwidth+1.92pt, 0.69pt) .. (\textwidth+2.92pt, 1.15pt)
    .. controls (\textwidth+3.08pt, 1.23pt) and (\textwidth+3.23pt, 1.27pt) .. (\textwidth+3.38pt, 1.27pt)
    .. controls (\textwidth+4.08pt, 1.31pt) and (\textwidth+4.73pt, 0.73pt) .. (\textwidth+4.77pt, 0.00pt)
    .. controls (\textwidth+4.77pt, -0.73pt) and (\textwidth+4.19pt, -1.31pt) .. (\textwidth+3.50pt, -1.31pt)
    .. controls (\textwidth+2.81pt, -1.31pt) and (\textwidth+3.15pt, -1.31pt) .. (\textwidth+3.00pt, -1.19pt)
    .. controls (\textwidth+2.08pt, -0.81pt) and (\textwidth+1.08pt, -0.54pt) .. (\textwidth+0.08pt, -0.54pt)
    -- cycle;
\end{tikzpicture}

\end{document}

%% file: Text/1_introduction.tex
The world around us is full of systems, consisting of vast numbers of interacting agents. As individuals collaborate and compete, their unique behaviours jointly integrate into the systems overall characteristics. In ecology individuals of different species aggregate to inter-species competition. In neuroscience individual cells adapt their tuning to sensory stimuli to jointly show balanced responses to perceived stimuli. In economics, workers and countries specialise to have advantageous positions in trade relationships. In all of these cases well crafted theoretical models have been successful in deciphering the underlying principles that generate the stunningly complex systems level behaviour.

The most impactful theoretical work does not only reveal principles behind the complex systems level behaviour of individual subject areas, but shows how the same principles are replicated across many disciplines. A prominent example for the case of complex interaction in large distributed systems likely is Game Theory \citep{vonneumann1944, nash1951}. Game Theory has been impactful not only modelling microeconomic behaviours \citep{kagel1986winner} but also have generalised well to biology \cite{bshary2006image}, ecology \cite{giraldeau1994test}, and neuroscience \citep{fakhar2025human}. Other examples of theories with cross disciplinary impact are Information Theory \citep{shannon1948mathematical}, the $x^{\frac{3}{4}}$ scaling law \citep{west1997general}, and most recently peak selection with continuous attractor dynamics \citep{Khona2025}. 

All of the above examples have revealed important principles of how systems of distributed agents interact, however they also not yet capture key properties that have been observed as important across recent work. Critical examples are \textit{(1)} to capture the characteristics and skills of individual actors to allow those to impact how agents interact \citep{Freund2025, cook2025brainlike, Eaton2012, dahmen2025heterogeneity, Ito2026}, \textit{(2)} to capture the resource constrains \cite{achterberg2023spatially, gershman2015computational, kool2018mental, sun2025exploiting}, \textit{(3)} the topology of communication and interaction \citep{achterberg2023spatially, achterberg2023building, akarca2021generative, carvalho2019production}, and \textit{(4)} how joint and individual production and task goals influence the interaction of \textit{(1)}, \textit{(2)}, and \textit{(3)} \citep{achterberg2023spatially, safran2017depth, ellefsen2015neural, caliendo2012impact}.

Here we built a new theoretical model to study the general principles underlying the dynamics of distributed systems, taking into account the four additional forces discussed above, while also lending core mechanisms from the already established theoretical frameworks. We call the new integrated, but still tractable model, \textit{The Distributed Production System}. Extensive experiments on the model show that it captures a vast set of findings across distributed production systems observed in economic, biological, neural, and computing systems, validating it as a cross-discipline model of distributed production systems. We then study which laws underlie the model's behaviour and find a striking tendency towards heterogeneity as the optimal system's level solution in terms of productivity, efficiency, and robustness. We characterise this drive towards heterogeneity and its dependency with the system's workload and topology in the \textit{Maximum Heterogeneity Principle}.

The \textit{Maximum Heterogeneity Principle} represents not only a new unifying mechanism for how large systems of distributed agents can jointly optimise their production but also promises to act as a coherent design blueprint for how to construct new optimal large distributed production systems. We apply the discovered principle to the case of large-scale computing systems in the context of AI systems to outline how current systems follow a suboptimal scaling paradigm and how they could be built differently in order to achieve more efficient and scalable infrastructure.

%% file: Text/2_model.tex
On a high level we want to use a model to capture how distributed agents with individual characteristics optimally act and interact to optimise for their own or their joint progress towards a goal. As we want to use this model to identify which universal principles drive systems-level behaviour across scientific disciplines, the formulation of the model relies on simple and generalisable description of all components. In the following we will first give the reader an intuitive overview over the model (\cref{sec:model_intuitive}). Then we will describe in more mathematical detail how we characterise agents (\cref{sec:model_agents}), their interaction (\cref{sec:model_networks}), how they individually and jointly produce outputs, how they adapt their outputs to demands of the environment and additional constraints, and lastly which measures we use the study the behaviour of our model system.  

\subsection{Intuitive summary of the model}
\label{sec:model_intuitive}

The model represents a distributed production system as a set of agents, each carrying an individual skill function that describes how much of each type of operation the agent can perform (\cref{fig:model_intuitive}, lower half). Agents are connected via an interaction network $Q_{(1)}$, which governs how they communicate and collaborate. The joint production function of the system emerges from this combination: as shown in \cref{fig:model_intuitive}, the individual skill densities of agents connected via the collaboration network integrate into a collective output, $W_{(1)}$, whose shape reflects both the individual agents' skills and how they are linked.

Whether this collective output is sufficient depends on the demand placed on the system. As illustrated in \cref{fig:model_intuitive} (upper right panel), the same production function can be well-suited to one demand and poorly-suited to another. The model formalises this as a coverage problem: a demand is met when the joint production function covers it across the full space of operations. For any given demand shape, we can recover the optimal configuration of agent skills and network connections via gradient descent and then study the properties of that solution. Many of the experiments in following section focus on studying under which conditions empirically observed system level architecture develop (\cref{sec:findings}) or what optimal systems level architecture look like (\cref{sec:principles}). While this intuitive description focuses on a one-layer setup with a demand at the top, captured by the production of a set of agents, our model extends naturally to hierarchical systems, where multiple layers of interdependent agents sit below the demand, each layer producing the one above it, with resource constraints acting as additional production requirements at the bottom of the hierarchy. The remainder of the methods section formalises both the single-layer and multi-layer cases in detail.

\begin{figure}[H]
    \centering
    \includegraphics[width=0.7\textwidth]{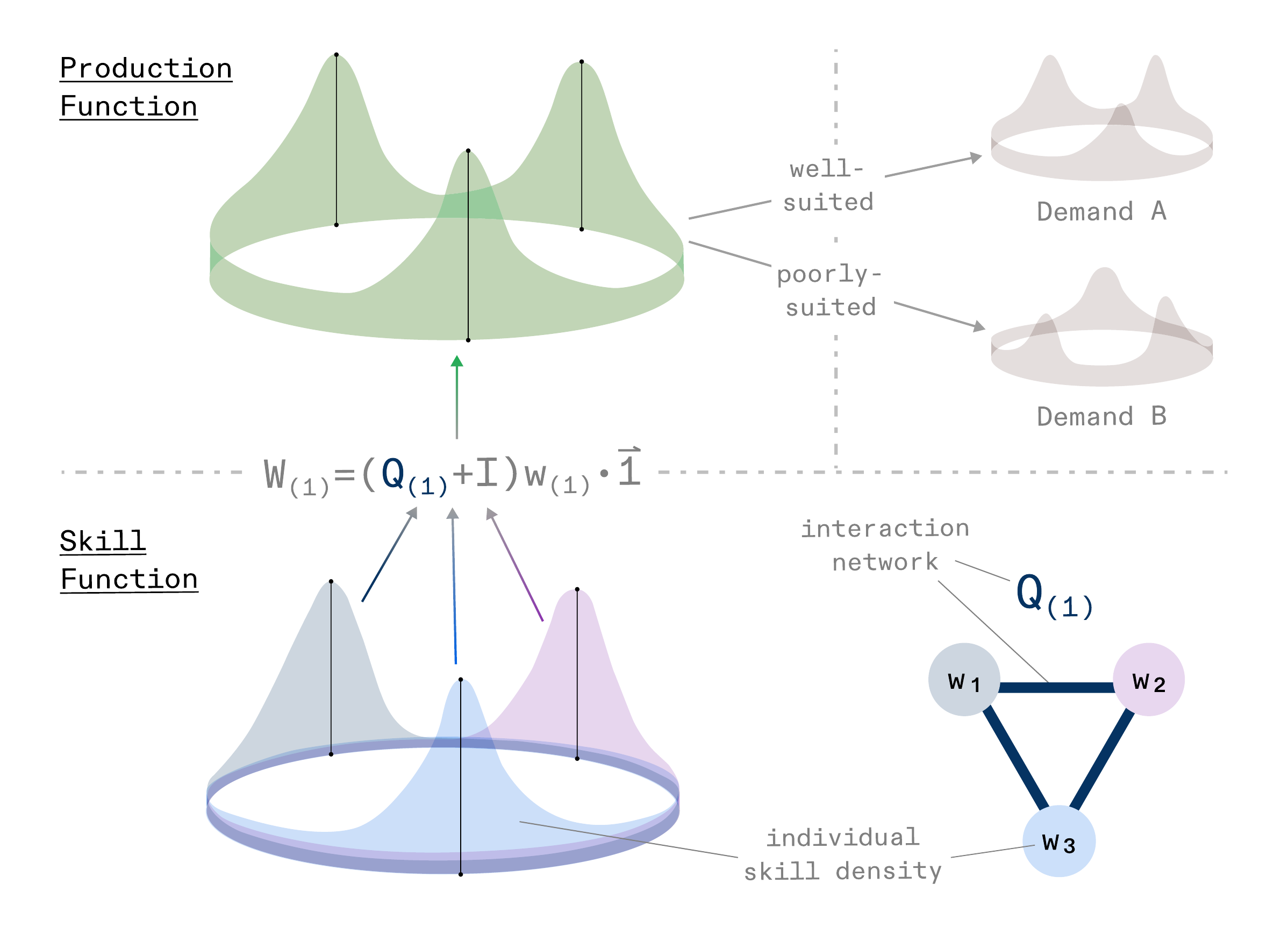}
    \caption{\textbf{Intuitive visualisation of the distributed production system model in which each agent contributes an individual skill density and connected agents jointly produce an output that can match or fail to match a given demand.} The lower half of the figure shows three agents ($w_1$, $w_2$, $w_3$), each characterised by an individual skill density concentrated around a specialism. The interaction network $Q_{(1)}$ encodes which agents communicate. Together, skill densities and network structure determine the joint production function shown in the upper half, computed as $W_{(1)} = (Q_{(1)} + \mathbf{I})\, w_{(1)} \cdot \vec{1}$ where $\vec{1}$ denotes the vector full of ones. The right panel illustrates that the same production function can be well-suited to one demand shape (Demand A) and poorly-suited to another (Demand B), motivating the optimisation problem at the core of the model.}
    \label{fig:model_intuitive}
\end{figure}

\subsection{Parametrisation of agents}
\label{sec:model_agents}

As discussed in \textit{(1)} in the introduction, we want agents to have individual characteristics that can influence their interaction. Inspired by prior economic models \cite{Freund2025}, we characterize each agent through a density of their skills and abilities. We represent them as wrapped Gaussian densities over the space of skills that is modelled by a one-dimensional torus (or circle) $\mathbb{T_{\text{skills}} }$. This means that agents de facto have a core talent, the mean of that density distribution, and can spread out their skills to individual degrees around that mean. More complex shapes of talents can be considered but we opt for this simple shape of skills to allow for more detailed analyses of resulting model setups (\ref{sec:model_metrics}). Note that we more specifically discuss why this is a good representation for most agent classes, even if they are universal function approximators in \cref{app:ufa}. Following this, an agent $i$, could be represented by the mean $\mu_i$ and standard deviation $\sigma_i$ of its skills distribution $s_i$,
$$s_i(\mu_i, \sigma_i) : \theta \in \mathbb{T}_{\text{skills}} \mapsto \frac{1}{\sigma_i\sqrt{2\pi}} \sum_{k \in \Z} \exp \left(\frac{-(\theta - \mu_i + 2k \pi)^2}{2 \sigma_i^2} \right).$$
The talent / production of one agent corresponds to the integral of its skills function over the entire space of skills.

\subsection{Parametrisation of agents' pairwise interactions}
\label{sec:model_networks}
The production system is a distributed system that achieves production via distributing the task across the agents that jointly interact to meet a given demand.
In most distributed systems, the interaction between all agents is not uniform. \textit{(2)} in the introduction references how the topology of networks in which agents lie creates varying levels of interaction between agents. To capture these interactions, we connect the agents with a network structure, which takes the form of a graph where the vertices are the agents and the edges represent the interactions. The graph structure is contained in its adjacency matrix that we denote as $Q$. For simplicity, we assume that the graph is undirected, i.e. $Q$ is a symmetric matrix. Integrating this structure with the agents definition and we get a representation of a simple production system that focuses on the effects of different skill sets linked together via network interaction. For $N$-agents the formalism is:
$$ ((\mu_i, \sigma_i)_{1 \leq i \leq N}, Q) \overset{\mathrm{not.}}{=} ((\mu, \sigma), Q),$$
with $\mu_i \in \T$, $\sigma_i \in \R_+^*$ and $Q \in \R^{N \times N}.$  We introduce the set $\mathscr{C}$ which contains all the maximally (in term of size) connected components of the graph. Every element $c \in \mathscr{C}$ corresponds to the set of all agents connected in this component. We note $(\mu_c, \sigma_c)$ the vector containing all the agents' means and standard deviations within $c$. Agents can only interact with agents within their connected component. We note that the network structure implies the agents are unique. One should think of the agents as associated to an identifier, and together they form a network. We have modelled the interactions as a network, and so, alongside the agents parametrisation, we have now created a distributed system through a network of agents.

\subsection{Agents' individual and joint production output}
With the distributed system formalism, the production is represented as a function of both the skills of agents and their pairwise interactions. To keep the model simple, we characterise outputs of the function in an operation space $\T_{\text{ops}}$ that is similar to the skills space. To be mathematically precise, similar means here equality in terms of set, but not necessarily on the nature of the points forming the set (this is expanded on in \cref{app:between_layer_mapping}). Formally we have the following mapping, 
$$\T_{\text{skills}} \xrightarrow{\quad \mathrm{id} \quad} \T_{\text{ops}}.$$
Where no ambiguity arises, we write $\T_{\text{skills}} = \T_{\text{ops}} \overset{\mathrm{not.}}{=} \mathbb{T}$.
Before defining the joint production function of the distributed system, we introduce the individual production function. The individual production of a single agent $i$ corresponds to the output of agent $i$ considered in isolation. This output corresponds to the agent skills density mapped onto the operation space, namely:  
$$
    s_i \in \mathscr{F}\left(\T_{\text{skills}}\right) \xrightarrow{\quad \mathrm{id} \quad} w_i \in \mathscr{F}\left(\T_{\text{ops}}\right),
$$
where $\mathscr{F}\left(\T_{\text{skills}}\right)$ and $\mathscr{F}\left(\T_{\text{ops}}\right)$ are respectively the set of real functions on the operations space and on the skills space. With the mapping in mind, the reader can intuitively think of
$$w_i = s_i, \quad \text{ with } \: w_i,s_i\in \mathscr{F}\left(\T\right).$$
The individual production of an agent is therefore a workload density over the operations space. 
The joint production function of the distributed system is based on this mechanism but also takes into account the network structure. Let $S:= S((\mu, \sigma), Q)$ be a distributed system composed of $N \in \N$ agents, $s = (s_1, \dots, s_N)^T$ be the vector of all individual agents' skill densities, and $w = (w_1, \dots, w_N)^T$ be the vector of all individual agents' production densities. The production function of $S$ in its simplest form is:
$$W(s, Q) = \mathds{1}^T(I+Q)w,$$
where $I$ is the identity matrix and $\mathds{1}$ the $N$-dimensional vector containing only ones. 
\begin{remark}
    In its simplest form, the production function is linear with respect to the vector of agents' skills. Consequently, the production is a weighted sum of the individual production functions and the interaction matrix does not have an interaction effect on the production function itself here. The interaction acts on the optimisation (c.f.\ \cref{sec:model_optimisation}) and thus affects the optimal skill densities.
\end{remark}

\begin{example} We give some examples of simple production outputs: 
    \begin{itemize}
\item \textbf{No collaboration.} A `no collaboration' distributed system (i.e.\ $Q = 0$) has a production function equal to the sum of the individual agents' production functions,
        $$W(s, 0) = \mathds{1}^Tw = \sum_{1 \leq i \leq N} w_i.$$
\item \textbf{Uniform interactions with unit intensity.}
Consider the case in which all agents interact symmetrically and with equal intensity, so that 
\[
Q = \mathbf{1}_{N\times N} - I,
\]
where $\mathbf{1}_{N\times N}$ denotes the $N\times N$ matrix of ones and $I$ is the identity matrix.  
In this setting, the aggregate production function becomes
\[
W(s,Q) 
= N \sum_{i=1}^N w_i.
\]
Hence, total production equals $N$ times the aggregate stand-alone production.  
The interaction structure generates an additional contribution of 
\[
(N-1)\sum_{i=1}^N w_i,
\]
which reflects the fully connected nature of the interaction network.
\item \textbf{Uniform interactions with intensity $\lambda$.}  
Suppose now that interactions remain symmetric but occur with uniform intensity $\lambda \in \mathbb{R}$, so that
\[
Q = \lambda (\mathbf{1}_{N\times N} - I).
\]
The production function can then be written as
\[
W(s,Q) 
= \sum_{i=1}^N \left( w_i + \lambda \sum_{j\neq i} w_j \right).
\]
In this case, each agent's production receives a homogeneous ``mean-field'' contribution proportional to the aggregate production of all other agents. The parameter $\lambda$ measures the strength (and possibly the sign) of the spillover effects.
\item \textbf{Non-uniform interaction structure.} 
In the general case of non-uniform interactions, where $Q = (Q_{ij})$ is an arbitrary interaction matrix with zero diagonal, total production is given by
\[
W(s,Q) 
= \sum_{i=1}^N \left( w_i + \sum_{j\neq i} Q_{ij} w_j \right).
\]
Thus, each agent's production is augmented by a weighted sum of the outputs of the other agents, with weights determined by the interaction matrix $Q$. This specification allows for asymmetric collaboration patterns.
\item \textbf{Circular topology interaction.} Throughout this work, we extensively use the circular topology and the corresponding interaction matrix, 
$$Q = J_N = A_N + A_N^T,$$
where $A_N$ denotes the $N$-cycle matrix, 
$$A_N = \begin{pmatrix} 
    0 & 1 & 0 & \cdots & 0 \\
    0 & 0 & 1 & \ddots & \vdots \\
    \vdots & \vdots & \ddots & \ddots & 0 \\
    0 & 0 & \cdots & 0 & 1 \\
    1 & 0 & \cdots & 0 & 0
\end{pmatrix}$$
For such a network, we remove the contribution of the identity matrix, which leads to a production function,
$$W(s, Q) = 2 \sum_{i = 1}^N w_i.$$ 
In \cref{app:network_effect} we justify its extensive use and explain why it can be considered as a canonical example.
\end{itemize}
\end{example}

\subsection{Workloads and demands for produced outputs}
Additionally, we use different parameters on the model. One specific goal of our model setup is to allow for complex demand functions to influence the interaction of agents and their production, as discussed in the introduction with \textit{(4)}. In the context of economic scenarios, the demand may be the actual global demand for goods or services, but in disciplines like neuroscience or computer science it may also represent the environmental task or computational task faced by humans or machines which naturally are an important component for driving the optimal system-level setup. We represent it as a positive function over the operations space, 
$$W_0 \in \mathscr{F}(\T_{\text{ops}}).$$ 

We note that a task, specifically in its computing occurrence of workload, is more naturally represented as a graph informing on internal dependencies. However, we argue in \cref{sec:workload_as_graph} that the core result of this paper, i.e. \textit{The Maximum Heterogeneity Principle}, still holds. For the sake of simplicity, we use demand, workload, and task synonymously hereafter. In some cases, which correspond to specific modelling choices, the task is simply the maximisation of one quantity. For instance, plants that aim to maximise their biomass, or a large-scale economic system that aims at maximising GDP. In such cases, the task is the function that equals infinity. Such a technical detail is only mentioned to say that using functions for tasks meets a wide range of criteria. Generally, when there is no real function that models the task, it means that the goal of the system of agents is to maximise production. 

\begin{remark}
    Note that the task lies over the same space as the agents' outputs. This is by design as we want to be able to have a criterion of: a function is realised, a task is solved, a demand is met, the workload is run, and how ``far'' we are to completing it. 
\end{remark}

Consequently the task can take any shape, wrapped Gaussians, wrapped mixture of Gaussians, uniform, positive scaling of these shapes even non-continuous positive functions, and Dirac delta type functions are tasks but they constitute more of an edge case to represent a realistic workload (\cref{app:dirac_workload} discussing the Dirac delta case in more detail). However because we would like to perform a wide range of mathematical operations we restrict ourselves to the class of analytical functions on the torus that we denote $\mathscr{A}(\T)$. The rationale behind analyticity assumption is that it allows for simple theoretical guarantees the optimisation procedure we will introduce (c.f.\ \cref{app:technical_verifications} for detailed explanations). Furthermore this condition does not limit us for the different workloads we will consider.  
Naturally, the demand present in the environment influences the ideal design choices of the production system to meet this demand. The variables that make up the ideal system design include the number of agents, given amount of resources, each of their skill densities $s_i$ and their joint interaction $Q$. The following subsection formalises the framework we use to understand how the demand influences the ideal production system.

\subsection{Ideal design of production system: optimisation framework}
\label{sec:model_optimisation}
We now need a strategy that enables us to find the ideal design of production system given the demand function. Because there is an objective, namely to find an ideal production system for a given task, optimisation (under constraints) is a natural tool. The parameters fixed here are: 
\begin{itemize}
    \item The task, $W_0 \in \mathscr{A}(\T)$.
    \item The number of agents $N$.
    \item The network structure $Q$ and its connected components set $\mathscr{C}$.
\end{itemize}
Given these parameters and the context, a form of production function is set and we denote it as $W_1$. The loss we chose can be thought of as the minimum amount of energy needed (or the minimum amount of financial resource, or time, or more generally, the resource that enables production) for the given production system and for the individual agents production functions $s$ to perform the task: 
$$\mathcal{L}(s) := \sum_{c \in \mathscr{C}}\left\lVert\frac{W_0 - W_1(s_c)}{W_1(s_c)} \mathds{1}_{\lbrace W_0 - W_1(s_c) > 0\rbrace} \right\rVert_{p},$$
where $\left\lVert \cdot\right\rVert_{p}$ denotes the $p$-norm on the torus. The indicator function $\mathds{1}_{W_0 > W_1(s_c)}$ ensures we only penalise the system for operations it fails to cover, as overshooting part of the demand is acceptable while undershooting is not. The summation over the connected components implies that non interacting agents cannot jointly optimise: they have no knowledge about the skills of the agents in other connected components and thus cannot take it into account to solve the task.\footnote{For the moment, with the simplest form of the production function, the interaction effects of the network lie at the connected component level. This interaction will be made deeper for the next iteration of the model.} For simplicity, we have dropped $Q$ from the production function as it is a fixed parameter, not a variable. Since $s$ is uniquely defined by all the means $\mu$ and standard deviations of the agents $\sigma$, the loss can be written as such: 
$$\mathcal{L}(\mu_c, \sigma_c) :=\sum_{c \in \mathscr{C}} \left\lVert\frac{W_0 - W_1(\mu_c, \sigma_c)}{W_1(\mu_c, \sigma_c)} \mathds{1}_{\lbrace W_0 > W_1(\mu_c, \sigma_c)\rbrace} \right\rVert_{p}.$$
Consequently, the framework is a parametric optimisation problem: 
$$ \mu^*, \sigma^* \in \text{Arginf}_{\mu, \sigma} \mathcal{L}(\mu, \sigma).$$
We impose the production function to be analytical with respect to $\mu, \sigma$ and $\theta$, and strictly positive. We further impose that every standard deviation $\sigma_i$ is above a strictly positive constant. Under these physically realistic assumptions, the loss is well-defined, fully differentiable, and amenable to standard gradient descent (c.f.\ \cref{app:technical_verifications}), allowing us to efficiently solve for the optimal agent composition $\mu^*, \sigma^*$ across a wide range of configurations: different workload shapes, different numbers of agents, different numbers of layers, and different resource constraints. Our decision to opt for gradient-based optimisation instead of local optimisation or analytically solvable systems is explained in \cref{app:model_design_optimiser}. In short, opting for a fully analytically tractable system would have significantly limited the model design space. In contrast, local optimisation processes generally have worse convergence guarantees, making the system harder to study across parametrisations. However, as discussed in \cref{app:model_design_optimiser}, global gradient optimisation can often behave similarly to the aggregate of local optimisation processes.

\subsection{Constraints acting on the model}
\label{model_second_order}
We want now to account for different constraints that influence the optimal composition of production systems. Though the constraints will be modelled in a more specific way across the different experiments, we present here the general setup of how constraints act on the production system. We consider three constraints: the \textit{resource constraint} that solely acts on the production function, the \textit{communication cost} and the \textit{second-order constraints} that take the form of penalties added within the optimisation loss. \cref{fig:model_constraints} provides an intuitive visual of how the constraints act on the production system. We give details: 
\begin{enumerate}
    \item The resource constraint $R$ is a level-type constraint that acts on the production function by dividing its output. It models the scarcity of external resources present in the environment used to fuel the production. An absence of such resources (energy, capital, food) is modelled by an infinite constraint, and the production is the null function. Its role on the system will be studied in general in \cref{sec:principles_efficiency}.
    \item The communication cost accounts for the cost of collaboration. We use this penalty only twice (c.f.\ \cref{sec:findings_labour_market} and \cref{sec:principle_topology}), but it leads to interesting properties of the network topology. We model the penalty as such, 
    $$\mathcal{L}_c(s;Q) = \lambda_c \lVert Q \odot D(s)\rVert_1,$$
    where $\odot$ denotes the Hadamard product, $\lambda_c > 0$ is a weighting coefficient and $D(s)$ is a dissimilarity matrix between the agents accounting for harder communication between dissimilar agents.
    \item The second-order constraints come in the shape of an underlying internal `resource' level that can modify the skills function and hence production of agents to varying degrees. As before, we describe the available internal resources as densities, denoted $s^{(2)}$, connected via $Q^{(2)}$ in the same numerical space as the skill densities of agents $\T$, and, if present, they influence the systems by providing the resources to output the skills function of the layer and constrain it in the form of an additional cost in the loss proportional to the distance of agents' skills to the lower level resources. We only use this layer for a couple of specific experiments, and in \cref{sec:compute_systems} more generally, where they highlight how agents produce outputs according to their skills and underlying tools / hardware when they are available. Mathematically, the second-order constraints act as a penalty in the optimisation loss and we denote it $\mathcal{L}_s$.  
\end{enumerate}

By convention, as soon as any penalty is present, we denote the optimisation loss introduced above $\mathcal{L}_m$ and refer to it as the task mismatch cost. In full generality the optimisation takes the form, 
$$\mathcal{L}(s; Q, R) = \mathcal{L}_m(s;W_0, R) + \mathcal{L}_c(s; Q) + \mathcal{L}_s(s; Q, s^{(2)}, Q^{(2)}).$$
Note that in the following we will use a simpler model of the second-order layer.
\vspace{0.4cm}
\begin{figure}[H]
    \centering
    \includegraphics[width=0.5\textwidth]{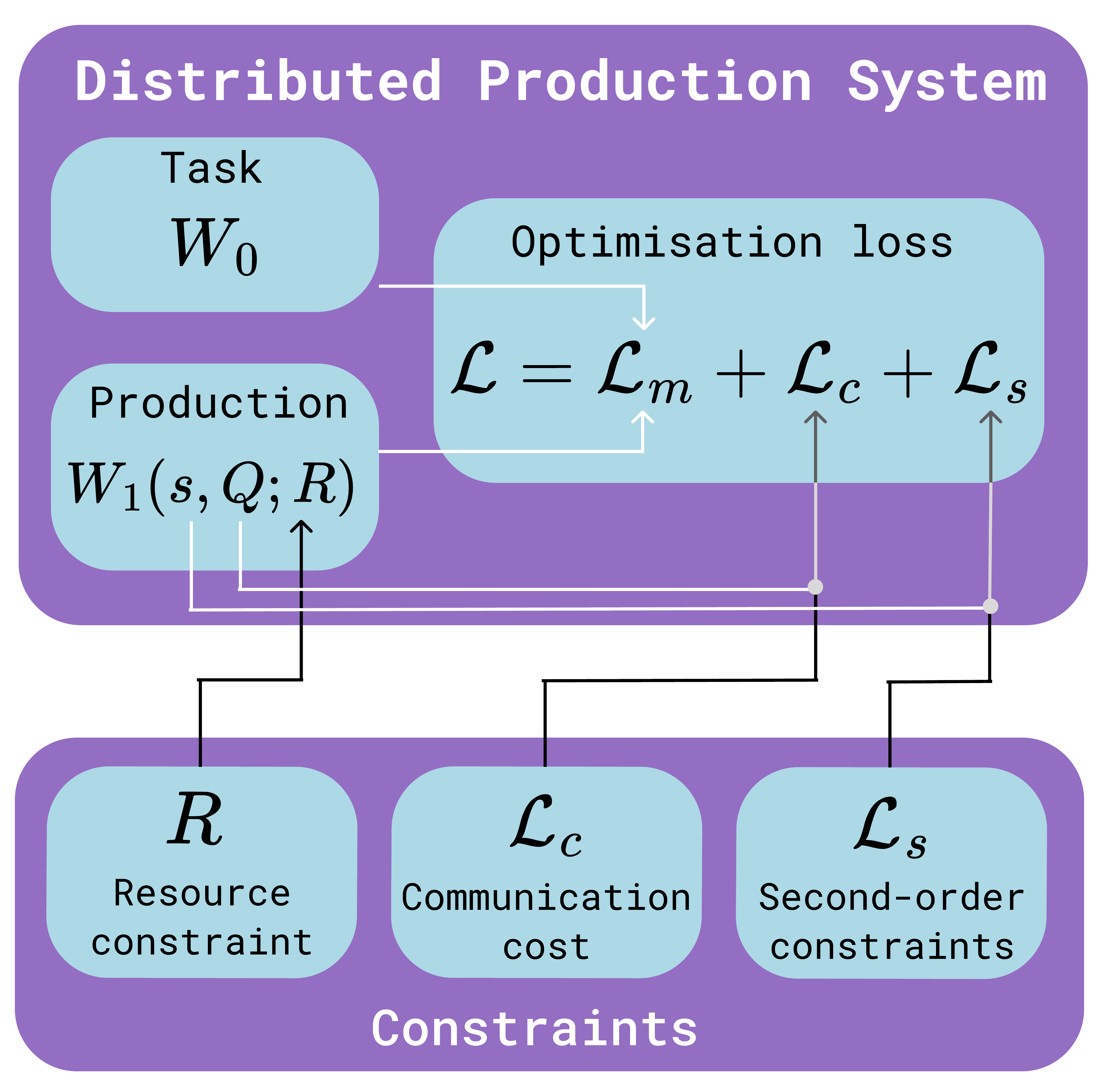}
    \vspace{0.2cm}
    \caption{\textbf{Intuitive visualisation of the distributed production system model with constraints.} The resource constraint acts as a parameter on the production function. The communication cost, combined with the network structure, influences the gradient descent by adding a penalty. The second-order constraint, combined with the agents' skill densities, takes the form of a penalty in the optimisation loss. The arrows indicate the effects: white denotes the parameters of the model, black the effect of the constraints and grey the combined effects of the constraints and model parameters.}
    \label{fig:model_constraints}
\end{figure}

\subsection{Full formulation of the model}

We now consolidate all previously introduced components into a unified framework. The model describes a distributed production system whose purpose is to meet a given demand workload, by optimally composing agents with heterogeneous skills connected through a network structure.

\paragraph{Spaces.}
The model is built on a common one-dimensional torus $\mathbb{T}$, which simultaneously plays three distinct roles depending on the layer it represents:
\[
\mathbb{T}_{\text{res}} \xrightarrow{\quad \mathrm{id} \quad} \mathbb{T}_{\text{skills}} \xrightarrow{\quad \mathrm{id} \quad} \mathbb{T}_{\text{ops}},
\]
where $\mathbb{T}_{\text{res}}$ is the resource space, $\mathbb{T}_{\text{skills}}$ is the agents' skills space, and $\mathbb{T}_{\text{ops}}$ is the operations space in which both production outputs and workloads are expressed. Since these three spaces are identical as sets, we write $\mathbb{T}_{\text{res}} = \mathbb{T}_{\text{skills}} = \mathbb{T}_{\text{ops}} \overset{\mathrm{not.}}{=} \mathbb{T}$ whenever there is no ambiguity. This recursive identification encodes the self-similar, hierarchical nature of the model, in which the output of one layer becomes the input of the next. Throughout this paper, we treat the mappings between spaces as the identity function, however \cref{app:between_layer_mapping} discusses how this is not a necessary assumption of our model.
\paragraph{Workload.}
The demand is represented as a positive, continuously differentiable function over the operations space,
\[
W_0 \in \mathscr{A}(\mathbb{T}).
\]
Each point $\theta \in \mathbb{T}$ denotes a particular operation, and $W_0(\theta)$ quantifies how much of that operation is required. The task is considered completed as soon as the system's aggregate production covers the demand everywhere, i.e.\ $W_1 \geq W_0$.

\paragraph{Agents and production system.}
The production system is composed of interacting agents. Each agent $i$ has a skill density modelled as a wrapped Gaussian density over $\mathbb{T}$ with mean $\mu_i \in \mathbb{T}$ encoding the agent's area of specialisation and standard deviation $\sigma_i \in \mathbb{R}_+^*$ encoding the breadth of its skills. An agent's individual production density is the direct mapping of its skills density onto the operations space.  
Interactions are realised through a network with a topology that is encoded in an adjacency matrix $Q \in \mathbb{R}^{N \times N}$. Collecting all skills densities in the vector $s = (s_1, \dots, s_N)^\top$, the aggregate production of the distributed system is denoted $W_1(s, Q)$.

\paragraph{Optimisation.}
Given the task $W_0$, the number of agents $N$, the network $Q$, and any additional constraints $C$ (e.g.\ resource constraints, second-order constraint, or communication too costly to be negligible), the optimal production system composition $(\mu^*, \sigma^*)$ is found by minimising the loss:
\[
\mathcal{L}(\mu, \sigma) \;:=\; \sum_{c \in \mathscr{C}} \left\lVert \frac{W_0 - W_1(\mu_c,\sigma_c)}{W_1(\mu_c,\sigma_c)}\,\mathbf{1}_{\{W_0 > W_1(\mu_c,\sigma_c)\}} \right\rVert_{p},
\]
which penalises only the operations for which the system undershoots the demand. The optimisation problem reads
\[
\mu^*,\, \sigma^* \;\in\; \operatorname*{arg\,inf}_{\mu,\,\sigma} \;\mathcal{L}(\mu,\sigma) \quad \text{subject to } C.
\]
Under the assumptions that $W_0$ is analytical, $W_1$ is analytical in $(\mu, \sigma, \theta)$, strictly positive in $\theta$, and that all $\sigma_i$ are bounded away from zero, the loss is well-defined and fully differentiable, making gradient-based optimisation directly applicable. 

In summary, the model couples a demand $W_0$, a set of $N$ specialised agents $((\mu,\sigma), Q)$ forming a distributed production system, and a hierarchical torus structure $\mathbb{T}_{\text{res}} \to \mathbb{T}_{\text{skills}} \to \mathbb{T}_{\text{ops}} = \mathbb{T}$ into a single, gradient-optimisable framework for studying how the composition and interaction of agents drive optimal, scalable production.

\subsection{Measures to study the model}
\label{sec:model_metrics}

To systematically study the optimal solution that our model converges on for a specific discipline specific experiment (\cref{sec:findings}) or when studying the general behaviour of the model to extract its principles (\cref{sec:principles}), we need metrics that can characterise the current state of all agents, and their joint solution in context of a specific demand or workload. To achieve this, we derive a new metric to describe the system's level heterogeneity and use it alongside established measures to capture agents' degree of specialisation, the overall efficiency of the system, and overall amount of production. These are introduced in the following, ordered by level of analysis, from single agents to whole system.

\subsubsection{Specialisation}
The first metric we consider is specialisation or equivalently level of specialisation, which is a property that is measured on the level of individual agents. For an agent's skills density $s_i$, it corresponds to how concentrated the distribution is around its peak — a specialist has a tall and narrow density, whereas a generalist has a flat and broad one, with the same unit area distributed over a wider range of skills. It describes whether a given agent specialises in a specific part of the space of skills / abilities. Because $(\mu_i, \sigma_i)$ and $\sigma_i$ are sufficient statistics of the agents skill density, we can uniquely define, the specialisation of an agent to be the inverse of its standard deviation:
$$\varsigma(s_i) = \varsigma(\mu_i, \sigma_i)  = \frac{1}{\sigma_i}.$$
The individual agents' specialisations can be averaged across the population of agents to measure the level of specialisation of the production system with vectors of skills densities $s$:
$$\varsigma(s) = \varsigma(\mu, \sigma)  = \frac{1}{N}\sum_{1 \leq i \leq N} \frac{1}{\sigma_i}.$$
Because of the equivalence between skills space and operations space, the specialisation is also measured equivalently in terms of production. For one single agent, with production density $w = s$, its productivity specialisation is 
$$\varsigma (w_i) = \varsigma(\mu_i, \sigma_i)  = \frac{1}{\sigma_i}.$$

\subsubsection{Heterogeneity}
The heterogeneity is a property that has started to receive attention through accumulating empirical findings \cite{eisenhauer2023heterogeneity, vijay2015}. It is inherently a measure of the diversity of a group of agents based on characteristics of interest. For example, in a group of flowers, we can focus on the heterogeneity of colours, of sizes, or on the number of bees they attract per year, or all of these features. There is no standard measure of heterogeneity yet. We hence derive a comprehensive measure of heterogeneity, expanding prior relevant works \cite{Hill1973, Jost2006, Jost2009, LeinsterCobbold2012, Leinster2021}. Here we give a short overview of the metric and the full derivation of the metric is in \cref{app:heterogeneity_metric}.
We want to measure heterogeneity of the composition of production system. Thus the characteristic of interest for us is the agents' skill densities. Building on Hill numbers \cite{Hill1973} and the similarity-sensitive diversity framework of \cite{Leinster2021, LeinsterCobbold2012}, we construct the metric from the pairwise similarities between agents' skill densities, collected in a similarity matrix $Z \in [0,1]^{N \times N}$, where $Z_{ij}$ decreases as agents $i$ and $j$ grow more dissimilar. Specifically, we define 
$$Z_{ij} = \exp\left(-d_{\circ}(\mu_i, \mu_j)/\sqrt{\sigma_i^2 + \sigma_j^2}\right),$$
 where $d_{\circ}$ is the circular distance between agent means, defined as, 
 $$d_{\circ}(\mu_i, \mu_j) := \min(|\mu_i - \mu_j|, 2 \pi - |\mu_i - \mu_j|)$$ and the denominator accounts for the breadth of both agents' skills. Because agents are uniquely identified within the network, we use the uniform distribution over agents and fix the sensitivity parameter $q = 2$, yielding the heterogeneity 
 $$\mathcal{H} = \frac{N}{1 + \frac{2}{N}\sum_{i<j} Z_{ij}} - 1.$$ This measure takes values in $[0, N-1]$: lower values indicate a more homogeneous system and higher values a more heterogeneous one. When there exist subgroups (limit case) that are maximally similar (i.e.\ pairwise similarity equal to one), we can form subgroups within the set of agents based on identical similarity. Denoting $p_i$ the proportion of agents in the subgroup $i$ and $p$ the associated vector of proportions of agents in each subgroup, the heterogeneity becomes, 
 $$\mathcal{H} = \frac{N}{1 + \frac{2}{N}\sum_{i<j} Z_{ij}} - 1 = \frac{1}{p^T \widetilde{Z}p} - 1,$$
 with $\widetilde{Z}$ the reduced similarity matrix. For example for a production system with three agents with two maximally similar agents and one maximally dissimilar agent (similarity of 0), we have the reduction:
 $$Z = \begin{pmatrix}
     1 & 1 & 0 \\
     1 & 1 & 0 \\
     0 & 0 & 1 \\
 \end{pmatrix} \to \widetilde{Z} = \begin{pmatrix} 
     1 & 0 \\
     0 & 1 \end{pmatrix} \quad \text{and } p = \left(\frac{2}{3}, \frac{1}{3} \right).$$

When the pairwise similarity is not taken into account, i.e.\ when $Z$ is the identity matrix, the heterogeneity reduces to the classical Hill number of order $2$, thereby satisfying the standard axioms of symmetry, expandability, the transfer principle, and the replication principle \cite{Jost2009, Nunes2020} while taking into account similarity \cite{Leinster2021, LeinsterCobbold2012}. Its full derivation and properties are given in \cref{app:heterogeneity_metric}.

\subsubsection{Systems-level production and productivity}
The system-level production $\mathcal{P}(s, Q)$ results from how interaction of all agents and is hence a function of all agents' parametrisation and their pairwise interactions. It is defined as the total area under the joint production function, i.e.\ the integral of $W_1$ over the operations space:
$$\mathcal{P}(s, Q) = \|W_1(s, Q)\|_1 = \int_{\mathbb{T}} W_1(s,Q)(\theta)\, d\theta.$$
The system-level over-production $\overline{\mathcal{P}}(s, Q)$ corresponds to the amount of production that exceeds the demand. It is defined as the area is above the demand and below the joint production function, i.e. the following integral:
$$\overline{\mathcal{P}}(s, Q) = \|(W_1(s, Q) - W_0) \mathbf{1}_{\{W_0 < W_1(\mu_c,\sigma_c)\}} \|_1 = \int_{\{W_0 < W_1(\mu_c,\sigma_c)\}} (W_1(s,Q)(\theta) - W_0 (\theta))\, d\theta.$$
With these definitions we can define the concept of effective production, which is the production that truly serves to meet the task. It corresponds to the area below both the joint production function and the demand, i.e. the part of production that is actually absorbed by demand:
$$
\mathcal{P}_{\mathrm{eff}}(s,Q) = \|(W_1(s, Q) - W_0) \mathbf{1}_{\{W_0 > W_1(\mu_c,\sigma_c)\}} \|_1
$$
This quantity measures the share of system-level productivity that is effectively matched to demand. Moreover, the total production decomposes naturally as
$$
\mathcal{P}(s,Q)
=
\mathcal{P}_{\mathrm{eff}}(s,Q)
+
\overline{\mathcal{P}}(s,Q).
$$

Related to effective production, we define the productivity, which is can be seen as an index version of the effective production, which also takes into account the constraints. We define it as:
$$\Pi(s; Q) = \frac{1}{1 + \mathcal{L}(s;Q)},$$
where $\mathcal{L}$ is the optimisation loss. 
A productivity of zero corresponds to no efficient production, while a productivity of one is related to a task entirely performed with no additional costs given the constraints.

\subsubsection{Systems-level efficacy and efficiency}
Systems-level efficacy and efficiency both expand on production by taking the demand $W_0$ into account. Efficacy $\mathcal{F}$ captures how much of the production is actually used to meet the demand. It corresponds to the ratio demand is covered by the system's production, defined as the fraction of total production that falls within the demand:
$$\mathcal{F}(s, Q) = \frac{\mathcal{P}_{\mathrm{eff}}(s,Q)}{\mathcal{P}(s,Q)}.$$
Efficiency $\mathcal{E}$ is related to how well the resource are used. It measures the proportion of resource that goes is use for efficient production: 
$$\mathcal{E}(s, Q) = \frac{\mathcal{P}_{\text{eff}}(s,Q)}{N},$$ where $N$ is the number of agents as a proxy for resources used. Efficacy thus measures coverage of the demand, while efficiency measures how economically that coverage is achieved.

%% file: Text/3_findings.tex
We believe that the simple model system described in \cref{sec:model} captures the core functioning of the distributed production systems we observe across biology, neuroscience, economics, finance, and computing. If true, we could then use the model to derive which core principles drive the productivity of distributed systems across fields. \textbf{Before we start validating our model on established findings, we ought to define what exactly we mean with distributed production systems.} A distributed production system is any system in which agents jointly, through both their collaboration and competition, produce an output. For example, all plants of an ecosystem jointly produce biomass as a joint output through a complex network of collaborative and competitive relationships. In the same way, neurons jointly produce appropriate muscle responses or cognition, and countries’ production and trade creates a global GDP. Another way of characterising distributed production systems is that they are systems where the agents within them engage in a non-zero sum game, so that there is no ‘the winner takes it all’-dynamic and that appropriate coordination between agents (which can be achieved both through collaboration but also competition) can increase the output of the entire system, so produce additional output.

In the following we show how our model matches distributed production system observed across disciplines. We begin each discipline specific section of findings with an `introductory' finding that follow from the model in a straightforward manner, to allow the reader to develop an intuition for the model, before moving on to more complex findings. As each discipline has an `introductory' finding in its respective section, readers should be able to skip individual disciplines.

\subsection{Ecology}
\label{sec:findings_ecology}

The first set of findings that we want to compare our model to is biology, specifically macro-scale interactions as they are studied in ecology. Ecology provides an interesting comparison point as it very naturally suits the description of a distributed production system where multiple independent agents produce a joint output through their interaction. In the case of biology the productivity of a system is usually captured by the joint biomass produced by all individuals. In this first section, we hence want to look in detail at how our model fits findings of inter-individual biomass production.

\subsubsection{Specialised species develop over time in new environments}
\label{sec:ecology_specialisation}

The theory of adaptive radiations describes the development of a species when it arrives in a new environment with currently unrealised ecological niches, so parts of the natural habitat containing enough food and lack of predatory pressures to proliferate which are otherwise uninhabited. In such cases it can be observed that a species with a single ancestral lineage starts forming multiple ecologically distinct species, where new species take increasingly specialised roles in the environment \cite{yoder2010ecological}. We want to test whether our model shows the same overall dynamic observed in the data. For this we create a simulation with a multi-Gaussian workload shape and observe the average degree of specialisation of the population of agents as we increase the total number of agents. We opt for this workload shape to highlight some parts of the environment have more resources than others, however we also contrast this with a uniform workload shape. We use a fully circular network where every agent can interact via $J_N$. The results of this analysis are depicted in \cref{fig:finding_ecology_scaling}. We observe that for non-uniformly shaped distribution of resources in the environment we get the expected relationship, where agents at first take generalist roles as the environment is not yet strongly populated, but as more agents develop, they take increasingly specialised roles. Importantly this is not true for fully uniform distribution of resources across the environment, in which case our model would predict a set of generalists could potentially co-exist. Whether a truly uniform distribution of resources across the environment is realistic is however doubtful and we assume that a complex-shaped distribution would be observed in most environments. Of course, in these cases our model aligns with the expected trend.
\begin{figure}[H]
    \centering
    \includegraphics[width=0.5\textwidth]{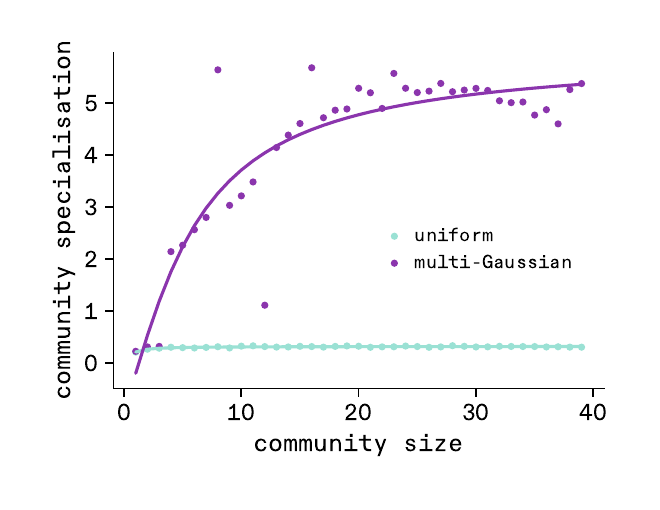}
    \caption{\textbf{According to the theory of adaptive radiations, as species proliferate in a new environment, they become increasingly specialised.} For two workload shapes (uniform, multi-Gaussian) we study the average specialisation of the converged model as we increase the number of agents. The predicted trend according to the theory of adaptive radiations is highlighted with the dotted line. As expected, the multi-Gaussian workload with its `ecological niches' within its distribution replicates the expected trend.}
    \label{fig:hsl_2_0}
    \label{fig:finding_ecology_scaling}
\end{figure}

\subsubsection{Diversity of specialisation increases and stabilises biomass production}

The preceding results cover how species interact with environments during the unusual moment of new environments opening up to species. We next want to focus on the more typical interaction of species with the environment, and the specific mechanisms that lead to increased or decreased biomass production over time. One specific example of such a mechanism has been described in the spatial insurance hypothesis \citep{yachi1999insurance, Loreau2003, schindler2010population, shanafelt2015biodiversity, eisenhauer2023heterogeneity} that predicts that community diversity both raises the temporal mean of biomass production and buffers its variability under environmental pressure. 

We verify whether these two signatures, higher mean production and lower temporal variability, emerge from the composition layer of our model when a homogeneous community is compared with a heterogeneous one. The homogeneous community is formed of two similar agents with both equal means $\mu{=}(\pi,\pi)$ and equal standard deviations $\sigma{=}0.5$. The heterogeneous community also has two agents with the same standard deviation $\sigma{=}0.5$ but with different means $\mu{=}(\tfrac\pi2,\tfrac{3\pi}{2})$. Interaction is modelled by the identity matrix and resource constraints are set to $4$. The environmental pressure is a unimodal Gaussian with the same standard deviation ($\sigma_{\text{task}}{=}0.5$) that drifts continuously as a wave, as in \citep{Loreau2003}, around the circle completing ten
full periods. At each time step we compute the production of biomass for the two communities $$\Pi(t)=1/\bigl(1+\mathcal{L}(t)),$$
with the optimisation loss of our model. The mean production is computed by averaging the production across the ten periods, and the coefficient of variation (CV), defined as the ratio of temporal variance to temporal mean, is the measure of variability we use.
We first run a single-task experiment and then generalise the result with a catalogue of sixteen unimodal environmental pressures (c.f. \cref{app_fig:workloads_unimodal} for visual of the catalogue) whose standard deviations
span the range $\sigma_{\text{task}}\!\in[0.3,2.0]$, reflecting an environmental stress that has a single dominant pressure drifting through niche space favouring one species at a given time.
For the single drifting Gaussian pressure (c.f. \cref{fig:combined_metacommunity_results} (top)), the heterogeneous community achieves a temporal mean production $70\%$ higher than the homogeneous community and has half the variability, measured by the coefficient of variation, confirming both predictions of the insurance hypothesis qualitatively. This result is further verified across the sixteen environmental pressures as shown in \cref{fig:combined_metacommunity_results} (bottom). The heterogeneous community constantly produces a higher temporal mean than its homogeneous counterpart and exhibits a lower temporal variability in all $16/16$ tasks. Table \ref{tab:ecol2_catalogue} summarises the aggregated statistics and associated tests.
\begin{figure}[H]
     \centering
     \begin{subfigure}{\linewidth}
         \centering
         \includegraphics[width=0.8\linewidth]{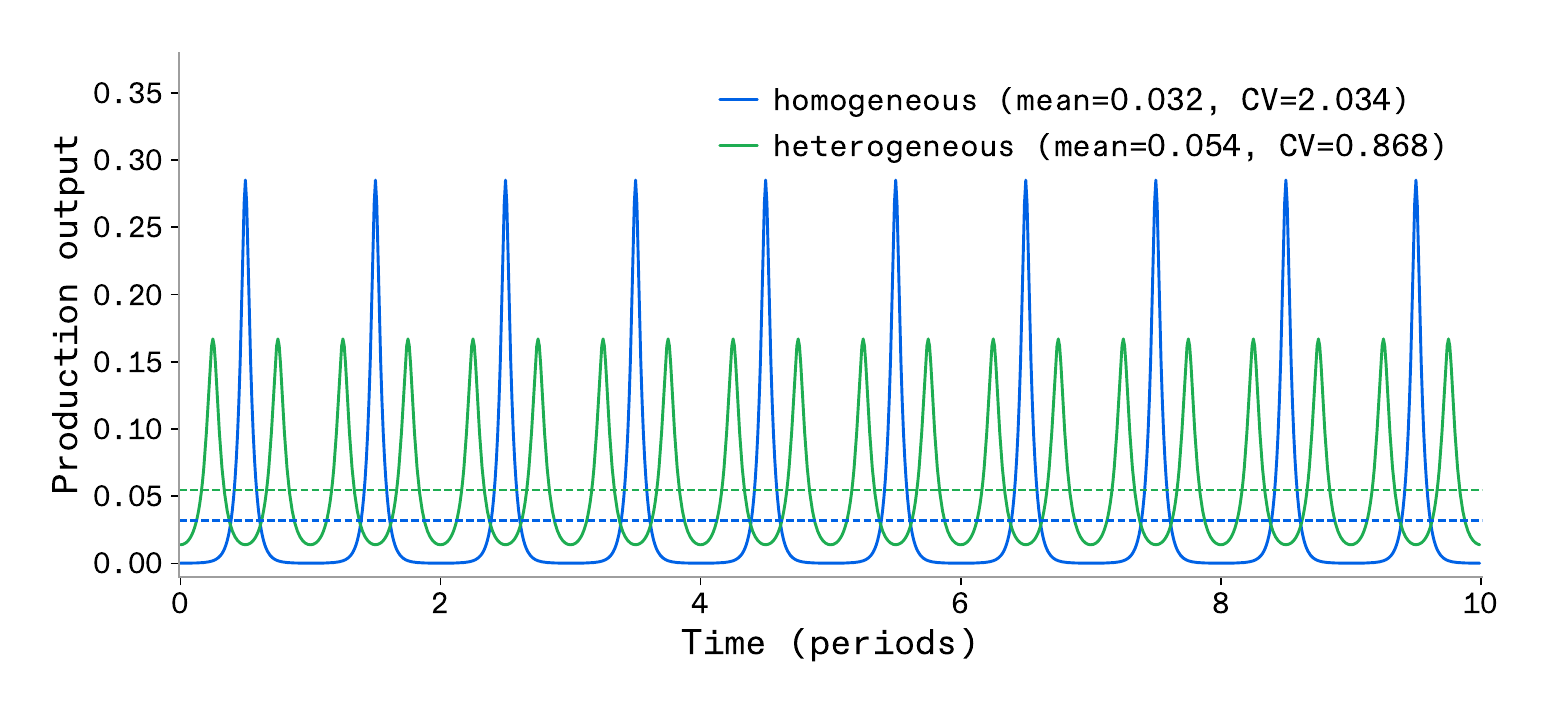}
     \end{subfigure}
     \begin{subfigure}{\linewidth}
         \centering
         \includegraphics[width=0.8\linewidth]{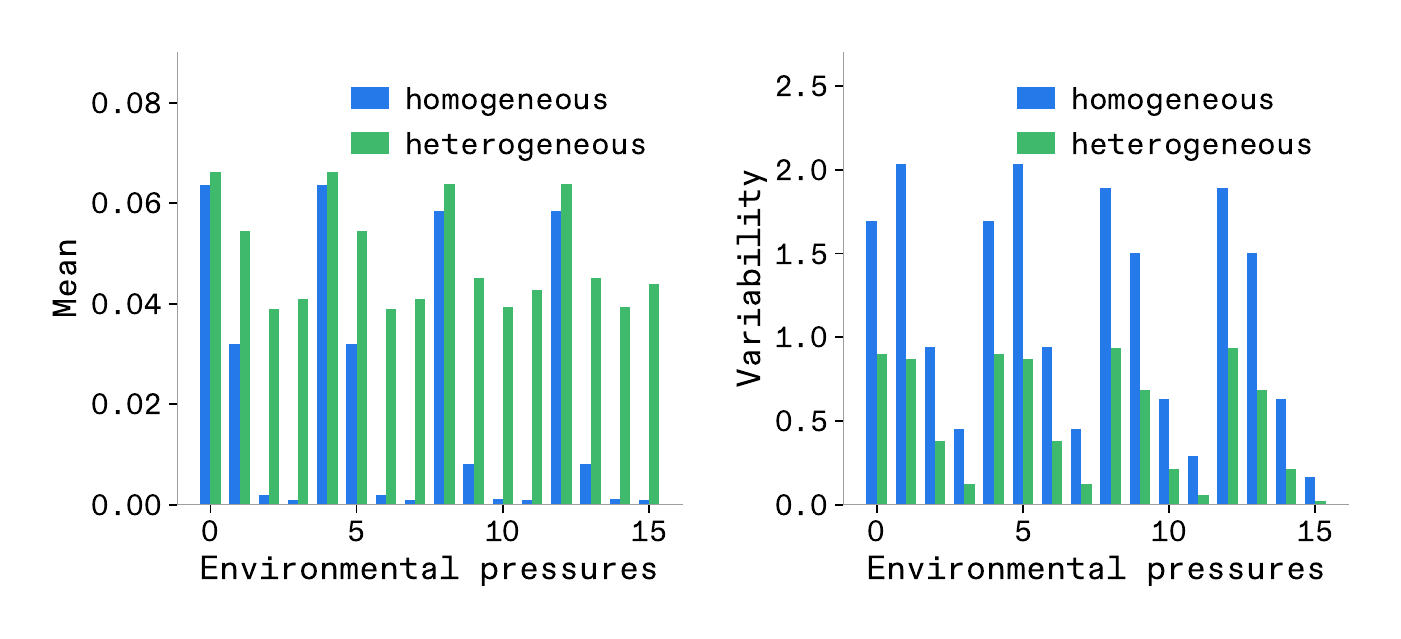}
     \end{subfigure}
     \caption{\textbf{Homogeneous and heterogeneous temporal biomass production across unimodal environmental pressures.} \textit{Top:} Production time series for the single drifting Gaussian task. The heterogeneous community (green) maintains higher and more stable production than the homogeneous community (blue). Dashed lines indicate temporal means. \textit{Bottom:} Per-task mean production (left) and coefficient of variation (right) for the sixteen environmental pressures. The heterogeneous community (green) exceeds the homogeneous one (blue) in mean production and falls below it in CV for every task ($p<10^{-4}$, paired tests).}
     \label{fig:combined_metacommunity_results}
\end{figure}

\begin{table}[H]
  \centering
  \begin{threeparttable} 
    \caption{\textbf{Statistical comparison of community biomass production and variability.} Summary statistics, paired $t$-test results, and effect directions comparing homogeneous and heterogeneous communities across 16 unimodal environmental pressures.}
    \label{tab:ecol2_catalogue}
    \small
    \begin{tabular}{l c c c c c}
      \toprule
      Statistic
        & Homogeneous
        & Heterogeneous
        & Test
        & $p$-value
        & Direction \\
      \midrule
      Mean production
        & $0.021$
        & $0.049\;(+134\%)$
        & $t = 7.24$
        & $2.86\!\times\!10^{-6}$
        & $16/16$ \\[3pt]
      Variability\tnote{a} 
        & $1.17$
        & $0.52\;(-56\%)$
        & $t = 8.06$
        & $7.91\!\times\!10^{-7}$
        & $16/16$ \\
      \bottomrule
    \end{tabular}
    \begin{tablenotes}
      \small
      \item[a] A non-parametric Wilcoxon signed-rank test further confirms the CV result ($W{=}0$, $p{=}4.34\!\times\!10^{-4}$).
    \end{tablenotes}
  \end{threeparttable}
\end{table}
Both signatures of the spatial insurance hypothesis, elevated mean biomass production and buffered temporal variability, are reproduced by the composition and task layers of our production model. Crucially, these properties emerge from the interaction between diversity of agents' specialisations and a drifting environmental task, with no calibration to ecological data.

\subsection{Neuroscience}
After having looked at biology more through a macro-level perspective of species interacting, we now want to focus on the smaller scale processes within individual animals, specifically information processing in the brain. For this, \cref{sec:findings_neurons} will first discuss the interaction of individual neurons during perception, followed by \cref{sec:findings_MD} discussing how multiple regions specialise to collaborate via the connectivity and topology of the brain's connectome.

\subsubsection{How neurons interact to jointly perceive stimuli}
\label{sec:findings_neurons}
A large number of neurons in any animal's nervous system are concerned with processing visual information, and they are distributed across the eye's retina as well as species-specific visual processing regions of the central nervous system such as the visual cortex in mammals \cite{grill2004human}. Many of these `vision neurons' act like visual filters, responding to specific visual elements present across images, such as dots, lines, and directed movement \cite{hubel1962receptive, grill2004human}. We here see a system where large complex visual stimuli are perceived through the joint work of many distributed agents and hence we want to study whether the way that they work together follows our model system. 

To do so, we can study the specific distributions of specialisations that are observed in studies \cite{laughlin1981simple, qiu2021natural}.
If one looks at the distribution of visual filters in the visual system, one can observe a pattern that the distribution of such filters is not homogeneous \cite{wells2016variation, qiu2021natural}. Instead, when studying the spatial distribution of visual filters within the retina, where the location of a filter in the retina is directly linked to a location in the visual field, one finds that they are heterogeneously distributed according to expected stimuli \cite{laughlin1981simple}. As an example one can find that the distribution of specific colour perceiving filters in the retina of mice follows the distribution needed to identify colour contrasts well in the respective upper vs lower part of the retina, related to the difference in colours in the sky vs the ground \cite{qiu2021natural}.

We wish to verify whether the optimisation mechanism of our model drives the agents' skill densities, which correspond to the distribution of neural filters, to converge separately onto the stimulus distribution each cluster receives, and jointly onto the full stimuli distribution. The stimuli distribution maps onto the task, that is the workload density over the operations space. We set up the experiment as follows. The overall task is a bimodal Gaussian $W_0$. We consider 4 agents arranged in two clusters of two, and each cluster of agents is presented with one of the two modes corresponding to individual stimulus. The interaction matrix is block diagonal: agents within each cluster communicate perfectly ($Q_{ij}=1$) while agents across clusters do not communicate at all ($Q_{ij}=0$). The task consists of two Gaussian modes centred at $\mu_{\text{task}}=(\pi/2, 3\pi/2)$, both with standard deviation $\sigma_{\text{task}}=0.7$, presented in a disaggregated fashion: cluster 1 receives only the first mode and cluster 2 only the second. From the same initialisation we optimise the agents' skill parameters (means and standard deviations), minimising the production loss, under a resource constraint of twice the number of agents.
The results are displayed in \cref{fig:N2_1}. In the left panel, the combined filter of cluster 1 overlaps almost exactly with stimulus 1, showing that the two agents in this cluster have jointly shaped their skill densities to match the Gaussian mode they receive. The centre panel shows the same for cluster 2 and stimulus 2. In the right panel, the sum of all four agents' filters superimposes the full bimodal stimuli, and the two curves are virtually indistinguishable. The filters and stimuli match at every point of the operations space.
\begin{figure}[H]
    \centering
    \includegraphics[width=1\textwidth]{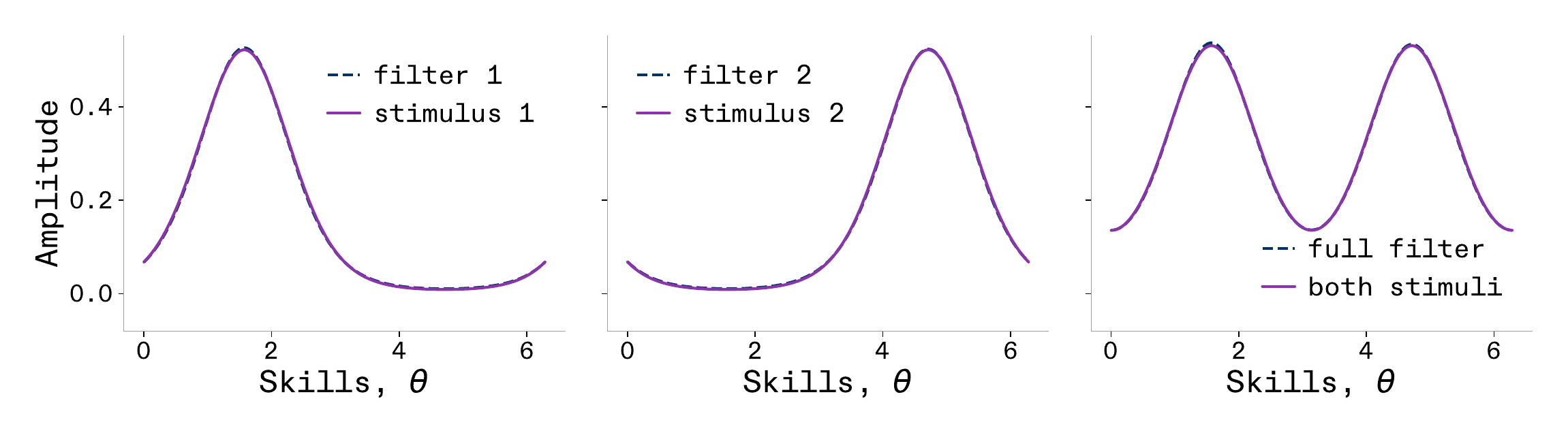}
    \caption{\textbf{Agent populations naturally adapt their filters to match local stimulus distributions.} Two independent pairs of agents (populations) are optimised on different stimulus distributions, mirroring distinct regions of a visual field. Solid purple lines represent the expected stimulus (workload), and dashed blue lines represent the learned filters. \textit{(Left, Center)} When isolated and presented with different expected stimuli, the two populations independently converge to specialised heterogeneities that perfectly match their local workloads. \textit{(Right)} The combined full system demonstrates how these spatially segregated, specialised populations collectively cover the broader, bimodal distribution of the entire environment, analogous to the heterogeneous distribution of colour filters across the retina.}
    \label{fig:N2_1}
\end{figure}
The Distributed Production System therefore reproduces the neuroscience finding that neural filter distributions adapt to match the stimulus distributions they encode. Each cluster of agents, receiving only a fraction of the total workload, independently converges onto the shape of its local stimulus, and the aggregate filter of the full population recovers the complete bimodal stimuli. This emergence arises from the optimisation of production under resource constraints. The result provides further validation that the model captures an organising property of biological neural systems, whereby the heterogeneous tuning of neuronal filters is not arbitrary but is the consequence of efficient adaptation to the environment.

\subsubsection{How regions specialise to collaborate via the brain's connectome}
\label{sec:findings_MD}
In the brain, complex information is processed not by individual neurons but instead by complex networks of them, arranged in brain regions with their unique specialisations \cite{zhou2025default, assem2020domain, achterberg2023building}. Depending on the situation and complexity of the information, different brain regions collaborate to jointly produce complex behaviour \citep{achterberg2023building, duncan2025construction, mashour2020conscious}. An intriguing finding from large primate brains is the variability of specialisation across brains with many regions. Specifically, we observe that a topologically central region acts as a generalist for complex integrative tasks, called the Multiple-Demand (MD) system, and links to a wide network of more specialised regions in the topological surroundings, which can handle domain-specific processing \cite{duncan2010multiple, duncan2020integrated, achterberg2023building, assem2020domain}.

\begin{figure}[H]
    \centering
    \includegraphics[width=0.3\textwidth]{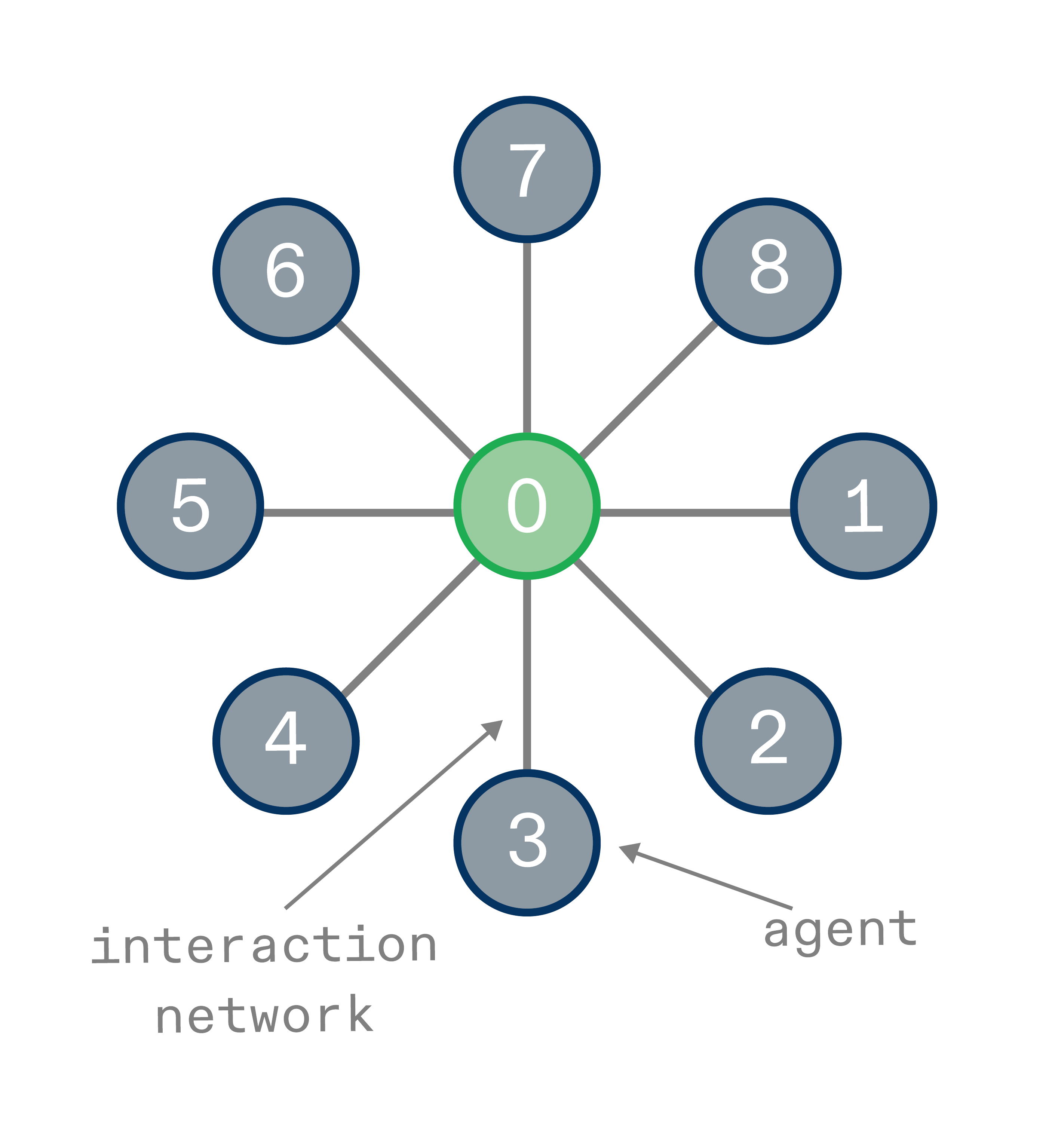}
    \caption{\textbf{Network topology underlying the MD system experiment.} To run the MD system experiment we use a star-like topology, where we expect the central agent to converge on a generalist role, with peripheral agents converging on specialised roles.}
    \label{fig:N1_1}
    \label{fig:findings_MD_topology}
\end{figure}

To test whether this emerges in our model, we create a setup with a Gaussian mixture workload shape and a set of $9$ agents that are connected in a star-like topology with one agent in the middle and all other agents at the periphery connected to the central agent (\cref{fig:findings_MD_topology}). We set a high resource constraint of 100 and we perform a Monte-Carlo approximation of the specialisation level of the optimal agents' skill densities, over 200 random initialisations. Figure \ref{fig:N1_1_grid} (left) shows an example of optimal agents' skill densities where the central agent (agent number 0) shows a more generalist set of skills spread across the workload distribution, while the peripheral agents show more specialist skill distributions around the central agent. Figure \ref{fig:N1_1_grid} (right) presents the Monte-Carlo means of the optimal level of specialisation and the segments represent one standard error around the mean. Though the deviations around the mean of the peripheral agents' average level of specialisation are large, the level of specialisation of the central agent is clearly less than the specialisations of the peripheral agents. This replicates the functional distribution observed in the context of the Multiple-Demand system.

\begin{figure}[H]
    \centering     
    \begin{subfigure}[t]{0.48\textwidth}
        \centering
        \includegraphics[width=\textwidth]{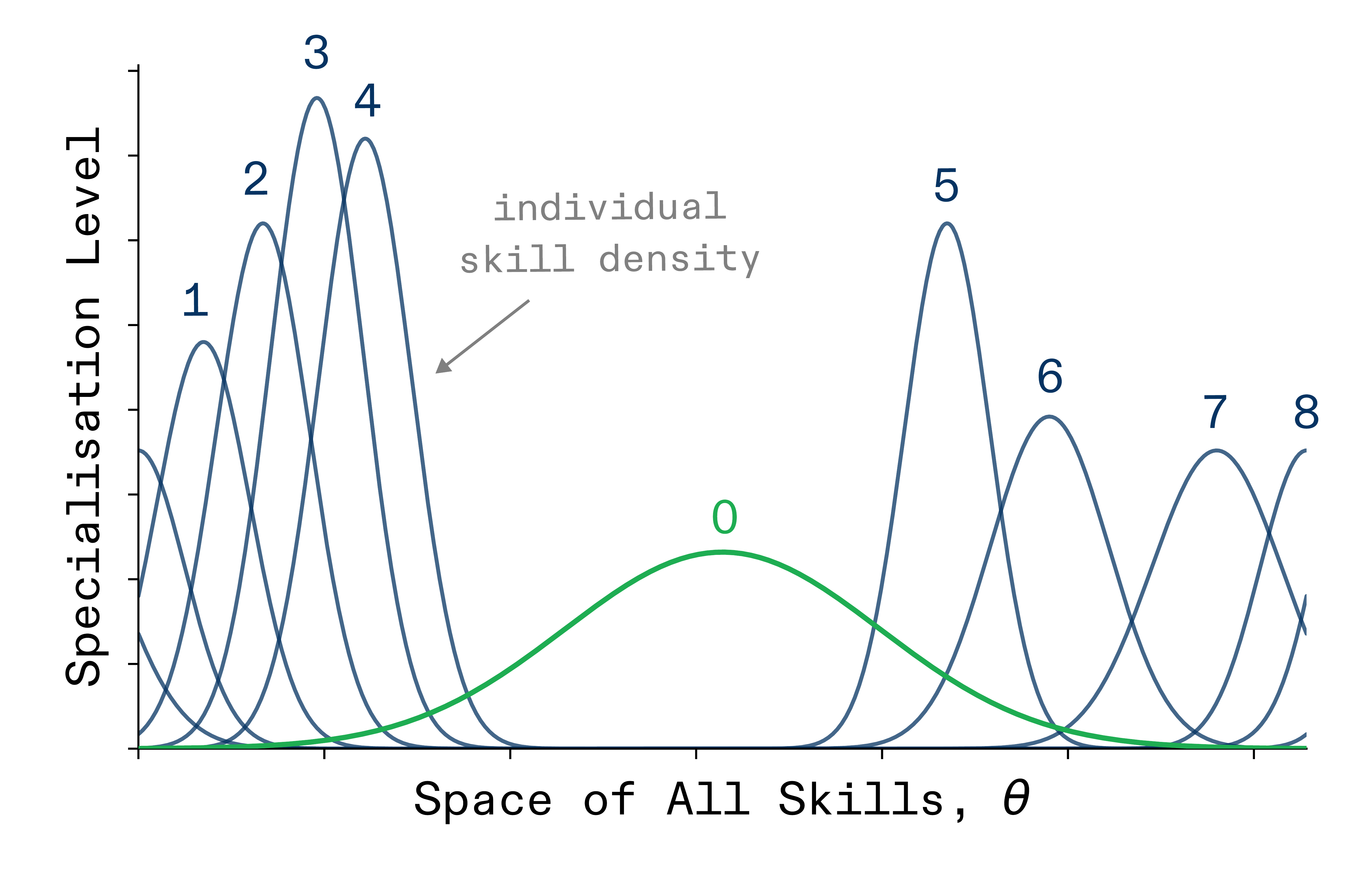}
        \label{fig:N1_1_3}
    \end{subfigure}
    \hfill
    \begin{subfigure}[t]{0.48\textwidth}
        \centering
        \raisebox{-9pt}{\includegraphics[width=\textwidth]{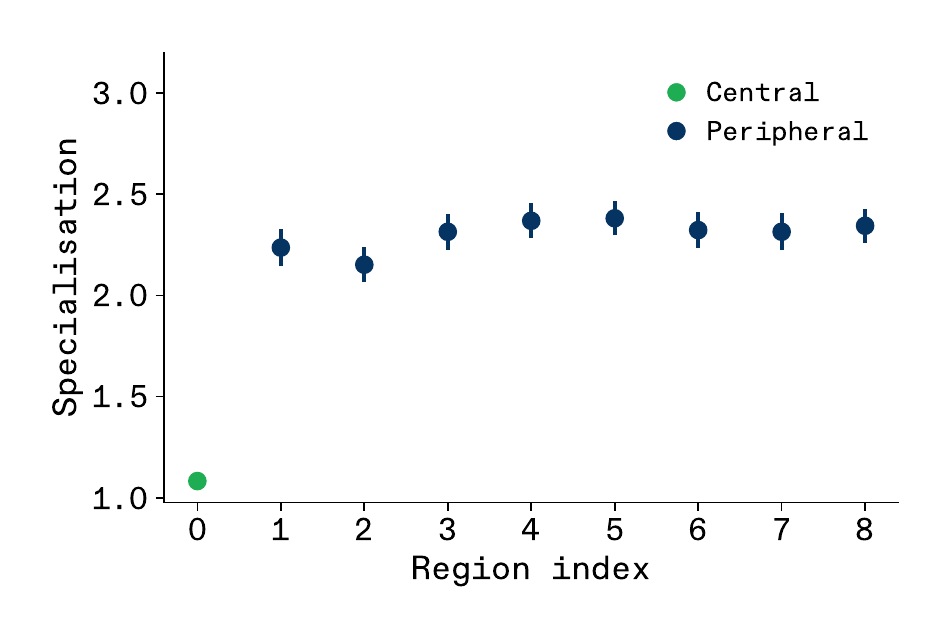}}
        \label{fig:neuron_specialisation_points}
        \label{fig:findings_MD}
    \end{subfigure}
    
    \caption{\textbf{The model converges on brain-like structure-function topology.} In the brain we observe that the topologically central region (MD system; 0) takes a generalist role during problem solving, with more peripheral regions taking more specialised roles (1-8). We see these findings replicated in our model across the skill distributions (left) and their respective specialisation levels averaged across multiple optimisations from random initialisation (right). Error bars on figures represent standard errors.}
    \label{fig:N1_1_grid}
\end{figure}

\subsubsection{Three-tier organisation of the Multiple-Demand System}

Beyond the binary centre/periphery distinction, newer neuroimaging studies based on new high-resolution imaging protocols identify a more granular three-tier organisation in the across cortex in the context of the Multiple-demand system (\cref{sec:findings_MD}): a highly connected core that participates broadly across cognitive demands, a penumbra of regions with intermediate connectivity and moderate domain generality, and a peripheral belt of narrowly tuned, domain-specific regions \cite{assem2020domain, duncan2025construction}. Specialisation increases monotonically from core to penumbra to periphery, mirroring the decrease in graph centrality.

We test whether this graded structure/function relationship emerges from our model when agents are arranged in a matching three-tier topology. Each agent plays the role of a brain region: agent~0 serves as the core, connected to eight penumbra agents, each of which links to three peripheral agents that have no other connections (\cref{fig:N1_penumbra_topology}). The workload is a fixed mixture of Gaussian densities. We optimise skill densities of all agents under a high resource constraint and measure the specialisation of each agent, averaged within each tier. We report the mean and standard error across $100$ independent Monte-Carlo simulations in \cref{fig:N1_penumbra_specialisation}.

\begin{figure}[H]
    \centering
    \includegraphics[width=0.5\linewidth]{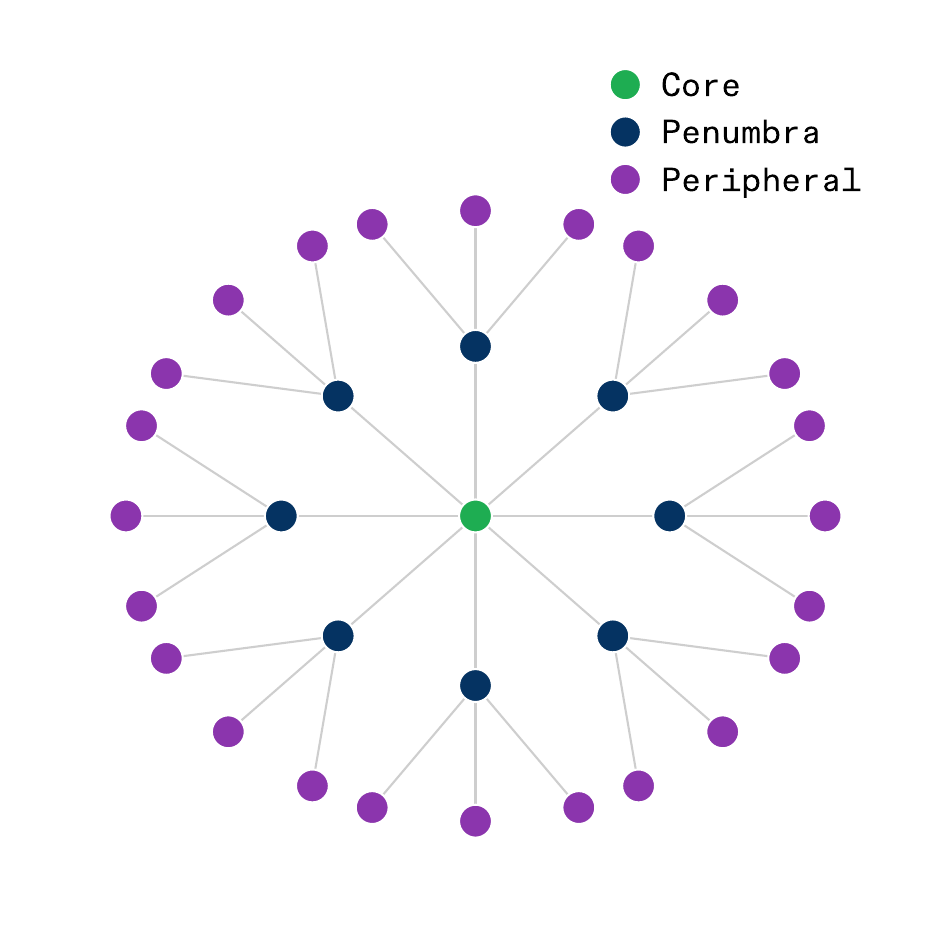}
    \caption{\textbf{Three-tier network topology.} The core agent (green) connects to all penumbra agents (blue), each of which links to three peripheral agents (purple) that have no other connections.}
    \label{fig:N1_penumbra_topology}
\end{figure}

The model recovers the expected gradient: core agents converge on the broadest skill distributions, penumbra agents on intermediate ones, and peripheral agents on the narrowest (\cref{fig:N1_penumbra_specialisation}). The ordering core~$<$~penumbra~$<$~peripheral holds with non-overlapping standard error bars, confirming that the three-tier specialisation hierarchy emerges from the optimisation. These results show that the three-tier specialisation gradient observed in the primate cortex arises in our model from the interplay between network topology and resource-constrained optimisation, extending our earlier replication of the centre/periphery dichotomy to the finer-grained core/penumbra/domain-specific hierarchy of the MD system.
\begin{figure}[H]
    \centering
    \includegraphics[width=0.5\linewidth]{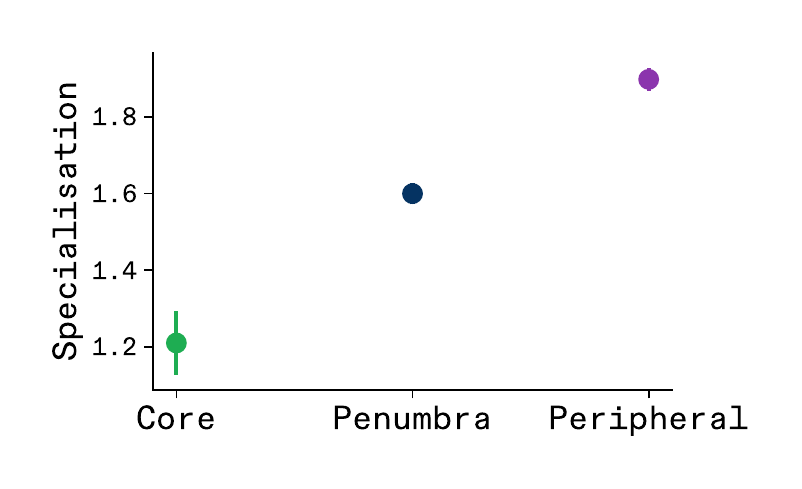}
    \caption{\textbf{Specialisation increases from core to periphery.} Mean specialisation per tier averaged over 100 simulations; vertical bars denote one standard error below and above the mean. Error bars on figures represent standard errors}
    \label{fig:N1_penumbra_specialisation}
\end{figure}

\subsection{Economic systems}

While the analyses in the prior sections focused on findings in the context of biology, distributed production systems are a much more general class of systems. To look beyond ecology and neuroscience, we next focus on human-made systems and hence analyse the behaviour of economic systems. We will first compare our model to productivity gains from global trade (\cref{sec:findings_trade}), followed by findings from interactions of firms and efficient labour markets (\cref{sec:findings_firms} and \cref{sec:findings_labour_market}).

\subsubsection{International collaboration in global trade}
\label{sec:findings_trade}
A foundational result in trade theory is that countries with differing production capacities collectively gain from trade~\cite{Ricardo_2015, Eaton2012}.
The core mechanism is comparative advantage: even when one country is more productive across all goods, both countries benefit if each specialises in the good it produces relatively more efficiently and trades for the remainder. This robust prediction underpins the case for open trade networks over isolated, autarchic production. However we want to note that this model truly assumes maximising production performance and ignores any considerations around sovereignty of production.  Crucially, the gains from trade are only realisable when countries are connected; isolated countries, unable to exchange output, have their production solely dictated by consumption demand and converge on homogeneous, self sufficient production profiles.

We wish to verify whether the model reproduces the classical finding: when two countries are allowed to trade, their collective production and realised consumption exceed those attainable under autarchy. To this end, we set up two scenarios. In the \emph{trade} scenario, countries are represented by two connected agents that optimise their production to meet a bi-Gaussian task representing the demand $W_0$. Here the cost can be seen as a measure of the shortfall between consumption demand and production. In the \emph{autarchy} scenario, the graph of trade connections decomposes into two singleton components. For each scenario $s \in \{\text{autarchy}, \text{trade}\}$, we compute the optimal production $W_{1,s}$ and the realised consumption,
$$\int_0^{2\pi} \min\!\bigl(W_0(\theta),\, W_{1,s}(\theta)\bigr)\, \mathrm{d}\theta,$$
which captures how much of the demanded goods/products have been produced. Both scenarios use 5\,000 gradient descent steps with a learning rate of $0.01$ and early stopping with patience $100$ and a fixed resource constraint.

Figure \ref{fig:production_autarchy_vs_trade} displays the optimised production profiles. Under autarchy (top row), each country converges toward a broad, generalist production capacity that attempts to cover both demand peaks simultaneously. The resulting individual profiles are nearly identical and relatively flat, reflecting the impossibility for a single agent to efficiently serve two separated demand modes. Under trade (bottom row), each country instead specialises its production toward one of the two demand peaks. The production in the trade scenario covers the demand more closely and with greater amplitude. Figure \ref{fig:realised_consumption} further shows the realised consumption for both scenarios. The shaded area, representing the realised consumption, is strictly larger under trade than under autarchy.
This confirms that enabling exchanging production outcomes, allows the two countries to collectively satisfy a greater share of consumption demand via different different and higher specialisation.

\begin{figure}[H]
    \centering
    \includegraphics[width=0.8\linewidth]{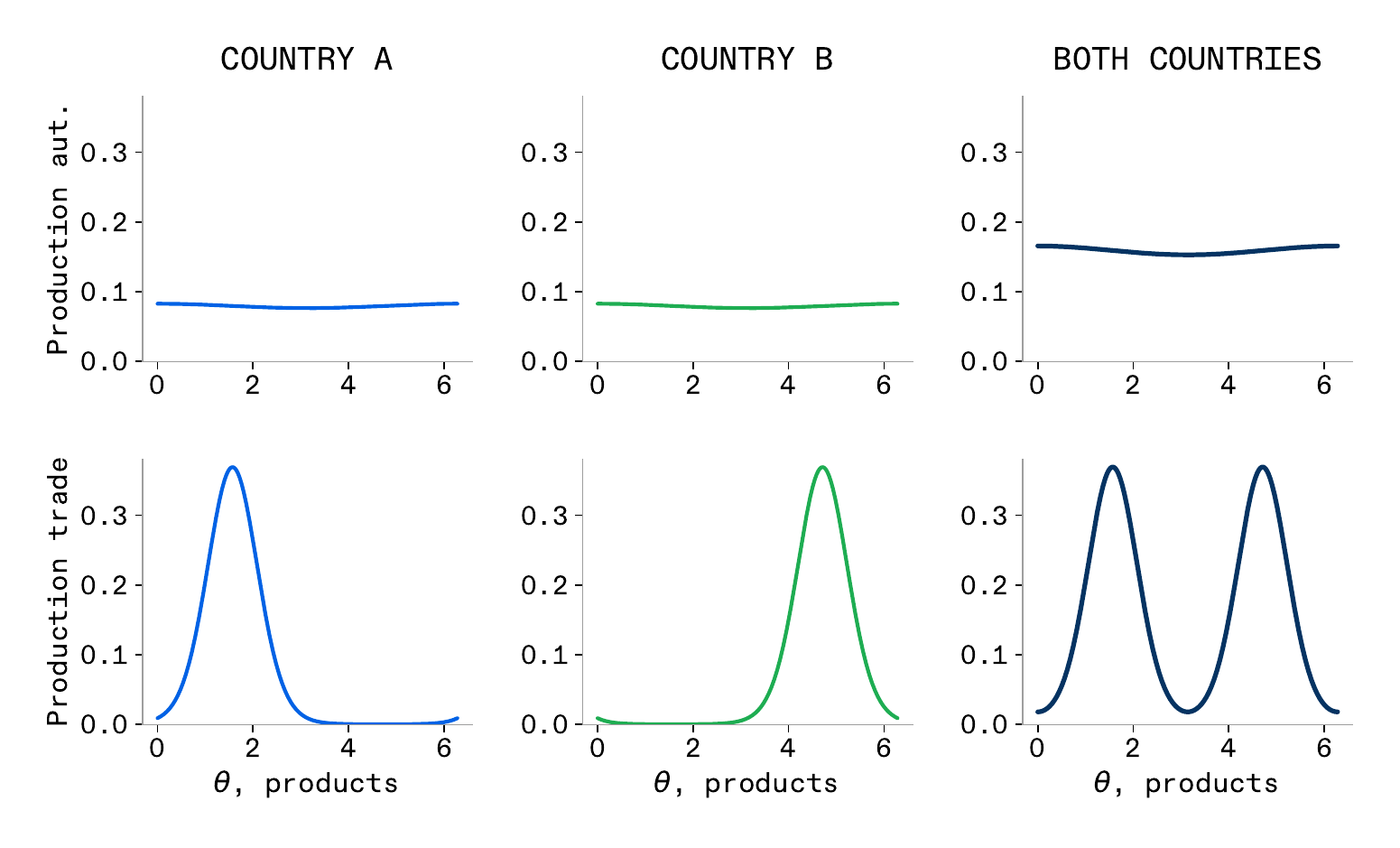}
    \caption{\textbf{Production profiles of two countries.} Top row: autarchy. Bottom row: trade. Under trade, each country specialises toward a distinct demand peak, yielding higher combined coverage.}
    \label{fig:production_autarchy_vs_trade}
\end{figure}
\begin{figure}[H]
    \centering
    \includegraphics[width=0.8\linewidth]{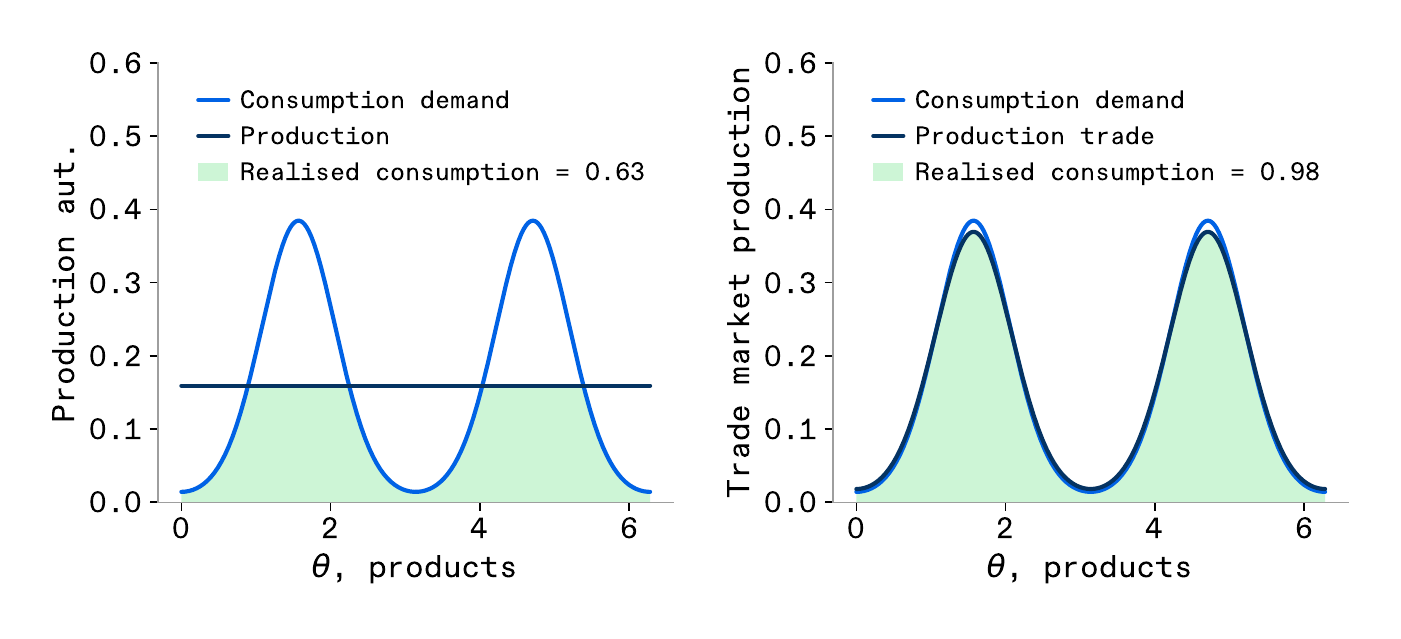}
    \caption{\textbf{Trading between countries increase realised consumption.} Realised consumption (shaded area) under autarchy (left) and trade (right). The integral of $\min(d, W)$ is strictly larger under trade, confirming collective gains from specialisation and exchange.}
    \label{fig:realised_consumption}
\end{figure}

These results demonstrate that our production model reproduces the classical trade theory prediction: countries with differing production capacities collectively gain from trade. The mechanism is consistent with Ricardian comparative advantage. When countries are connected and can coordinate production, each specialises in the region of the product space where it holds a relative efficiency advantage, and the resulting exchange yields higher aggregate consumption than isolated, self sufficient production.

\subsubsection{Firm-level production}
\label{sec:findings_firms}

On a smaller scale, foundations in organisational economics and social science emphasise that the optimal configuration of a firm relies heavily on how knowledge is distributed across agents. \cite{Garicano2000} demonstrates that diverse agents allow for the efficient organisation of knowledge: thanks to communication, generalists handle routine tasks while rare, complex problems are resolved by specialists, thereby minimising the costs of problem-solving. Complementing this,
\cite{HongPage2001, HongPage2004} establish that cognitive diversity between agents is a primary driver of performance in complex problem-solving environments. Their research concludes that a diverse group of problem-solvers can consistently outperform a group of high-ability identical individuals. Likewise, in the context of firms, research on productivity has shown that a firm of similarly talented workers can increase its output first by introducing diversity of skills across workers, and then further by increasing the degree to which individual workers specialise \citep{Freund2025}. The intuition is that workers with complementary, non-overlapping skills divide the production space efficiently: each unit of production capacity goes to the worker best suited to it, with no waste from redundant overlap. We argue that our model of production systems is capable of reproducing this finding.

To test this, we design a Monte Carlo experiment. We consider a firm modelled by a production system of 6 agents who collaborate in a circular topology ($Q = J_6$) and resource constraint equal to twice the number of agents. Each agent $i$ is characterised by its skill centre $\mu_i$ and skill breadth $\sigma_i$, optimised to minimise the local computation mismatch against a given task. To control the degree of workforce homogeneity, we use a second order prior $\mu^{(2)}_i = \pi$ for all $i$, and vary a stiffness parameter that penalises deviations from this common prior. The optimisation loss is therefore: 
$$\mathcal{L} = \mathcal{L}_m + \mathcal{L}_s,$$
with $\mathcal{L}_s\left(\mu, \mu^{(2)}\right) = \frac{\lambda_s}{N} \sum_{i = 1}^N d_{\circ}\left(\mu_i, \mu^{(2)}_i\right).$ At high stiffness, agents are constrained to similar skill density mean; as the constraint relaxes, they are free to diversify. We sample $50$ random tasks, each a Gaussian mixture with $1$ to $4$ modes drawn uniformly on $[0, 2\pi)$ with standard deviations in $[0.2, 1.5]$. For each task, we sweep across all stiffness values and record the resulting heterogeneity, average specialisation, and task productivity $1 / \mathcal{L}_m$. For each task, we compute the Spearman rank correlation between productivity and, respectively, heterogeneity and specialisation across the stiffness sweep. A one-sample $t$-test then assesses whether the mean correlation across tasks is significantly positive.

The results, shown in \cref{fig:econ2_mc}, are unambiguous. The mean Spearman correlation between heterogeneity and productivity is $0.822$ ($p = 9.12 \times 10^{-39}$), and between specialisation and productivity $0.763$ ($p = 1.58 \times 10^{-28}$). Across all $50$ tasks, relaxing the homogeneity constraint consistently allows agents to differentiate their skill centres and sharpen their expertise, which in turn reduces the mismatch with the task demand and raises productivity. The effect is robust: not a single task violates the positive association.
\begin{figure}[H]
    \centering
    \includegraphics[width=0.8\linewidth]{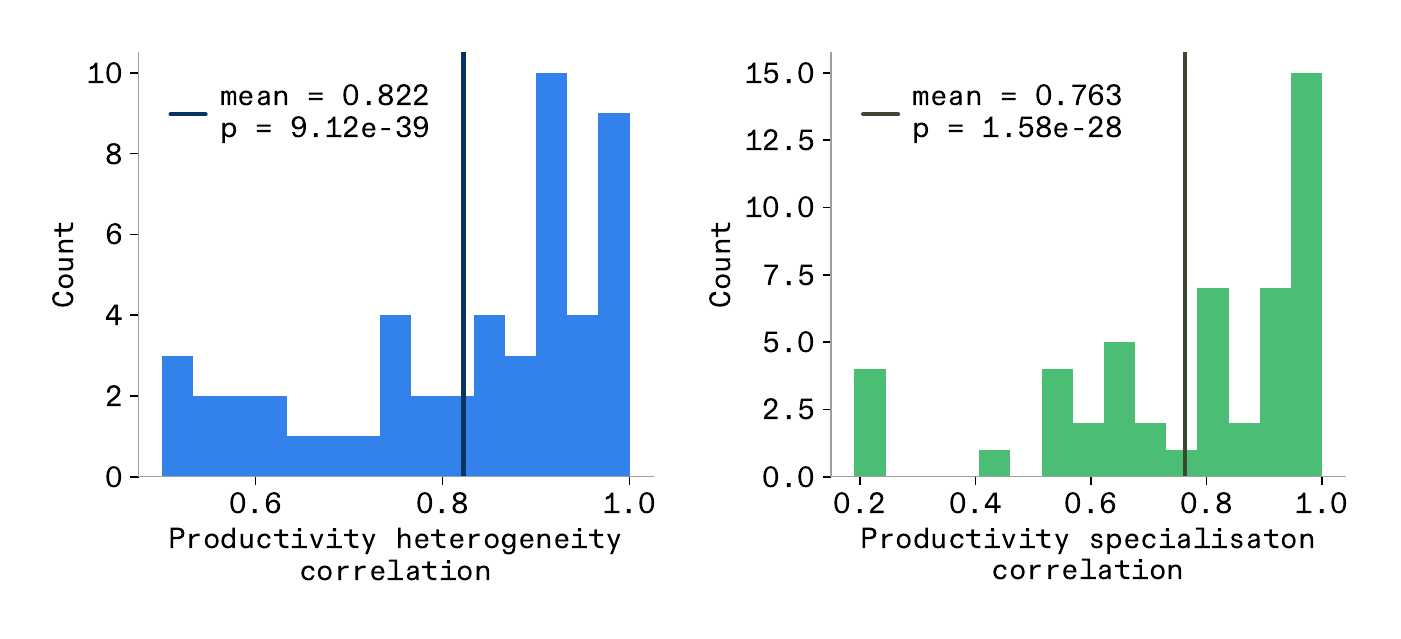}
    \caption{\textbf{Productivity is positively correlated with heterogeneity and specialisation.} Monte Carlo validation ($50$ random tasks). Distribution of Spearman rank correlations between productivity and heterogeneity (left) or specialisation (right) across a stiffness sweep. Vertical lines indicate the sample mean; both are significantly positive.}
    \label{fig:econ2_mc}
\end{figure}

These results confirm that our model recovers an established finding in economics: productivity grows when a workforce diversifies its competences and when individual workers deepen their specialisation. More broadly, this exercise illustrates that our system of $N$ optimised agents endowed with circular skill distributions, interacting through a communication matrix, and minimising a shared objective, behaves as a stylised firm whose production function exhibits the same qualitative properties documented in the organisational economics literature. The model can therefore serve as a microfounded framework in which classical results on the division of labour, comparative advantage, and team composition emerge endogenously from the optimisation of individual skill profiles.
The relationship between heterogeneity and specialisation is not symmetric, however. The reason is geometric: since individual production functions are wrapped Gaussian densities, workers with similar skill profiles will always have overlapping areas under their production functions, and any production unit in that overlap is wastefully competed for rather than efficiently divided.

\subsubsection{Labour market interactions}
\label{sec:findings_labour_market}

Economic organisations are not isolated systems as considered in \cref{sec:findings_firms} but embedded in an environment that acts on them. Research on organisational economics within labour markets highlights that the degree of worker specialisation depends on the efficiency of workers/firms matching in labour markets. Dating back to Adam Smith, it is acknowledged that larger markets increase worker specialisation, as workers can more easily find a matching firm for their specialised abilities, and reciprocally as firms can more easily find workers with suited specialised abilities \citep{Kim1989, GaricanoHubbard2009}. When firms actively search for workers who complement their existing workforce, the specialisation of hired workers increases as labour markets become more efficient and search becomes less costly \citep{Freund2025}. These connect to the findings of the previous \cref{sec:findings_trade}: just as network connectivity between countries enabled complementary specialisation to emerge, a frictionless labour market can be understood as a high-connectivity network in which workers and firms can readily reach the most complementary counterparts available. Conversely, a highly inefficient market introduces a cost on worker search, constraining firms to nearby, less specialised hires and suppressing the complementarity gains described in \cref{sec:findings_firms}.
Co-worker search is not a native mechanism of our model, but the core prediction can be approximated by treating communication between workers as costly. 

We fix a production system, subject to a resource constraint of 4, composed of two agents connected by $J_2$ and a task as a bi-Gaussian with peaks of equal heights at $(\pi/2, 3\pi/2)$. We consider a range of communication costs parameterised by a friction parameter $\lambda_c$ defined as follow, 
$$\mathcal{L}_c(s;Q) = \lambda_c \lVert Q \odot D(s)\rVert_1 = 2 \lambda_c \frac{d_{\circ}(\mu_1, \mu_2)}{\sqrt{\sigma_1^2 + \sigma_2^2}},$$
and optimise the production with our gradient-based optimisation algorithm. For each solution we measure the complementarity using the circular distance between the agents $d_{\circ}(\mu_1, \mu_2)$ as a proxy, with $\pi$ being the maximum complementarity, as it corresponds to maximum performance where agents are optimally located below the respective peaks of the bi-Gaussian.
We observe three regimes, depicted in \cref{fig:econ3_4}. In a low friction labour market, the firm converges on two maximally complementary specialised agents. As friction increases, the optimal agents are progressively less specialised, reflecting the firm settling for a more reachable but less complementary worker. In maximally inefficient markets, the two agents become generalists and productivity gains from specialisation described in the previous section are lost entirely, mirroring the autarchy regime in \cref{sec:findings_trade}.
\begin{figure}[H]
    \centering
    \includegraphics[width=0.5\textwidth]{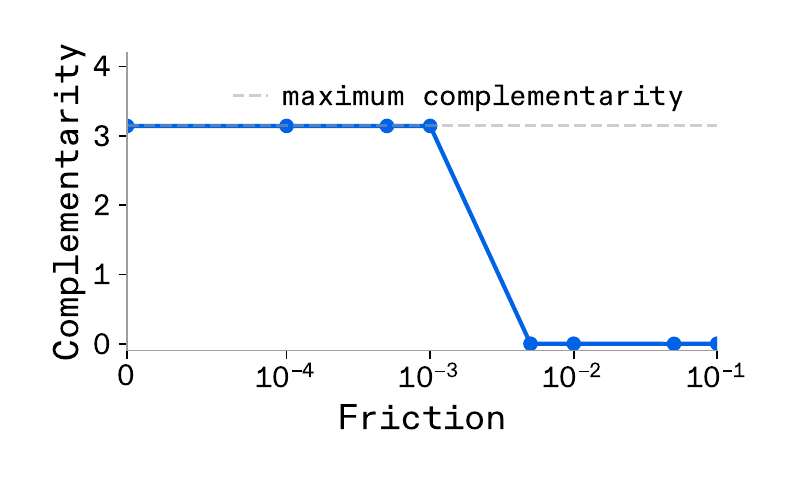}
    \caption{\textbf{Labour market efficiency determines the specialisation of hired
    co-workers.} Specialisation of the optimal hire as a function of labour market
    inefficiency $\alpha$ (right-to-left: decreasing friction). Three regimes are
    visible: a high-specialisation plateau under efficient markets, a transition region
    of declining specialisation as friction increases, and a low-specialisation plateau
    under maximally inefficient markets where search cost outweighs any complementarity
    gain.}
    \label{fig:econ3_4}
\end{figure}
Our model thus recovers the qualitative prediction of \cite{Freund2025}: labour market efficiency is a precondition for firms to realise the complementarity gains from specialised hiring.

\subsubsection{Optimal portfolio theory}
\label{sec:findings_portfolio}
A cornerstone of modern portfolio theory, formalised by Markowitz \citep{Markowitz1952},
is that diversification is the only free lunch available to a rational investor. By holding a portfolio of securities of diverse sectors, an investor
can reduce the temporal variability of returns without sacrificing the expected return. We verify whether our production system model, specifically its
composition mechanism (teamwork from multiple agents) and its external task (a drifting market evolution), can reproduce this diversification effect, thereby validating these model features in a financial context.

Each security is modelled as an agent whose skill density profile represents the return spectrum over financial sectors of the security. The market's return corresponds to a unimodal Gaussian task that drifts periodically as a wave with an additional Brownian-motion perturbation of small amplitude
($\sigma_{\mathrm{BM}}=0.15$), mimicking the stochastic component of real market evolution. The unimodality is justified by a modelling assumption: at any point in time there exists only one sector that provides more returns than any of the other sectors. We compare two portfolios of
$2$ securities: a homogeneous (non-diversified) portfolio with both securities centred at $\mu=\pi$, and a
heterogeneous (diversified) portfolio with securities at $\mu=\pi/2$ and $\mu=3\pi/2$ (all with
$\sigma=0.5$). To ensure robustness, we repeat the comparison across $16$ market initial conditions drawn from a task catalogue:
unimodal Gaussians whose standard deviations range from low ($\sigma=0.3$) to high ($\sigma=2.0$), reflecting the range from sector-specific shocks to economy-wide fluctuations; a visual of the catalogue is provided in \cref{app_fig:workloads_unimodal}. For each market evolution and for each portfolio we compute the temporal mean return $r = 1/(1+\mathcal{L})$, where $\mathcal{L}$ denotes the optimisation loss, and the coefficient of variation.
\Cref{fig:fin1_bars} displays higher mean returns and lower coefficients of variation for the diversified portfolio across the $16$ evolutions and \cref{tab:fin1_tests} verifies statistical significance via $t$-tests and $p$-values. Consequently, our production-system model reproduces the Markowitz diversification effect: the heterogeneous portfolio delivers comparable or superior expected returns while exhibiting lower return variability across all tested market evolutions.
\begin{figure}[H]
  \centering
  \includegraphics[width=0.9\linewidth]{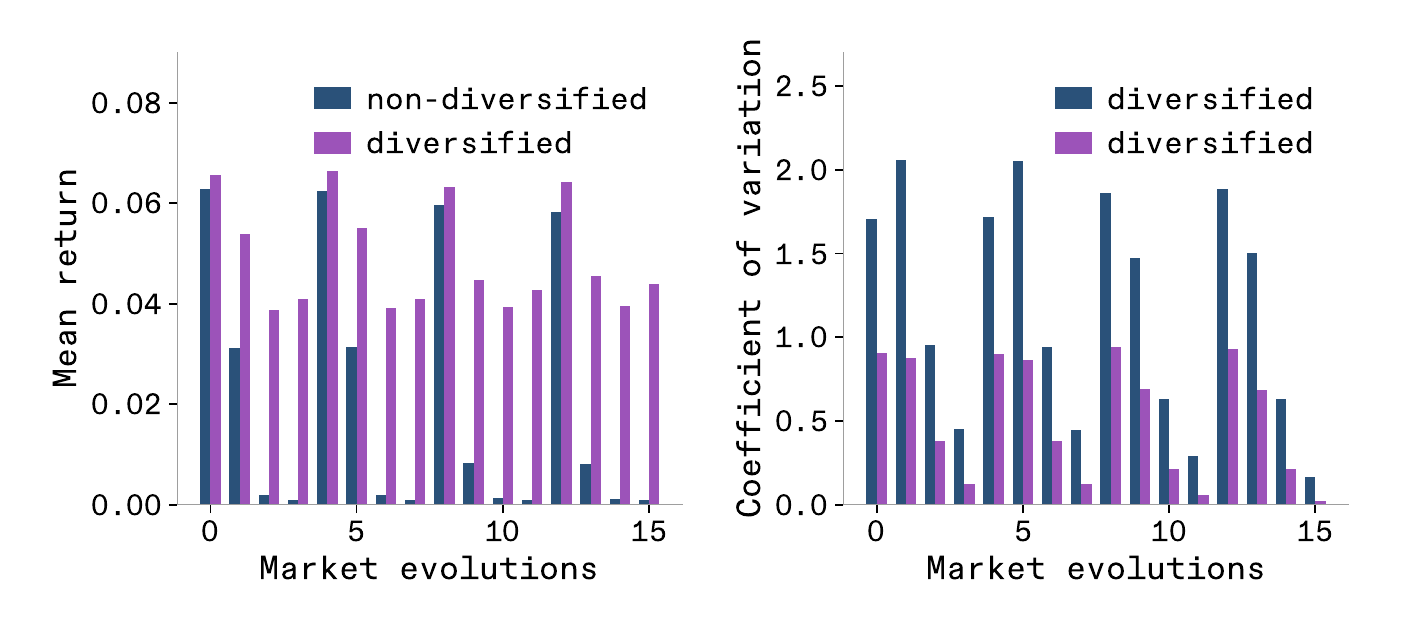}
  \caption{\textbf{Per-market-evolution mean return (left) and coefficient of
    variation (right)} for the homogeneous portfolio and the
    diversified portfolio across $16$ unimodal market conditions
    spanning narrow to broad sector spectra. The diversified portfolio achieves higher
    mean return and lower CV across all evolutions.}
  \label{fig:fin1_bars}
\end{figure}
\begin{table}[H]
  \centering
  \begin{threeparttable}
    \small
    \caption{Statistical tests across different market evolutions.}
    \label{tab:fin1_tests}
    \begin{tabular}{lccc}
      \toprule
      \textbf{Metric} & \textbf{Test} & \textbf{Statistic} & \textbf{$p$-value} \\
      \midrule
      Mean return
        & Paired $t$ & $t=7.31$ & $2.6\times10^{-6}$ \\
      Coeff.\ of variation\tnote{1}
        & Paired $t$ & $t=8.00$ & $8.7\times10^{-7}$ \\
      \bottomrule
    \end{tabular}
    \begin{tablenotes}
      \footnotesize
      \item[1] A Wilcoxon signed-rank test on the CVs further confirms significance ($W=0$, $p=3.1\times10^{-5}$).
    \end{tablenotes}
  \end{threeparttable}
\end{table}

\subsection{Computer Science and Neural Networks}

We have replicated a variety of findings across distributed production systems from ecology, neuroscience, and economics. We lastly want to expand this set with findings from computational sciences, as later sections will focus on applying our model in this specific context. We first focus on the central problem of function approximation as it naturally fits our model framing. Then on the language model scaling laws \cite{hoffmann2022}, as many data points have been captured in the context of that specific empirical relationship. These will be followed by a study of robustness of learning in a more computational neuroscience context.

\subsubsection{Performance and efficiency of learning}

\label{sec:findings_perf_eff_learn}
A wide literature, spanning fields such as AI, ML, and mathematics, focuses on approximation properties of families of functions. Most famous results take the form of convergence with scale, and sometimes lead to universal approximation theorems \citep{Cybenko1989}. Other studies, such as overcomplete-basis theory \citep{okoudjou2016}, take into account the composition or structure of the approximating class itself, and do not treat it as a monolithic homogeneous set. Different families of basis or dictionary elements can yield qualitatively different approximation behaviour: for example, nonlinear m‑term approximation theory studies how the error of best approximation depends not just on the number of terms but on the richness and redundancy of the underlying function dictionary and selection strategy \citep{devoreNonlinear1998, GribonvalNielsen2018}. These approaches highlight that approximation rates are sensitive to structural properties such as redundancy, sparsity, and the geometry of the dictionary rather than merely its scaling. In the context of neural-networks, a dictionary corresponds to a set of activation functions. In a recent work, \citep{perez2021neural} demonstrated that neural heterogeneity, defined as the diversity of intrinsic parameters across neurons, yields improvements in learning performance for spiking neural networks, artificial neural networks mimicking the firing mechanisms of biological neurons. Specifically, heterogeneous networks achieve lower test loss for a given architecture size, while their homogeneous counterparts require substantially more parameters to reach equivalent accuracy. We examine whether our production-system model, which encodes no neuroscience-specific mechanism, is sufficient to reproduce these two precise empirical findings.

We model an artificial neural network as a production system of $4$ agents connected in circular topology with interaction matrix $J_4$ whose individual learning functions are wrapped Gaussian distributions. The task, representing the target distribution, is modelled as a mixture of $5$ Gaussian distributions with different peak heights.
All agents share a second-order constraint that we model in a simple way as $\mu^{(2)} = \pi$ called the hardware.
In a first experiment, we sweep a stiffness parameter $\lambda$ that penalises deviation from the hardware, measured as:
$$\mathcal{L}_s\left(\mu, \mu^{(2)}\right) = \frac{\lambda}{2}\left(d_{\circ}\left(\mu_1, \mu^{(2)}\right) + d_{\circ}\left(\mu_2, \mu^{(2)}\right)\right),$$
thereby controlling the degree of post-optimisation heterogeneity while holding resource constraints fixed to $8$. For each stiffness value, the system is optimised, after which we record the heterogeneity and the learning performance $\Pi_m = 1/(1+\mathcal{L}_m)$, where $\mathcal{L}_m$ is the mismatch cost.
In a second experiment, we lock stiffness at its maximum tested value and progressively relax resource constraints, thereby increasing the production capacity provided to the homogeneous system.
The results are reported in \cref{fig:cs_perf} (left) and \cref{fig:cs_perf} (right). In the stiffness sweep (\cref{fig:cs_perf} (left)), learning performance is strongly and positively associated with heterogeneity: an ordinary least-squares regression yields $R^{2}=0.92$ ($\hat{P} = 0.106\,\mathcal{H} + 0.748$), with performance rising from $P=0.734$ at $\mathcal{H}\approx 0$ to $P=0.936$ at $\mathcal{H}=1.91$.
In the resource sweep (\cref{fig:cs_perf} (right)), the homogeneous network requires its resource-constraint parameter to be reduced by nearly half to recover the performance of the heterogeneous system, indicating that a homogeneous system requires a substantially larger resource budget to compensate for the absence of heterogeneity.
\begin{figure}[H]
    \centering
    \begin{minipage}[t]{0.48\textwidth}
        \centering
        \includegraphics[width=\linewidth]{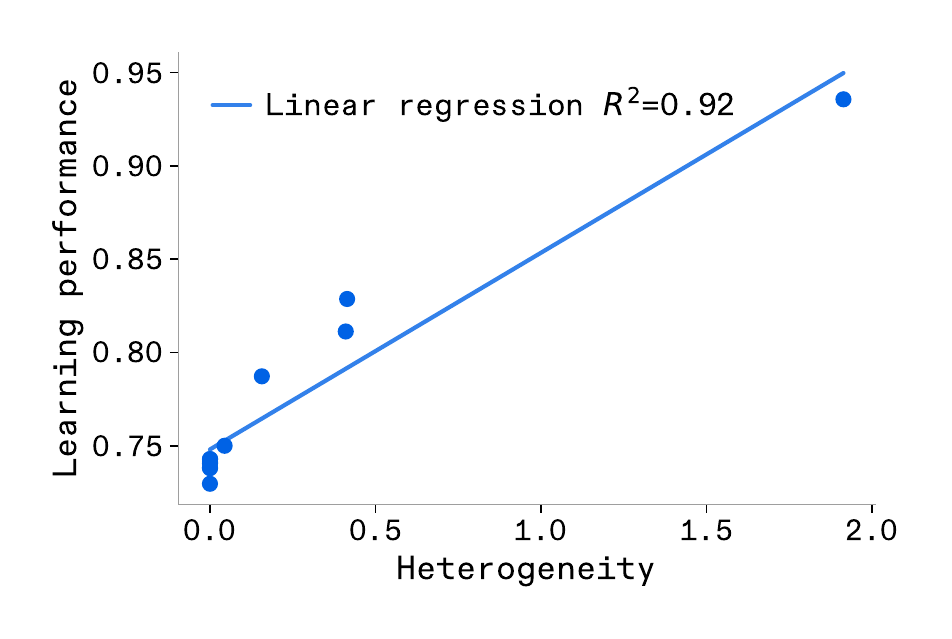}
    \end{minipage}
    \hfill
    \begin{minipage}[t]{0.48\textwidth}
        \centering
        \includegraphics[width=\linewidth]{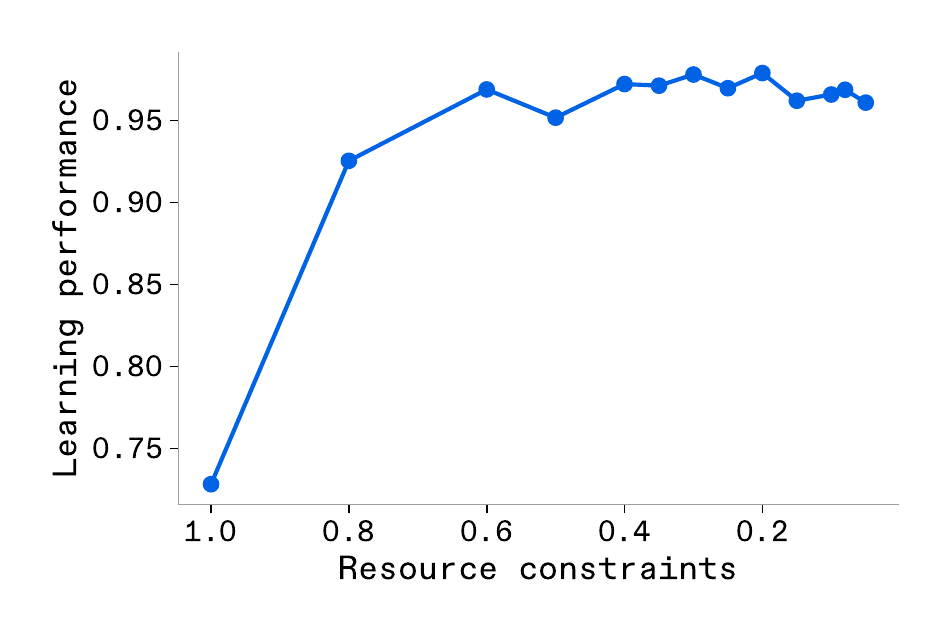}
    \end{minipage}
    \label{fig:cs_perf}
    \caption{\textbf{Learning performance evolution with heterogeneity and resource.} (Left) Learning performance as a function of heterogeneity under varying stiffness. Each point corresponds to a single stiffness value; the regression line ($R^{2}=0.92$) confirms a strong positive association. (Right) Learning performance of the homogeneous network ($s=20$, $\mathcal{H}\approx 0$) as resource constraints are relaxed. Recovering high performance requires a large increase in effective capacity. For readability, the resource constraints on the x-axis correspond to 8 times the resource constraints used in simulation.}
\end{figure}
These results jointly confirm that the production-system model reproduces two core findings of \citep{perez2021neural} without any neuroscience-specific assumptions: heterogeneity among agents directly improves learning performance, and an equivalent homogeneous system must be endowed with a larger resource budget to attain comparable accuracy. The underlying mechanism is generic: diverse agents partition the task space more efficiently, suggesting that the advantage of neural heterogeneity is an instance of a broader organisational principle rather than a property unique to spiking circuits. This will be investigated in \cref{sec:principles_workloads} and \cref{sec:principles_efficiency} more generally.

\subsubsection{Language model scaling laws}
\label{sec:findings_scaling_laws}
The Chinchilla scaling laws \citep{hoffmann2022} describe how the performance of a large language model depends on two resources that can be traded against one another under a fixed compute budget: the number of parameters $N$, which determines the model's capacity to represent knowledge, and the number of training tokens $D$, which determines how thoroughly that capacity is exercised during training. More parameters allow the model to store more distinctions; more data allow it to learn those distinctions more accurately. The empirical relationship between these quantities and the resulting loss $L$ takes the form
\[
\left\{
\begin{aligned}
    C &= C_0 DN \\
    L &= \frac{A}{N^{\alpha}} + \frac{B}{D^{\beta}} + L_0
\end{aligned}
\right.
\]
where $C = C_0 DN$ is the total compute cost, and $A$, $B$, $L_0$, $\alpha$, $\beta$ are constants fit to empirical data. The key insight of Chinchilla is that both terms in $L$ must be reduced jointly: adding parameters without sufficient data leaves capacity underutilised, while adding data without sufficient parameters leaves the model unable to absorb what it sees. In our distributed production model, agents play the role of parameters, and the number of subintervals $D$ used to discretise the demand function plays the role of training tokens, determining the resolution at which the system can match the workload. The compute cost maps onto the inverse of the resource constraints, since covering more of the demand space at finer resolution increases the resource requirement. The same trade-off therefore has a direct structural counterpart: too few agents leaves demand regions unmatched, while too coarse a discretisation of demand prevents the system from identifying where agents are needed. If this analogy is well-founded, the model should recover the Chinchilla functional form without it being imposed by construction.

To test if our model recovers the Chinchilla law, and its implication of joint scaling of data and parameters, we carry out two experiments. We first vary $N$ and $D$ across a range of values. The agents are interacting in a circle topology via $J_N$, subject to a resource constraint proportional to $1/(ND)$ and the task is fixed to a mixture of $5$ Gaussians with different peaks. We use gradient descent to solve for the optimal agents configuration at each $(N, D)$ pair, and record the resulting loss $L(N, D)$ corresponding to the final loss value of the descent. The Chinchilla functional form is then fitted to the data and we record the error. Then, to verify the joint-scaling implication, we consider a production system consisting of only one agent that optimises the same task. We increase the number of discretisation points and reduce the resource constraints proportionally to $1/D$.
\cref{fig:combined_scaling_results} (left) shows the results of the first experiment. The Chinchilla functional form provides a close fit to the simulated $(N, D, L)$ surface. The joint-scaling implication is also recovered, as shown in \cref{fig:combined_scaling_results} (right): with one agent, the loss eventually goes to zero as the number of data increases. However, this comes at the price of high overproduction (i.e. the area below the production function and above the task). The distributed production model thus recovers Chinchilla-type scaling as an emergent property.
\begin{figure}[H]
    \centering
    \begin{subfigure}[t]{0.38\textwidth}
        \centering
        \includegraphics[width=\linewidth]{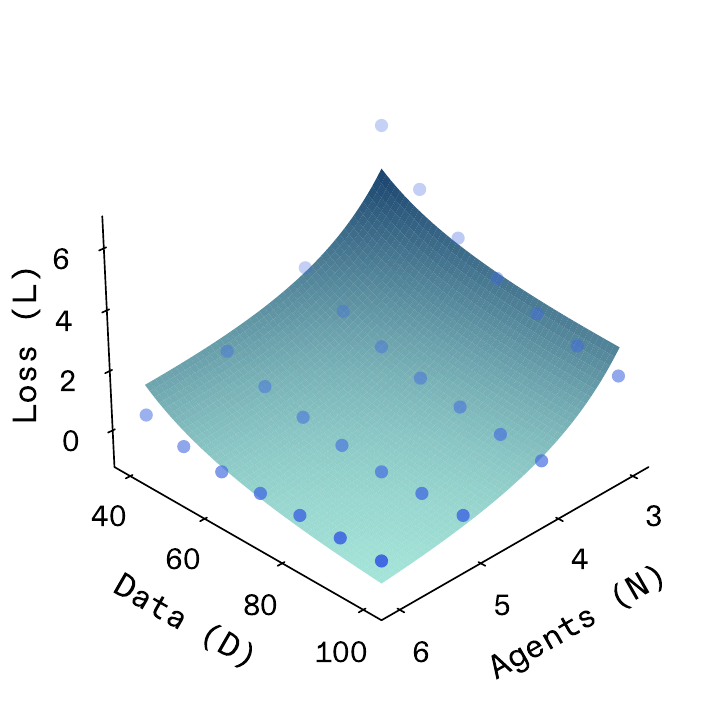}
    \end{subfigure}
    \hfill 
    \begin{subfigure}[t]{0.58\textwidth}
        \centering
        \includegraphics[width=\linewidth]{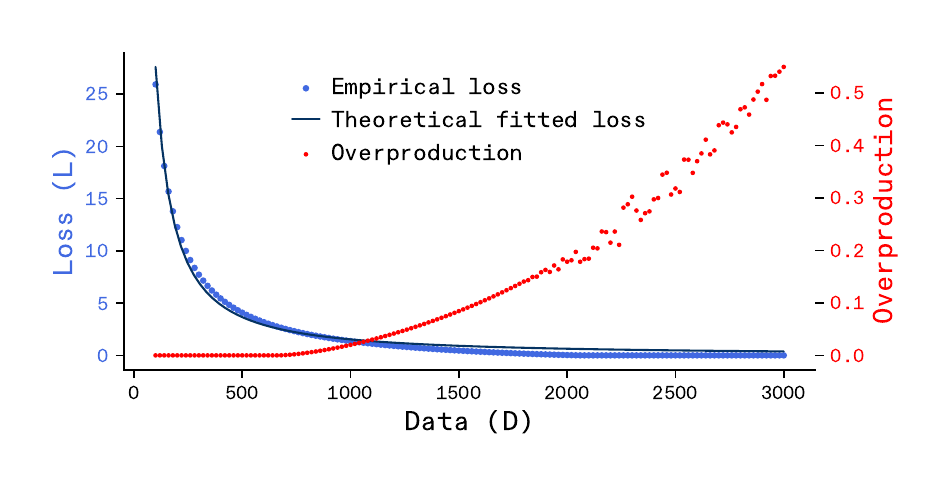}
    \end{subfigure}
    \caption{\textbf{Scaling behaviour of the distributed production model.} The model recovers multi-parameter scaling laws (left) while demonstrating the efficiency trade-offs of homogeneous scaling (right). (Left) Recovery of Chinchilla-type scaling. $L = A/N^{\alpha} + B/D^{\beta} + L_0$ fit to $(N, D, L(N,D))$. (Right) One-parameter scaling. $L = B/D^{\beta} + L_0$ fit to $(D, L(D))$, showing overproduction trade-offs.}
    \label{fig:combined_scaling_results}
\end{figure}

\subsubsection{Robustness of function approximation through heterogeneous processing time scales}
\label{sec:findings_robustness}
A recurring finding across small-scale computational studies, particularly in neuromorphic computing, is that heterogeneity in processing time scales improves the robustness of function approximation \citep{perez2021neural, sun2025algorithm, sun2025exploiting}.
Spiking networks endowed with diverse time constants maintain acceptable approximation quality under distributional shifts of the target signal, whereas homogeneous counterparts degrade sharply. We investigate whether the production-system model, which contains no neuroscience-specific mechanism, is sufficient to reproduce this empirical regularity.

The correspondence between the two frameworks is direct. In the cited works, a target signal is transformed into a distribution of synaptic delays that the network must learn; in our model, this distribution of delays is represented by the task density.
Each neuron maps to an agent, and the agent's skill density plays the role of the neuron's spiking delay distribution.
A homogeneous network, in which all neurons share identical time constants, corresponds to agents with equal skill parameters; a heterogeneous network corresponds to agents whose skill densities are spread across the task domain.
We instantiate a minimal system of $2$ agents communicating via $J_2$ and define two architectures: a homogeneous model ($\mu_1 = \mu_2 = \pi$, $\sigma = 0.5$) and a heterogeneous model ($\mu_1 = \pi/2$, $\mu_2 = 3\pi/2$, $\sigma = 0.5$).
Each architecture is evaluated on two target distributions, a unimodal Gaussian centred at $\pi$ and a bimodal Gaussian with modes at $\pi/2$ and $3\pi/2$, both before and after a distributional shift implemented as a rotation of each task by $+\pi/4$.
Performance is measured as $\Pi = 1/(1+\mathcal{L})$ where $\mathcal{L}$ is the optimisation loss (corresponding to the task/production mismatch cost), and robustness as the normalised absolute change $\rho = |\Pi_{\mathrm{init}} - \Pi_{\mathrm{rot}}|/(\Pi_{\mathrm{init}} + \Pi_{\mathrm{rot}})$, where lower $\rho$ indicates greater stability.
Results are reported in \cref{fig:cs_robustness}.
Both models achieve near-perfect performance on the task whose structure matches their own architecture, but degrade on the mismatched task.
After rotation, the homogeneous model collapses on its previously best-matched task (unimodal: $P$ drops from $1.000$ to $0.047$, $\rho = 0.91$), yielding a mean robustness index of $0.74$.
The heterogeneous model, while also affected by the shift, degrades more gracefully (bimodal: $P$ drops from $1.000$ to $0.095$, $\rho = 0.83$; unimodal: $\rho = 0.40$), yielding a $17\%$ reduction in sensitivity to distributional shift compared to its homogeneous counterpart.
These results confirm that the production system model reproduces the finding of \citep{perez2021neural, sun2025algorithm, sun2025exploiting}: heterogeneity in agent parameters confers robustness to function approximation under perturbation.
The mechanism is structural: diverse agents cover complementary regions of the task space, so that a local shift does not simultaneously degrade all contributions. This mirrors the robustness effects of heterogeneity we already observed in the ecology findings (\cref{sec:findings_ecology}), where community diversity stabilised biomass production under environmental pressure, and in the portfolio design finding (\cref{sec:findings_portfolio}), where heterogeneity of securities within a portfolio enables its robustness.
That the same relationship between heterogeneity and robustness appears across scales and substrates suggests that our model captures a general underlying mechanism linking heterogeneity to robustness in distributed production systems, rather than a domain-specific effect. \cref{sec:principles_robustness} investigates this mechanism systematically from the perspective of the model's general principles. This suggests that the production-system framework can serve as a minimal, analytically tractable model for understanding the functional advantage of heterogeneity in neuromorphic computing.
\begin{figure}[H]
    \centering
    \includegraphics[width=\linewidth]{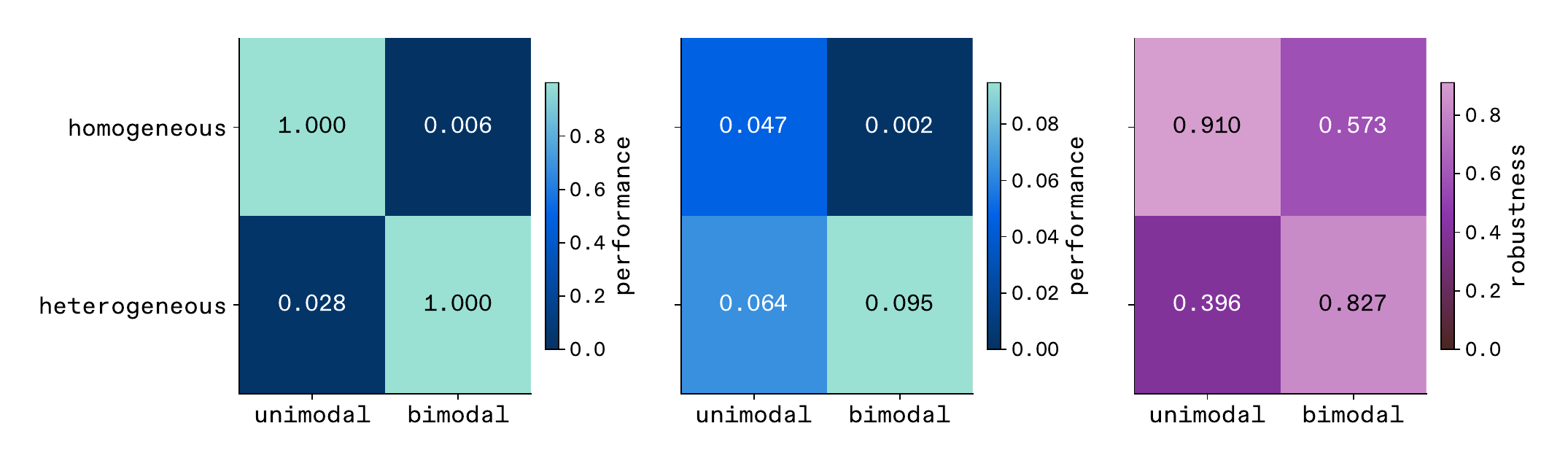}
    \caption{\textbf{Performance and robustness matrices for the homogeneous and heterogeneous models across two tasks.} Left: initial performance. Centre: performance after a $\pi/4$ rotation of the task distribution. Right: robustness index $\rho$ (lower is better). The heterogeneous model exhibits a lower mean robustness index, indicating greater stability under distributional shift.}
    \label{fig:cs_robustness}
\end{figure}

%% file: Text/4_principles.tex
It is now clear that our simple mechanistic model can replicate behaviours of a wide set of distributed processing systems but it is yet unclear which principles drive the behaviour of the model that are shared across the findings replicated in the prior section. To uncover these, we next want to study the model's attractors and fixed points, scaling behaviour, and characterise its main modulating variables, as these have proved useful to understand dynamics in modelling systems such as Game Theory \cite{vonneumann1944, nash1951}. In the following we use a combination of simulation-based and analytical approaches to derive the underlying principles that drive the behaviour of the model system.

Our model, on a fundamental level, tries to capture how agents through their interaction optimally produce an output, where the converged solution of the model tells us about the ideal system level setup given the number of agents, the work they ought to produce, and the means by which they can interact. In the following we hence want to study the principles that the model follows to arrive at the optimal solution based on the variables.

\subsection{Heterogeneous scaling laws: The drive towards heterogeneity in distributed production system}
\label{sec:principles_scaling}
The scaling behaviour of systems has played a prominent role across disciplines \cite{west1997general, kaplan2020scaling} and hence is the first focus of our analyses, where we want to understand what the optimal systems level solution looks like as we add more agents. To study this scaling behaviour we first consider a set of several mixtures of Gaussians with same peak size, forming a series of $8$ tasks with increasing number of peaks, from $1$ to $9$ (c.f. \cref{app_fig:workloads_pics}). We then gradually add increasingly more agents in a circular topology interacting via $J_N$ and $2N$ resource constraints and study the scaling behaviour of the overall system using heterogeneity and specialisation as metrics (see \cref{sec:model_metrics} for details). The results are depicted in \cref{fig:pics_monte_carlo_scaling}. What we observe is that, as the system increases in size, the optimal system's level solution is to increase the system's heterogeneity and then plateau, while the agents added take equally specialised roles after a sharp initial increase. Specifically, we find that
the optimal heterogeneity $\mathcal{H}^*$ and optimal specialisation scale with the number of agents as a log-logistic function: $$\mathcal{H}^*(x) = \frac{\mathcal{H}^{\infty}}{1 + \left(x / a\right)^{-\gamma}},$$
where $x$ is the variable denoting the number of agents, $a, \gamma, \mathcal{H}^{\infty}$ are three positive constants that depend on the shape of the workload and the collaboration between agents, as studied by the following sections. More precisely, $\mathcal{H}^{\infty}$ is the heterogeneity of the limit size system and corresponds to the optimal level of heterogeneity for large systems, and $a, \gamma$ are influenced by the transition dynamic from small to large systems.
For easier observation of the sigmoid scaling, \cref{fig:pics_monte_carlo_heterogeneity_fitted_logN} displays the scaling law as a function of the number of agents with a log scale. Additionally, \cref{app:metric_scaling_behaviour} verifies that this is not a confound of the metric used.
 \begin{figure}[H]
    \centering
    \includegraphics[width=0.9\textwidth]{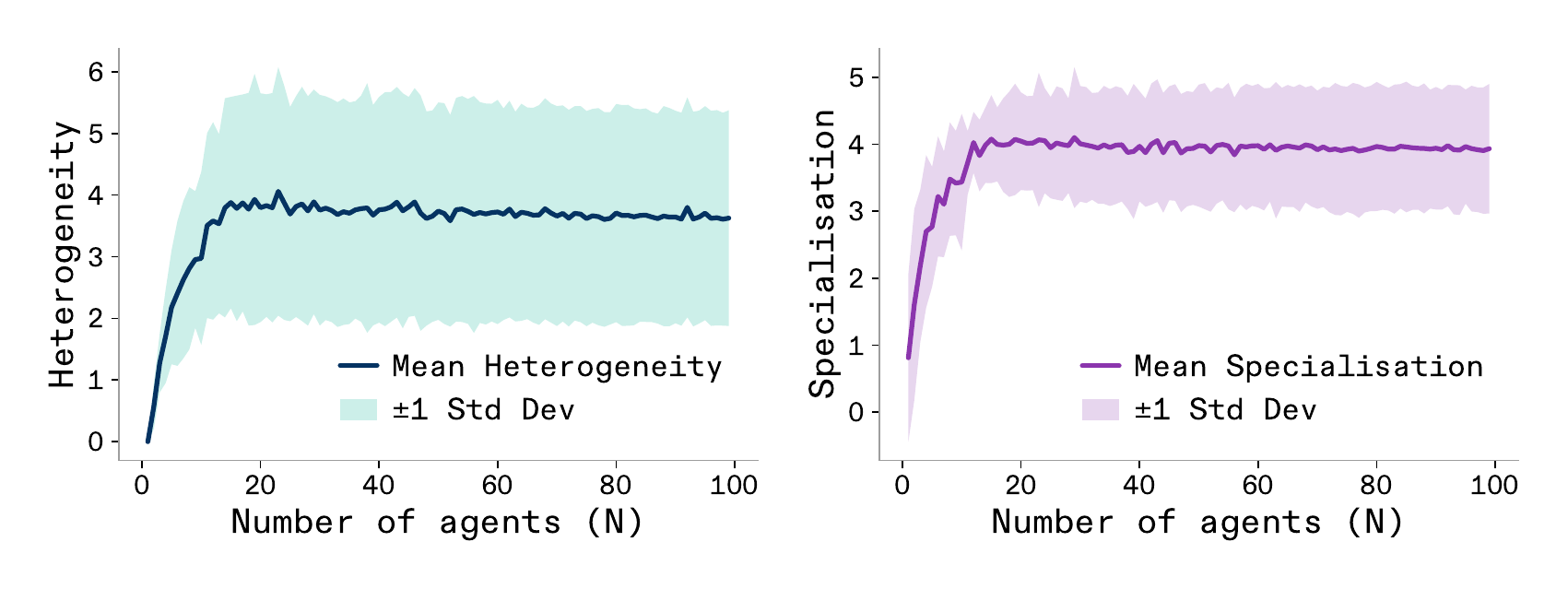}
    \caption{\textbf{Heterogeneity and specialisation scale with the system size.} The line shows the mean value for a system with a specific number of agents averaged across the 8 workloads studied. Shading represents one standard-deviation around the mean.}
    \label{fig:pics_monte_carlo_scaling}
\end{figure}
\begin{figure}[H]
    \centering
    \includegraphics[width=0.5\textwidth]{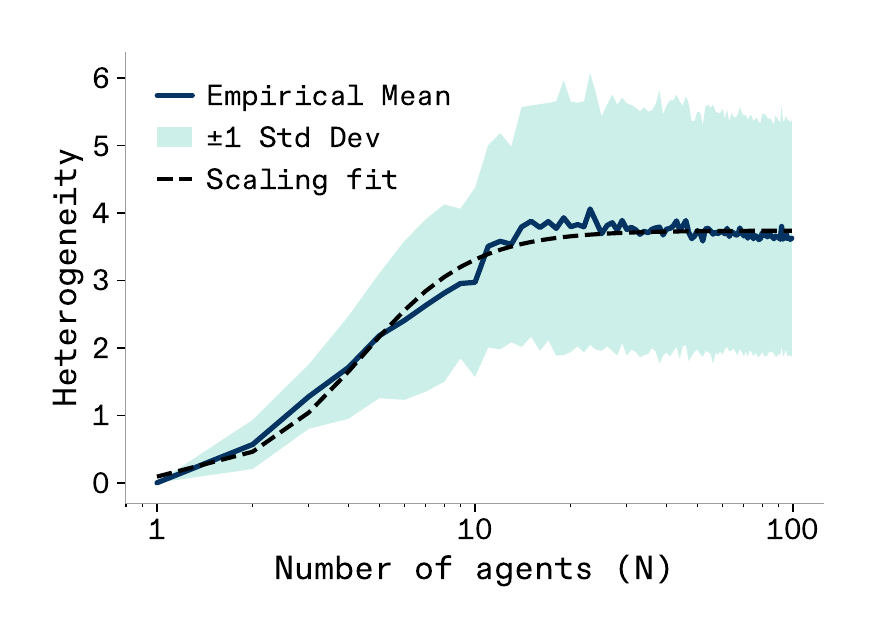}
    \caption{\textbf{Logistic scaling of heterogeneity with system size.} The dashed line is the fitted scaling relationship between heterogeneity and logarithm of system size. Shading represents one standard-deviation around the mean.}
    \label{fig:pics_monte_carlo_heterogeneity_fitted_logN}
\end{figure}
We note that the scaling of heterogeneity with production system size depends on the workload. Indeed the same scaling experiment with mixtures of Gaussians with different peak sizes (c.f. \cref{app_fig:workloads}) leads to different scaling constants, and the overall trend is less well approximated by log-logistic type functions as shown in \cref{fig:catalogue_monte_carlo_scaling}. We investigate in \cref{sec:principles_workloads} the influence of the workload shape on the scaling of system heterogeneity with size. It is clear that for a uniform task, no specialisation is required and a system of generalists, therefore leading to a system with low heterogeneity as generalists share a wide set of skills. Likewise, for a unimodal Gaussian with a high peak, a homogeneous system composed of agents with same high specialisation matches the task optimally, implying no benefits of heterogeneity and constant scaling near $0$. But these cases are extreme, \cref{app:dirac_workload} discusses the unimodal peak Gaussian case, and we focus on a narrower set of workloads (c.f. \cref{app_fig:workloads_pics}) to understand what features of the task influence the heterogeneity of large scale systems.

\begin{figure}[H]
    \centering
    \includegraphics[width=0.9\textwidth]{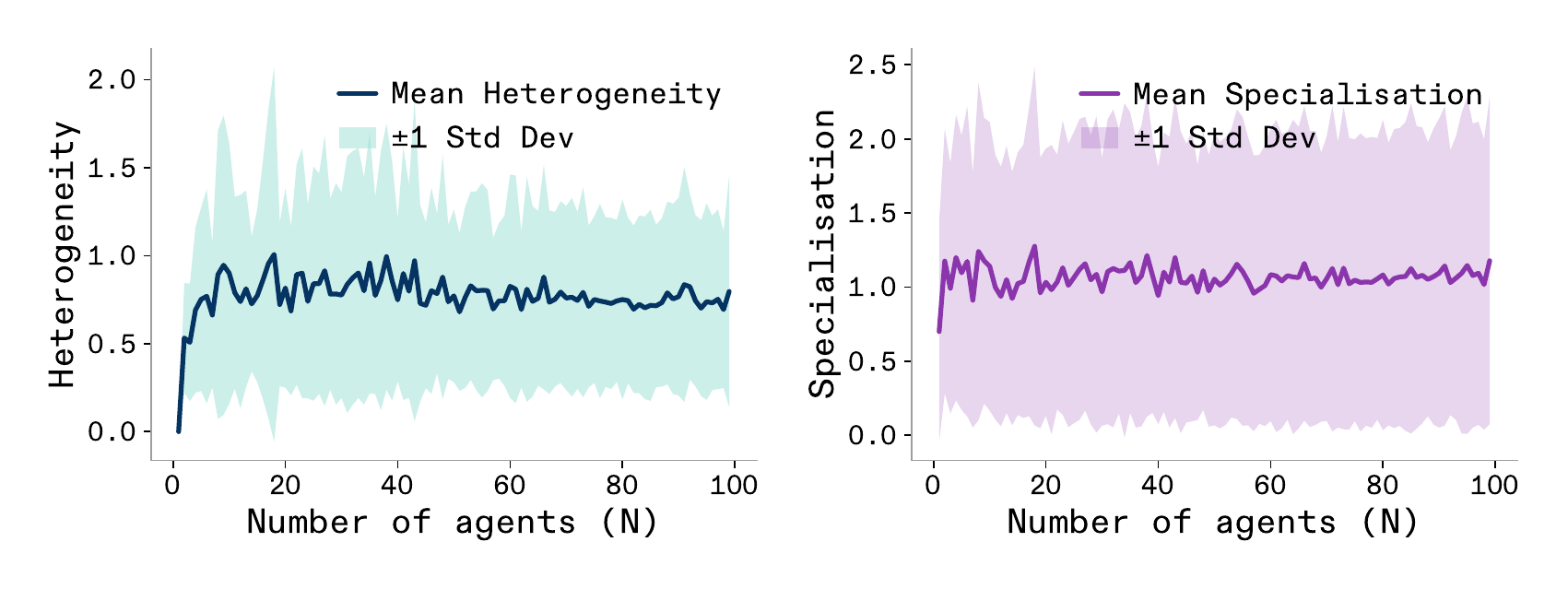}
    \caption{\textbf{Heterogeneity and specialisation scale with the system size.} Each dot is the mean value for a system with a specific number of agents averaged across the 16 workloads studied.}
    \label{fig:catalogue_monte_carlo_scaling}
\end{figure}

\subsection{Heterogeneous scaling laws: The influence of the workload on system-level heterogeneity}
\label{sec:MHP_workloads}
\label{sec:principles_workloads}
While the ideal system-level setup drifts towards heterogeneity across the Gaussian mixtures of the same peak sizes (c.f. \cref{app_fig:workloads_pics}) and other non-extreme cases of workloads, we want to understand how the shape of the task influences the ideal system-level heterogeneity solution. To do so, we utilise a selection of 8 workload shapes consisting of evenly spaced Gaussian peaks on the circular domain $[0, 2\pi]$, ranging from 1 to 8 peaks with fixed $\sigma = 0.25$, as depicted in \cref{app_fig:workloads_pics}. We then scale up the system size as in the prior section and extract the asymptotic heterogeneity value $\mathcal{H}_\infty$ for each workload by fitting $\mathcal{H}(N) = \mathcal{H}_\infty / (1 + (N/a)^{-\gamma})$. To understand how the shape of the workload influences the asymptotic heterogeneity of the system, we characterise it with a number of metrics: the number of peaks, the maximum spread (largest circular distance between peaks), the Shannon and R\'enyi entropies, the power in the first three Fourier modes and their concentration, the low-level heterogeneity $\mathcal{H}_\text{low}$ (a kernel based measure), the peak height, the flatness (kurtosis to variance ratio), and the circular variance (details on all metrics and their calculation are specified in \cref{app:principles_workload_metrics}). We then compute the Pearson correlation between each metric's value across all 8 workloads and the corresponding system asymptotic heterogeneity $\mathcal{H}_\infty$, resulting in one correlation per metric. The correlations across metrics are shown in \cref{fig:principle_metric_correlations} and reveal that the low-level heterogeneity of the workload $\mathcal{H}_\text{low}$ has the highest positive correlation ($r = 0.90$, $p = 0.002$), followed by peak height (negative, $r = -0.89$), circular variance ($r = 0.89$), and entropy ($r = 0.88$), suggesting that the intrinsic heterogeneity of the workload distribution is the single feature most predictive of the system-level heterogeneity that emerges from optimisation.

\begin{figure}[H]
    \centering
    \includegraphics[width=0.8\textwidth]{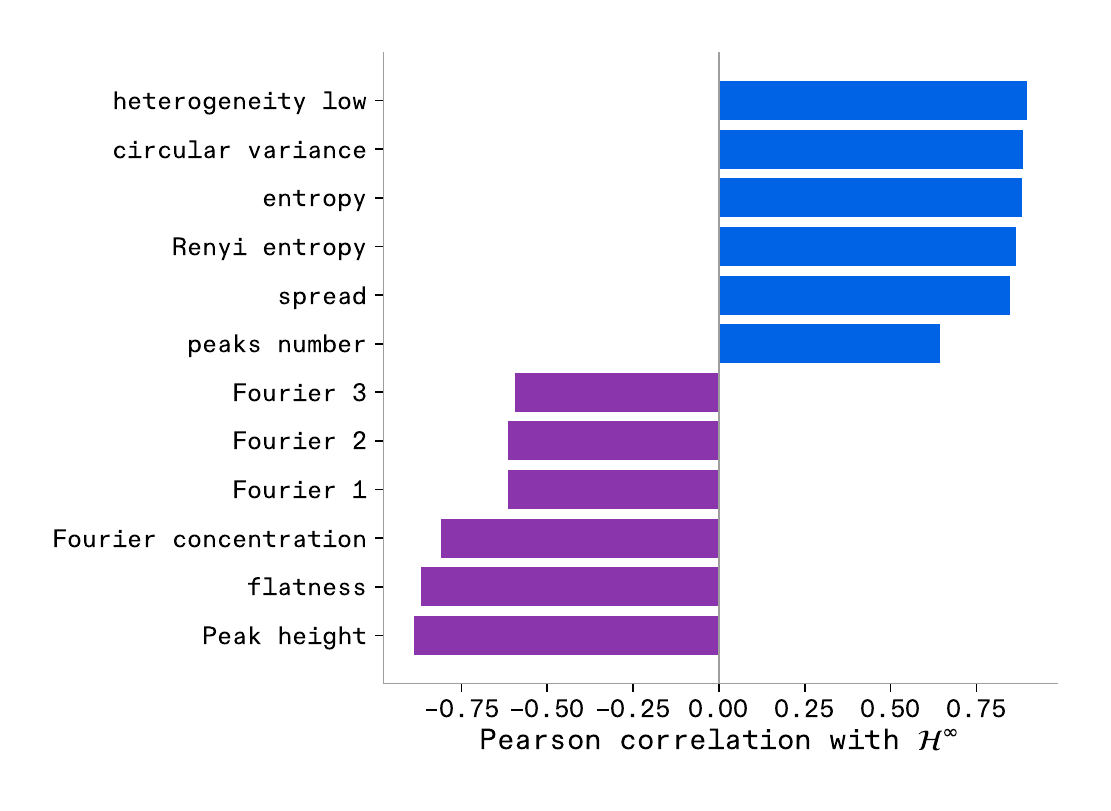}
    \caption{\textbf{The low-level heterogeneity of the workload has the highest correlation with the optimal system heterogeneity.} Pearson correlations between 12 workload shape metrics and the asymptotic system heterogeneity $\mathcal{H}_\infty$, computed across 8 workload profiles (1 to 8 evenly spaced Gaussian peaks). Positive correlations (blue) indicate features that increase with $\mathcal{H}_\infty$; negative correlations (purple) indicate features that decrease. Four features pass Bonferroni correction ($\alpha = 0.004$): $\mathcal{H}_\text{low}$, peak height, circular variance, and entropy.}
    \label{fig:feature_correlations}
    \label{fig:principle_metric_correlations}
\end{figure}
\begin{figure}[H]
    \centering
    \includegraphics[width=0.8\textwidth]{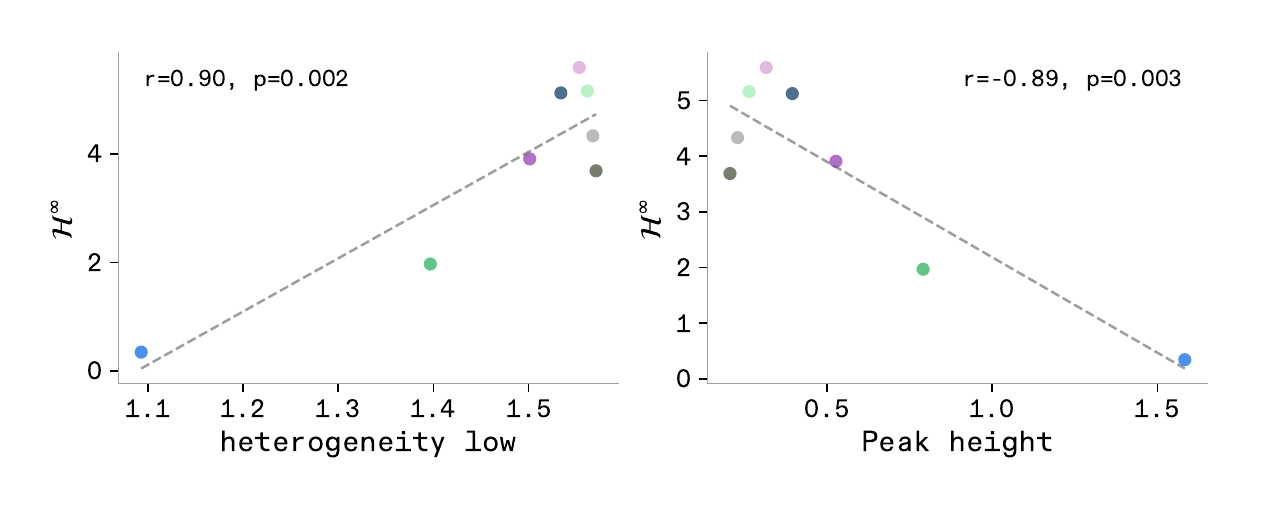}
    \caption{\textbf{The two strongest univariate predictors of system heterogeneity.} Each dot is one workload profile, coloured by its peak family (mono through oct). \textit{(Left)} Low-level heterogeneity $\mathcal{H}_\text{low}$ versus asymptotic system heterogeneity $\mathcal{H}_\infty$ ($R^2 = 0.81$). \textit{(Right)} Peak height versus $\mathcal{H}_\infty$ ($R^2 = 0.79$, negative relationship). Regression lines show the best linear fit. Together these two features confirm that workloads which are themselves more heterogeneous (high $\mathcal{H}_\text{low}$, low peak height) drive the system towards higher heterogeneity.}
    \label{fig:top2_features_vs_xi}
    \label{fig:principle_workload_regression}
\end{figure}

We have focused exclusively on how a non-exhaustive set of specific workload features affects the asymptotic heterogeneity, $\mathcal{H}_{\infty}$. While the workload also influences the transition regime of the scaling dynamics governed by the parameters $a, \gamma > 0$, investigating these effects is beyond the scope of this manuscript.

\subsection{Locality and topology in context of heterogeneity}
\label{sec:principle_topology}
A core quality of our model is that the set of agents that are jointly solving the workload interact via a network topology which allows collaboration between some but not all agents. In \cref{sec:findings_MD} we have already seen that the network topology has unique effects on the specific specialisations that individual agents converge on and here we want to more specifically explore the effect of topology in the context of heterogeneity.

\subsubsection{Fundamental topology and radius of collaboration}
We first investigate how the fundamental topology of the interaction graph, i.e. the decomposition of the system into independently
optimising connected components, constrains the heterogeneity that the system can express. Specifically, we argue that the heterogeneity of
the whole system is bounded above by the heterogeneity achievable within its largest connected component, so that a \emph{radius of
collaboration} (the size of the largest communicating subsystem) emerges and bottlenecks the diversity of the collective to meet the task's requirements as seen in \cref{sec:principles_workloads}.

We consider $N=8=2^{3}$ agents facing a bimodal Gaussian-type workload with peaks at $\mu_{\mathrm{task}} = (\pi/2,\,3\pi/2)$
and standard deviation $\sigma_{\mathrm{task}}=0.5$. The system faces a resource constraint of level $8$. The agents collaborate within their connected components based on the communication matrix $Q$, leading to each component minimising its own loss independently. We construct a hierarchy of four interaction graphs that progressively merge agents into larger collaborating groups, following a balanced binary tree:
\begin{center}
\renewcommand{\arraystretch}{1.2}
\begin{tabular}{clc}
\toprule
Level & Components & Max.\ component size $k$ \\
\midrule
0 & $\{0\},\{1\},\dots,\{7\}$ & 1 \\
1 & $\{0,1\},\{2,3\},\{4,5\},\{6,7\}$ & 2 \\
2 & $\{0,1,2,3\},\{4,5,6,7\}$ & 4 \\
3 & $\{0,1,\dots,7\}$         & 8 \\
\bottomrule
\end{tabular}
\end{center}
At each level, agents within the same component are fully connected ($Q_{ij}=1$) and agents in different components share no link
($Q_{ij}=0$). For each level, we run 20 independent optimisations with different random initialisations and record the system heterogeneity $\mathcal{H}_{\mathrm{sys}}$ (computed over all $N=8$ agents), the within-component heterogeneity $\mathcal{H}_{\mathrm{within}}$ (the mean heterogeneity computed within each connected component), and the termination optimisation loss. We report the median together with the $[q_{10},q_{90}]$ quantile interval over the 20~seeds.
Two behaviours emerge clearly (Figures~\ref{fig:hard_het}--\ref{fig:hard_loss}):
\begin{enumerate}[label=(\roman*)]
  \item System heterogeneity and within-component heterogeneity increase together monotonically with component size.
    At $k=1$ (singletons) every agent is independently incentivised to meet the full bimodal task, producing a homogeneous system (heterogeneity of $\mathcal{H}_{\mathrm{sys}}\approx 0.18$). As pairs form ($k=2$), the system heterogeneity jumps to $\approx 0.97$; it continues to grow at $k=4$ ($\approx 1.2$) and reaches $\approx 1.46$ at $k=8$. This monotonic increase confirms that the largest connected component sets a binding upper bound on the heterogeneity the system can express.
   
  \item Loss decreases with component size.
    The cost drops as shown in \cref{fig:hard_loss}. This confirms that isolated agents are severely suboptimal: a single agent with $1/N$ of the total resources cannot cover a bimodal task, whereas a group of~8 agents can allocate roughly four agents per peak and nearly eliminate the mismatch.
\end{enumerate}

\begin{figure}[H]
\centering
\begin{subfigure}[t]{0.48\linewidth}
  \centering
  \includegraphics[width=\linewidth]{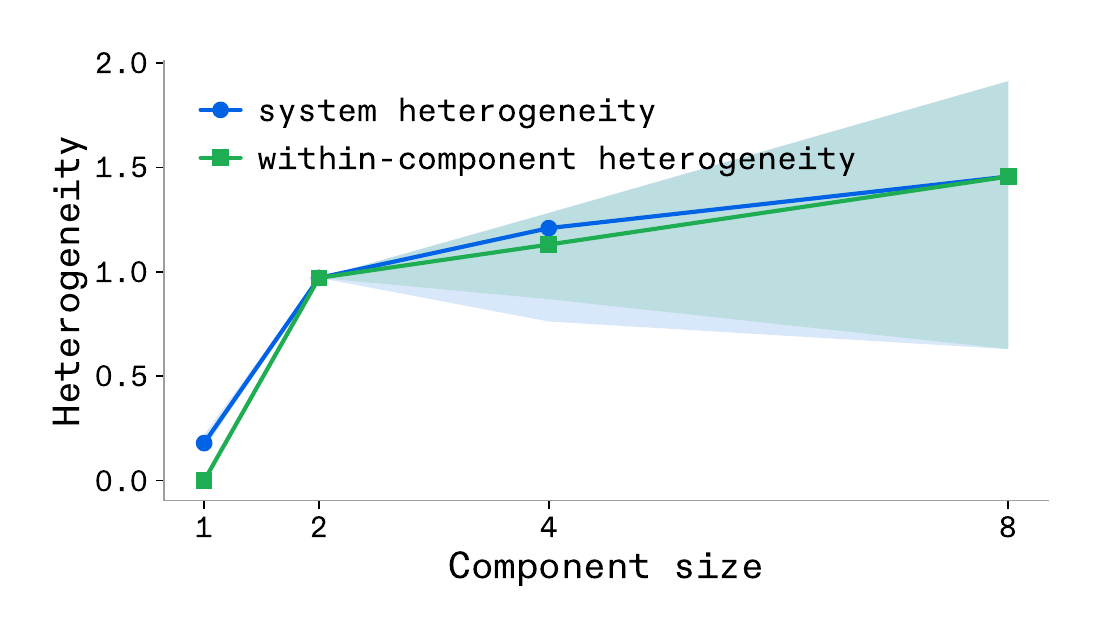}
  \caption{Heterogeneity vs.\ maximum component size.}
  \label{fig:hard_het}
\end{subfigure}
\hfill
\begin{subfigure}[t]{0.48\linewidth}
  \centering
  \includegraphics[width=\linewidth]{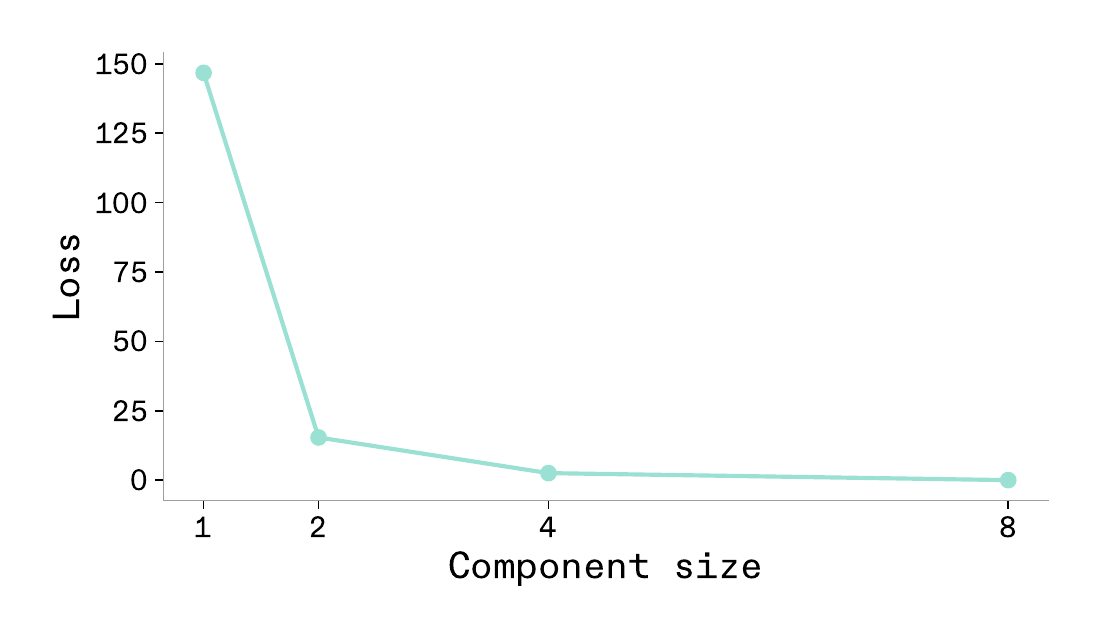}
  \caption{Optimisation loss vs.\ maximum component size.}
  \label{fig:hard_loss}
\end{subfigure}
\caption{\textbf{Effect of the maximum communication component size.} \textit{Left:} heterogeneity as a function of $k$. System heterogeneity (blue) and within-component heterogeneity (green) both increase monotonically with $k$; at $k=8$ they merge since a single component spans the entire system. Shaded bands show $[q_{10},q_{90}]$ over 20 seeds. \textit{Right:} Optimisation loss vs.\ $k$. The cost drops by nearly four orders of magnitude from isolated agents ($k=1$) to full connectivity ($k=8$), showing the large efficiency gain from collaboration.}
\label{fig:hard_combined}
\end{figure}

The specialisation structure underlying these aggregate statistics is visible in \cref{fig:hard_positions} and \ref{fig:hard_spec}, which plot every agent's optimised position~$\mu$ for each component size and their respective level of specialisation. At $k=1$, all agents are generalists, as displayed in \cref{fig:hard_spec} right panel, and thus the position of their means has no effect; they are interchangeable copies of one another. At $k=2$ and above, agents progressively spread to tile the task support, with each connected component developing its own internal division of labour with increasing specialisation. As the component size grows, two simultaneous effects appear: agents spread out, first under the task peaks and then more uniformly around the circle with generalists (\cref{fig:hard_positions} and \cref{fig:hard_spec} left-panel), and their specialisation deepens (c.f. \cref{fig:hard_spec} right-panel), indicating that they sharpen their coverage around their assigned part of the task. This joint increase in positional diversity and specialisation depth is precisely the synergy captured by the heterogeneity metric.
\begin{figure}[H]
  \centering
  \includegraphics[width=0.6\linewidth]{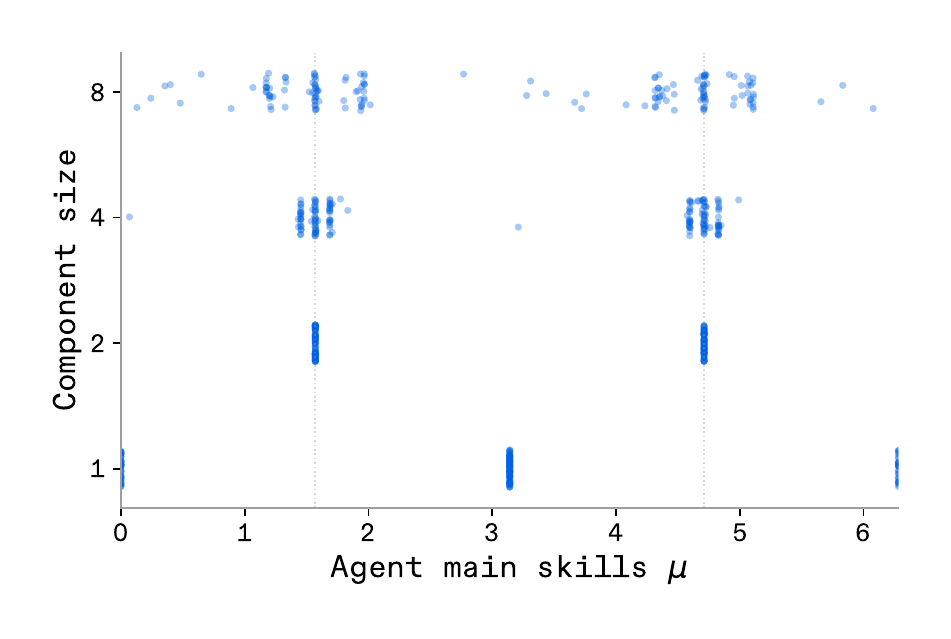}
  \caption{\textbf{Optimised agent positions for each component size.}
    Each dot is one agent from one of 20~Monte-Carlo seeds; vertical jitter separates overlapping points.
    Dotted lines mark the two task peaks.
    At $k=2$ agents are confined near the peaks; larger components enable a broader spread across the task support.}
  \label{fig:hard_positions}
\end{figure}
\begin{figure}[H]
  \centering
  \includegraphics[width=\linewidth]{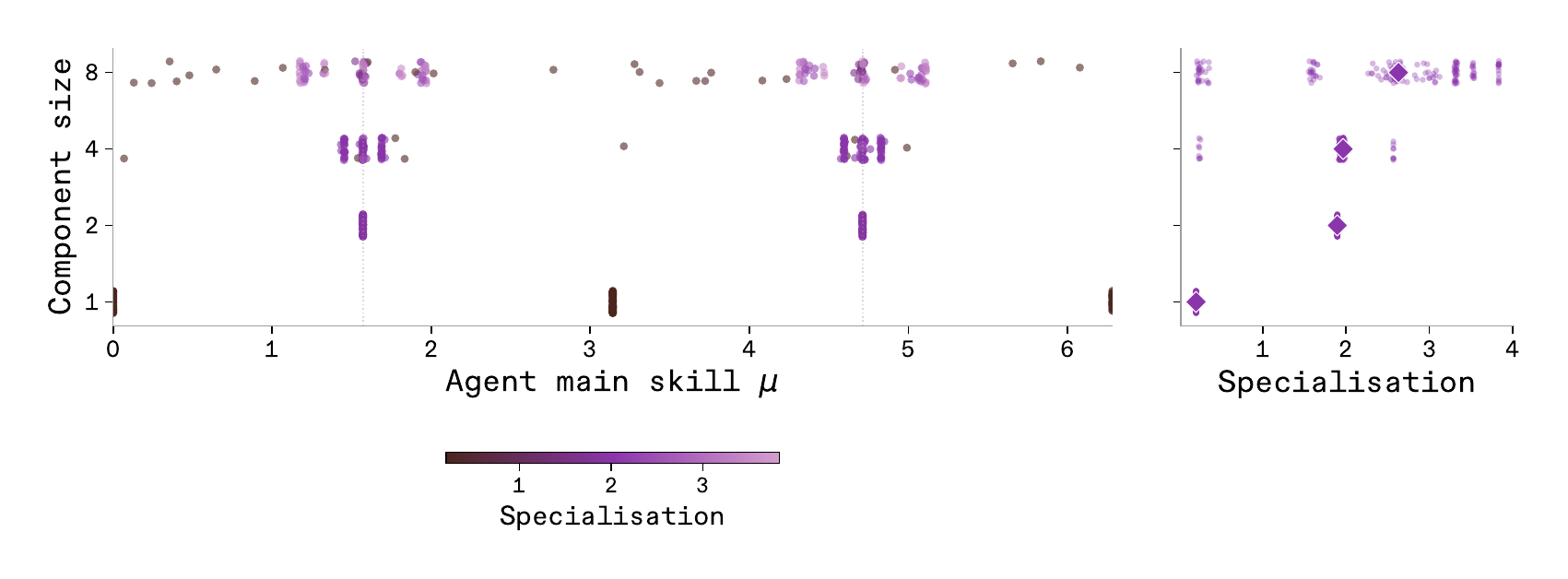}
  \caption{\textbf{Agent positions coloured by specialisation level.}
    \emph{Left:}~each dot is one agent, placed at its optimised~$\mu$; colour intensity encodes breadth $\sigma$
    (pale~$=$~generalist, dark purple~$=$~specialist).
    \emph{Right:}~distribution of $1/\sigma$ for each component size; diamonds mark the median.
    Larger components enable agents to become simultaneously more spread out and more specialised.}
  \label{fig:hard_spec}
\end{figure}
The fundamental topology of the interaction network imposes an effective \emph{radius of collaboration}: the heterogeneity of the system is bounded above by the heterogeneity that the largest connected component can sustain. The analysis reveals that this bound is not merely about positional diversity
($\mu$-separation) but also about specialisation depth ($\sigma$-narrowing): larger components enable agents to become
simultaneously more differentiated and more specialised up to a certain point where the task is covered and new generalists appear. In the extreme limit of complete isolation ($k=1$), all agents converge to the same broad skill profile and the system is homogeneous and highly
suboptimal. In the opposite limit of full connectivity ($k=N$), the bound is not binding and agents freely differentiate to match the task,
achieving both high positional spread and sharp specialisation up to the task requirements. In the intermediate regime the key question is whether the largest component is large enough to express the heterogeneity demanded by the workload: if yes, the lack of global collaboration has little effect on performance; if not, it
imposes a structural sub-optimality that no amount of within-component optimisation can overcome. This mechanism provides a formal rationale for the observation that the diversity of a production system optimised for a task is ultimately set by the span of its communication network, not merely by the number of its constituents.

\subsubsection{Communication cost and spatial distribution of heterogeneity}
In \cref{sec:findings_labour_market}, we have previously seen that we can model environmental effects on the production system with an additional cost of communication, based on the interaction matrix between agents. Next we study the more general effect of the communication cost between agents $\mathcal{L}_c$ on the optimal positioning of agents on the interaction network. 

We consider $N=6$ agents whose workload is a Gaussian mixture on the circle with seven modes placed at $\mu^{\mathrm{task}}_{k}=\tfrac{k\pi}{4}$ for $k\in\{1,\dots,7\}$, each with standard deviation $\sigma^{\mathrm{task}}=0.3$. The communication topology $Q$ consists of two fully connected clusters $\{0,1,2\}$ and $\{3,4,5\}$ of equal size linked by a weak bridge $Q_{2,3}=0.3$. The total cost minimised by the agents is
$$\mathcal{L}(\mu, \sigma) = \mathcal{L}_m(\mu, \sigma) + \lambda_c\mathcal{L}_c(\mu, \sigma),$$
where $\lambda_c\geq 0$ is the communication cost weighting coefficient. Note that we omit fixed parameters from the notation of the optimisation loss.
We sweep $\lambda$ over 21 log-spaced values from $0$ to $0.16$ and, for each value, run 20 independent optimisations with different random initialisations. For each run we record the within-cluster heterogeneity using the average pairwise circular distance between agents in cluster $c\in\{1,2\}$ as a proxy, and the between-cluster dissimilarity
$$\mathcal{H}_{\mathrm{within}}(c) =\frac{1}{3}\!\sum_{i<j\in c} d_{\circ}(\mu_i,\mu_j), \quad d_{\mathrm{between}}=\frac{1}{9}\!\sum_{i\in c_1,\,j\in c_2} d_{\circ}(\mu_i,\mu_j).$$
We report the median together with the $[q_{10},q_{90}]$ quantile interval over the 20~seeds.
The results displayed in \cref{fig:n6_mc_cluster_convergence} reveal a sharp phase transition. For $\lambda\lesssim 3\times10^{-4}$, agents remain diverse within each cluster ($d_{\mathrm{within}}\approx 2$, comparable to the no-cost baseline). Around $\lambda_{\mathrm{c}}\approx 5\times10^{-4}$ a rapid
transition occurs, reflected by wide quantile bands, after which within-cluster distances collapse to near zero ($d_{\mathrm{within}}<10^{-2}$), while the between-cluster distance saturates at $\pi$, its maximum possible value. This plateau persists up to $\lambda_{\mathrm{c}}\approx 5\times10^{-2}$, beyond which even
the between-cluster distance collapses, leading to full homogenisation of all six agents.

Hence, communication costs induce a \emph{locality of heterogeneity}: agents that interact more are pushed towards similar
specialisations, whereas agents interacting via weak links or even separated by the absence of links
preserve complementary skills. The transition is not gradual but phase-transition-like, suggesting the existence of a critical
communication cost above which the pressure to align overwhelms the incentive to specialise. This mechanism provides a parsimonious
explanation for why tightly connected subsystems (e.g. cortical columns, departmental teams, regional economies) tend to be internally
homogeneous yet collectively diverse.
\begin{figure}[H]
    \centering
    \includegraphics[width=0.6\textwidth]{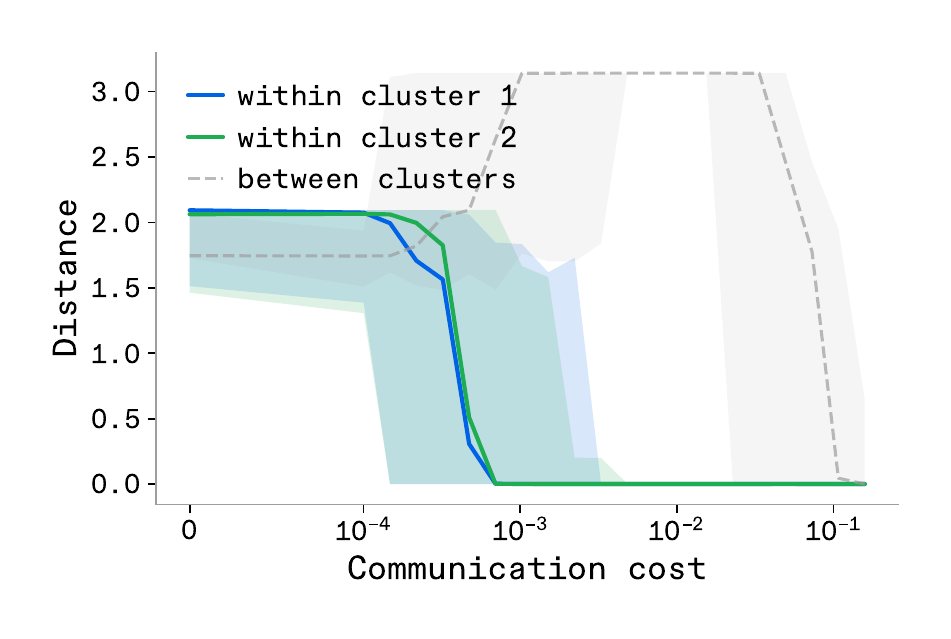}
    \caption{\textbf{Locality of heterogeneity induced by communication costs.} Plotting within- and between-cluster distances across varying communication costs reveals sharp phase transitions from maximal system heterogeneity, to local homogenisation within tightly connected clusters, and ultimately to complete global homogenisation.}
    \label{fig:n6_mc_cluster_convergence}
    \label{fig:principle_communication_cost}
\end{figure}

\subsection{The connection between heterogeneity and efficiency}
\label{sec:principles_efficiency}
The idea of the optimal system-level solution that our model converges on rests on an implicit statement of 'optimal performance given the resources available to the system'. The resources available to the system can for example come in the number of agents and fundamentally the concept of optimal allocation implies the efficient utilisation of such resources. In \cref{sec:findings_perf_eff_learn}, we have observed, in the context of learning functions, that homogeneous production systems needed more resources for the same level of productivity. We want now to more generally test the role of resource constraints on our system. 

To test this we actively control the resource constraints and observe how this changes the optimal system-level composition. We choose a bi-Gaussian task and gradually increase the resource constraints on the system, by dividing the overall output of all agents by an increasingly large number. The system needs to use its production capacity efficiently. The number of agents is set to $N= 16$ and they collaborate via the interaction $J_N$.
\cref{fig:principle_ressource_constraints} shows how, when forcing agents to find the efficient solution by limiting their resources, the system increases its heterogeneity. In any scenario where the agents are forced to find the efficient solution, they converge on a heterogeneous configuration to ensure maximum performance given the resources. We can further highlight this effect by actively controlling the heterogeneity via the second order constraint. We introduce a stiffness parameter $\lambda_s$ that acts as an attracting force towards $\mu^{(2)} = \pi$. The optimisation is therefore,
$$\mathcal{L} = \mathcal{L}_m + \mathcal{L}_s,$$
with, 
$$\mathcal{L}_s\left(\mu, \mu^{(2)}\right) = \frac{\lambda_s}{N} \sum_{i = 1}^Nd_{\circ}\left(\mu_i, \mu^{(2)}\right).$$
In \cref{fig:efficiency_stiffness_log_only}, where we have aggregated the results across resource constraint levels, we observe that heterogeneity and performance decrease concurrently.
To confirm this observation we compute the Pearson correlation coefficient between heterogeneity and performance aggregated across all resource constraint levels. It yields a correlation of $r = 0.976$ ($p < 0.001$). Furthermore, analysing this correlation across specific resource constraint levels reveals a crucial regime-dependent dynamic. When resource constraints are very low (meaning resources are abundant), there is no particular incentive for the system to match the shape of the task. Because resources are plentiful, even a homogeneous system suffices to meet the demand; the optimisation loss is nearly always zero, resulting in a task-agnostic system whose final configuration depends on initialisation. However, as resource constraints increase (meaning resources become increasingly scarce), the system must adapt its compositionality to reproduce the task in order to ensure maximum performance. It is forced to converge on the heterogeneity dictated by the task's specific characteristics. In these constrained regimes, a robust positive correlation emerges, peaking at $r = 0.997$ and remaining strong as constraints tighten further. This confirms that while heterogeneity is linked to performance globally, it becomes strictly necessary for maximising performance when resources are scarce. Thus, heterogeneity emerges from the model under resource constraints, as a way to efficiently optimise the production capacity. This result resonates with the brain being a fundamentally resource-constrained computing system and heterogeneity being the key property of its structure to efficiently process the information demand \citep{vijay2015}.

\begin{figure}[H]
    \centering
    \includegraphics[width=0.55\textwidth]{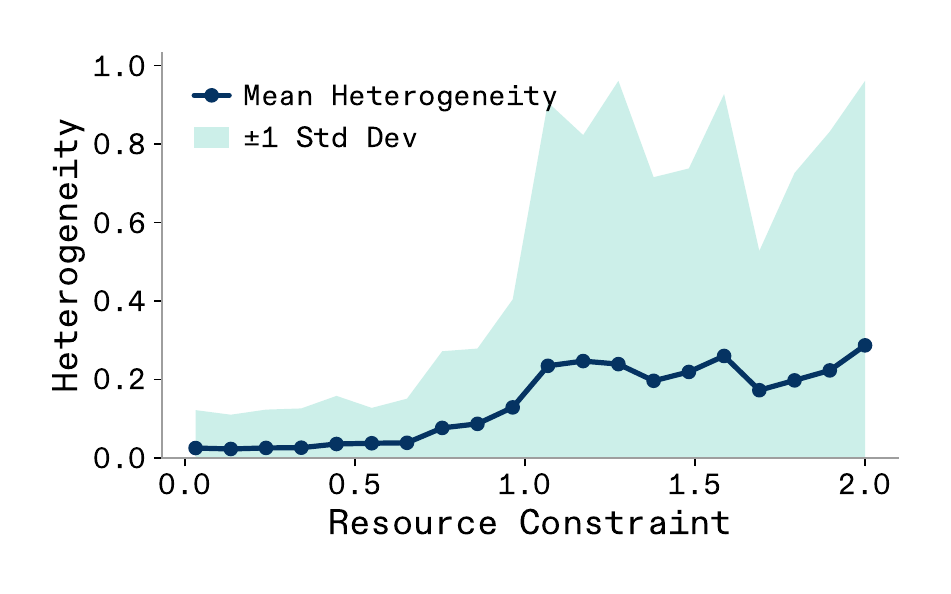}
    \caption{\textbf{The efficient system-level architecture is a heterogeneous architecture.} When forcing agents to find the efficient solution by imposing pressure on their resources, the system increasingly converges upon a heterogeneous setup, across workloads. The shaded area represents one standard deviation around the mean.}
    \label{fig:efficiency_monte_carlo_heterogeneity}
    \label{fig:principle_ressource_constraints}
\end{figure}
\begin{figure}[H]
    \centering
    \includegraphics[width=0.9\textwidth]{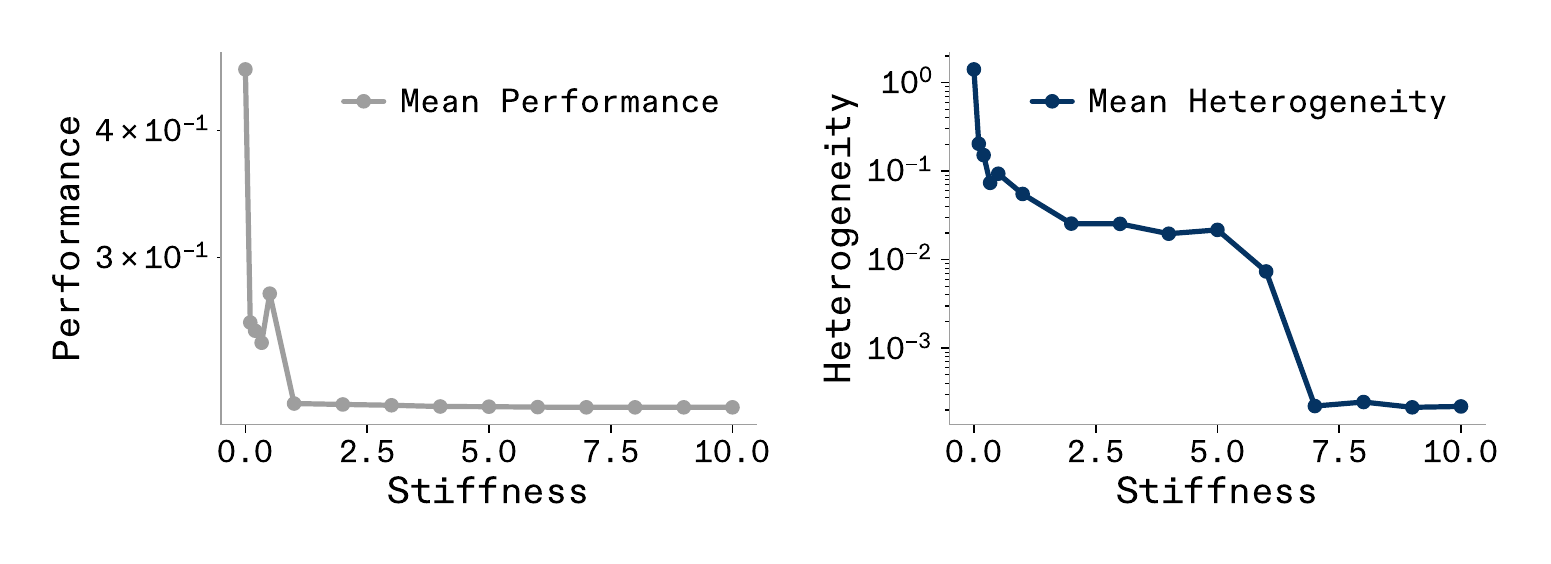}
    \caption{\textbf{Performance and heterogeneity decrease concurrently with stiffness.} Aggregated across varying resource constraints, the system's performance and heterogeneity exhibit joint degradation in response to applied stiffness.}
    \label{fig:efficiency_stiffness_log_only}
\end{figure}

\subsection{Heterogeneity for future risk mitigation and robustness}
\label{sec:principles_robustness}

Part of good system-level design is not only optimising for the problem that is observed right now, but also anticipating how it might change in the future.
Across a wide range of disciplines, prior empirical results suggest that heterogeneity is specifically powerful for such robust future performance.
In ecology (c.f. \cref{sec:findings_ecology}), the spatial insurance hypothesis holds that biodiversity buffers ecosystem functioning against environmental fluctuations because functionally distinct species respond differently to the same perturbation \citep{yachi1999insurance, Loreau2003, schindler2010population, shanafelt2015biodiversity, eisenhauer2023heterogeneity}. In finance (c.f. \cref{sec:findings_portfolio}), \citep{Markowitz1952, Choueifaty2008} showed that combining diverse assets with heterogeneous return profiles reduces portfolio variability, the foundational insight of modern diversification theory. In computational neuroscience (c.f. \cref{sec:findings_robustness}), \citep{perez2021neural} demonstrated that networks with heterogeneous processing time scales approximate functions more robustly than their homogeneous counterparts. These converging observations raise the question of whether heterogeneity specifically links to future risk mitigation as a general structural principle, rather than a domain-specific artifact. We therefore investigate whether heterogeneity of a production system provides robustness under several evolution of the task, abstracting the uncertain evolution of compute demand, the incremental modifications of the demand asked of a firm, or an extreme rare event in the environment imposing a new pressure on a community of species.

We construct two minimal production systems of $N=2$ agents interacting via $J_2$: a homogeneous system ($\mu_1=\mu_2=\pi$, $\sigma=0.5$) and a heterogeneous system ($\mu_1=\pi/2$, $\mu_2=3\pi/2$, $\sigma=0.5$), both operating under identical resource constraints level of $4$. Production is computed as $\Pi(t)=1/(1+\mathcal{L}(t))$, where $\mathcal{L}(t)$ is the cost of unmet time-dependent demand. We subject both systems to three task-evolution regimes that isolate distinct sources of uncertainty: a wave that periodically drift of the task centre ($\mu(t)=\mu_0-\omega t$, $\omega=2\pi$), modelling predictable cyclical shifts in demand, Brownian motion (BM), a pure random walk of the task centre ($\mu(t)=\mu_0+B(t)$, $\sigma_{\mathrm{BM}}=0.15$), modelling unpredictable, chaotic evolution of the task, and
extreme event corresponding to BM augmented by a transient demand spike of duration $0.5$~periods at $t=5$, whose angular location is drawn uniformly on $[0,2\pi)$ and averaged over $10$ independent draws, modelling a rare exogenous shock at an unpredictable position.
Each regime is evaluated on three task catalogues of $16$~tasks each:
a unimodal catalogue  (c.f. \cref{app_fig:workloads_unimodal}) (Gaussian peaks at equally spaced angles with varying breadths), a diverse catalogue (c.f. \cref{app_fig:workloads}) (uniform, unimodal, bimodal, trimodal, quadmodal, and quintmodal configurations), and a catalogue suited to the homogeneous system in which all tasks are initially centred at $\mu=\pi$, the exact location of both homogeneous agents, so that the homogeneous system is maximally well matched at $t=0$.
For each (regime $\times$ task) combination we record the full production time series over $10$~periods ($1\,000$~steps) and compute summary statistics:
mean production, coefficient of variation~(CV), and minimum production. Paired $t$-tests and Wilcoxon signed-rank tests across the $16$~tasks quantify
the statistical significance of the differences.

A more detailed description of the results is given in \cref{app:add_principles_robustness}, we give here a condensed version. Evaluating the unimodal catalogue reveals that the heterogeneous system maintains a higher and steadier production floor. During wave evolution, the heterogeneous system achieves a 134\% higher mean production and a 56\% lower coefficient of variation ($p<10^{-6}$ for both results). This operational stability persists under Brownian motion, showing a mean minimum production $7\times$ higher ($p<10^{-3}$ for both results). Results from the diverse catalogue confirm these advantages are structurally robust. Wave evolution favours the heterogeneous system across all 16 tasks for both mean and minimum production. Under Brownian motion and extreme events, the heterogeneous system outperforms in 14 of 16 tasks for mean and minimum production, with the minimum production advantage remaining highly significant across all evaluated regimes ($p<2 \times10^{-3}$). Finally the catalogue suited to the homogeneous system (\cref{fig:homosuited_summary_robustness}) offers a rigorous test by initially centering all tasks at the exact specialisation of the homogeneous system. While the homogeneous system achieves a $4\times$ higher mean production under Brownian motion ($p=0.04$), the heterogeneous system proves significantly more stable, securing a 41\% lower coefficient of variation ($p=0.014$). Extreme events exacerbate this divergence: the heterogeneous system yields a 44\% lower coefficient of variation ($0.20$ versus $0.36$, $p=0.004$) and a superior minimum production floor. Ultimately, the homogeneous system suffers severe output collapses when tasks shift away from its narrow competence, whereas the heterogeneous system ensures consistent and stable production.

\begin{figure}[H]
    \centering
    \includegraphics[width=\linewidth]{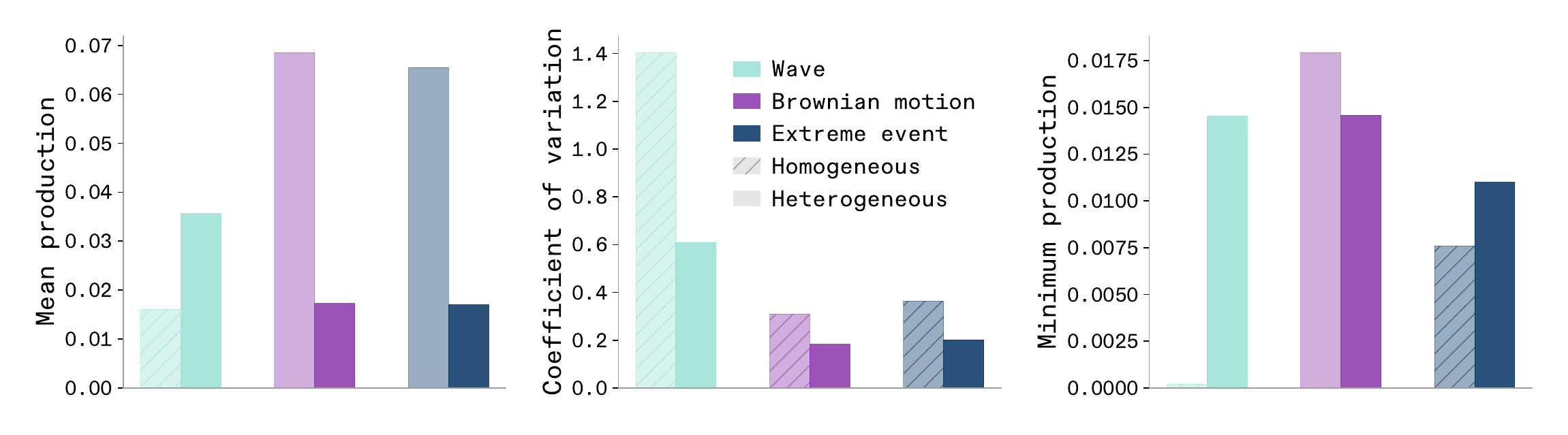}
    \caption{
        \textbf{Summary robustness for the catalogue suited to the homogeneous system}.
        Despite lower mean production, the heterogeneous system exhibits significantly lower CV and higher minimum production under extreme events.}
    \label{fig:homosuited_summary_robustness}
\end{figure}

Across $48$ tasks, three evolution regimes, and both favourable and unfavourable initial conditions, the heterogeneous production
system consistently outperforms the homogeneous system on robustness metrics.
When the environment is not specifically tailored to the homogeneous system,
the heterogeneous system dominates on mean production, variability, and minimum of production.
When the environment is initially tailored to the homogeneous system,
the heterogeneous system still provides significantly lower variability and
better protection against extreme events.
These results unify, within a single formal framework, the spatial insurance
effect in ecology, the diversification benefit in finance, and the robust
approximation advantage in learning.
They establish \emph{robustness under evolving demand} as a general principle
of heterogeneous production systems: heterogeneity acts as an insurance policy
whose premium, a modest reduction in peak performance when conditions are not
perfectly matched, buys substantially lower variance and a higher floor when
conditions inevitably shift.

\subsection{The recursive property of heterogeneity maximisation}
\label{sec:MHP_recursive}

The Maximum Heterogeneity Principle applies not only to a single layer of agents but propagates recursively through any hierarchical system in which the output of one layer constitutes the workload of the layer below. To see this, consider a system with $n$ layers, each defined over its own skill space $\mathbb{T}_k$, with agents parametrised by $(\mu^{(k)}, \sigma^{(k)})$ and connected through a network $Q^{(k)}$. The output of layer $k$ is a production function $W^{(k)}$ over $\mathbb{T}_k$, which serves as the workload demand that layer $k+1$ must satisfy:
\[
\mathbb{T}_0
\xleftarrow{\;\phi_1\;}
\mathbb{T}_1
\xleftarrow{\;\phi_2\;}
\cdots
\xleftarrow{\;\phi_{n-1}\;}
\mathbb{T}_{n-1}
\xleftarrow{\;\phi_n\;}
\mathbb{T}_n,
\]
where each $\phi_k$ is the mapping between adjacent layers. In the three-layer case that describes a cognitive system running on algorithms executing on hardware, this takes the concrete form:
\[
\mathbb{T}_{\mathrm{cog}}
\xleftarrow{\;\phi_1\;}
\mathbb{T}_{\mathrm{algo}}
\xleftarrow{\;\phi_2\;}
\mathbb{T}_{\mathrm{hardware}}.
\]

The key argument is as follows. At the top layer, the Maximum Heterogeneity Principle establishes that the optimal agent configuration is the maximally heterogeneous one. This optimal configuration produces a joint output $W^{(1)}$ that, by the nature of the heterogeneous coverage solution, is itself a broad and distributed function over $\mathbb{T}_1$. That distributed output then acts as the workload for layer 2. Since a broad workload is precisely the condition under which the Maximum Heterogeneity Principle applies (\cref{sec:principles_workloads}), the optimal configuration at layer 2 is again maximally heterogeneous. The same logic applies at every subsequent layer: a heterogeneous solution at layer $k$ produces a workload at layer $k+1$ that is broad enough to drive heterogeneity there too. Heterogeneity at the top therefore induces heterogeneity all the way down.

This argument holds as long as the inter-layer mappings $\phi_k$ are order-preserving on the region of skill space where the inter-agent variance lies, the condition established in \cref{app:between_layer_mapping}. Under that condition, the breadth of the workload produced at one layer is preserved as it is passed to the next, and the drive toward heterogeneity is not lost in translation between layers. Partial bottlenecks, where the mapping collapses parts of the skill space outside the region of meaningful variance, do not disrupt this: as long as the mapping preserves the structure of the workload where it matters, the recursive propagation of heterogeneity holds. Only a complete collapse of the workload to a near-uniform or Dirac-delta function at some layer would break the chain, as that would remove the distributional pressure that drives heterogeneity at the layer below.

\subsection{The resulting \textit{Maximum Heterogeneity Principle}}

What the results of this section jointly point to is an overall dynamic whereby the ideal solution for distributed production system is to converge on a maximally heterogeneous solution, only constrained by the maximum range of the task it is presented with, with the exact topological order specialisations within systems influenced by the internal communication structure of the system. As already shown in the discipline specific findings, the trend towards heterogeneity supports both overall system productivity as well as efficiency and robustness. Lastly, the heterogeneity maximising property stays true in multi-level system setup, due to recursive property of our observations. We hence suggest that the overall system dynamic of distributed production systems can be summarised through what we call \textit{The Maximum Heterogeneity Principle}.

Beyond the all findings described in \cref{sec:findings}, the now derived \TMHP also more specifically aligns with observation already made in biology and economics and hence reconciles them \cite{yoder2010ecological, carroll1985concentration, carroll2000microbrewery}. More interestingly, the principle does not only reconcile empirical observations but also connects naturally to several established mathematical frameworks, suggesting it may serve as an integrative theoretical principle more broadly. The most direct link is to the Maximum Entropy Principle \cite{banavar2010applications}: when the similarity structure between agents is uniform, maximising heterogeneity reduces to maximising the Shannon entropy of the agent distribution, since the two quantities are related by a strictly increasing transformation. The Maximum Heterogeneity Principle can therefore be understood as a distributed agent-focused view of the Maximum Entropy Principle, extending it to settings where agents carry structured pairwise similarities rather than being exchangeable. A related connection arises with overcomplete bases in signal processing and sparse coding: an overcomplete basis achieves robustness and flexibility by spanning a space with more elements than strictly necessary, mirroring the logic by which heterogeneous systems outperform minimal ones under changing workloads.

%% file: Text/5_forcing.tex
The prior section has made clear that distributed production systems across many fields will achieve optimal performance through building up heterogeneous systems-level architecture. As these conclusion were drawn from a theoretical model that was validated to empirical phenomena, it is important to test for its edge cases to reveal the model assumptions and show the specific necessary conditions are for \TMHP to be observed. In the following we explore edge cases of the workload definitions, communication networks and task assignment, and conclude that the conditions to break \TMHP are outside what is to be expected for any large distributed system.

\subsection{Absolutely homogeneous workloads}
\label{sec:forcing_workloads}

The heterogeneity observed across the system's level architecture is influenced by the specific workload / demand function of the environment, for cases when such a functions is present. The results from the prior section make this explicit as \cref{sec:MHP_workloads} shows how the kind of heterogeneity that develops is related to properties of the workload function and the this property ought to recursively permeate through the different layers of the model as shown in \cref{sec:MHP_recursive}. This also means that with relatively narrow and low stochasticity workloads, the ideal heterogeneity is going to be lower, but importantly, the system overall will still converge upon a heterogeneous solution within the bound of expected task heterogeneity. What about the extreme case when a single-valued Dirac Delta function describes the workload? This is an edge case that does break the optimisation of the model, so our model does not make specific prediction for such cases. However, this also not a realistic scenario, not because Dirac Delta distributed workloads do not exist at all, but because no distributed production system would ever face them. In \cref{app:dirac_workload} we provide a more detailed discussion of why that is, but in short, a Dirac Delta workload would imply that a single irreducible function (so it could not be decomposed) needed to be performed with a fixed (never varying) rate of executions per unit time. The closest example would likely be a clock oscillator, such as those found in CPUs or quartz crystals, which performs a single irreducible operation, toggling a bit, at a fixed frequency. However the kind of economic, biological, and computing systems that we are concerned with our model are supposed to cover workloads of higher complexity that naturally distribute across actors, as the many cases described in \cref{sec:findings} make clear.

\subsection{Only allowing communication between identically skilled agents}
\label{sec:forcing_communication}

The analyses in the prior principles section has revealed how good communication between agents is important to realise the needed levels of heterogeneity for good performance (\cref{sec:principle_topology}). On the flip-side that also means that any system with very poor communicability across the system will not be able to benefit from heterogeneous architectures and hence also not converge on it. A specific example of this was discussed in more close detail in context of non-trading countries (\cref{sec:findings_trade}), where countries without trade relationship converge on a homogeneous production of goods. Hence one can force a homogeneous solution by not allowing for collaboration between agents. The most direct route to this was of course to not allow for any network connections between agents. This could either be because no network connections are present, but that of course defeats the point of a model of ideal distributed production. A more nuanced scenario could be found for a constraint on communication which only allowed agents to collaborate if they were very similarly specialised. One could imagine a case of human collaboration (\cref{sec:findings_firms}) where differently skilled humans struggle to efficiently collaborate due to jargon, where increasing specialisation would also cause a decrease in the effectiveness of communication and hence act as an upper level constraint on the optimal level of heterogeneity. While this would decrease the optimal heterogeneity, it would only make the system to converge on a homogeneous solution if one did not allow for any communication as soon as agents were just minimally differently specialised. This does not seem like a reasonable effect of specialisation in any realistic distributed production system.

\subsection{Task to agent assignment and the orchestration problem}
\label{sec:random_task_assignment}
\label{sec:orchestration_probem}

An implicit assumption of the model is that task assignment is optimal: each agent $i$ is allocated the subworkload it is best placed to cover, such that the full workload is partitioned as $W_0 = \sum_i W_0^{(i)}$. Here we argue that the heterogeneity result is robust to violations of this assumption, even under fully random task assignment.

To see this, consider what random task assignment means for the system. Rather than each agent receiving its ideal subworkload, each agent receives a subworkload drawn from a distribution $\mathcal{P}$ over possible partitions of $W_0$. The system must therefore be configured to perform well in expectation over $\mathcal{P}$, rather than for one specific assignment.

This is the central trade-off that optimal portfolio theory addresses \cite{Markowitz1952, Choueifaty2008}. An investor who must allocate capital across assets before knowing their future returns faces an analogous problem: concentrating all capital in a single asset maximises return if that asset performs well, but produces large losses otherwise. The optimal response under uncertainty is diversification \citep{Choueifaty2008}, spreading exposure so that performance is robust across the range of possible outcomes. This logic does not require knowing future outcomes precisely; it only requires that outcomes are uncertain and no single one is guaranteed.

The same reasoning applies here. A homogeneous configuration concentrates all capacity near a single region of skill space. Under a favourable assignment this performs well, but the expected loss $\mathbb{E}_{\mathcal{P}}[\mathcal{L}(\mu, \sigma)]$ over random partitions will be high because the probability that every agent consistently receives a subworkload near its single shared specialisation approaches zero as the system grows. A heterogeneous configuration, by distributing agent positions $\mu_i$ broadly across $T$, reduces this expected loss at the cost of a modest reduction in best-case performance, exactly as a diversified portfolio trades peak return for robustness. We do not claim that the optimal configuration under random assignment is identical to that under perfect assignment. What the portfolio argument \ref{sec:findings_portfolio} establishes is the weaker but sufficient result: a homogeneous configuration is never optimal under an unbiased $\mathcal{P}$, because concentrating capacity is always the worst hedge against uncertainty in task assignment. The Maximum Heterogeneity Principle therefore holds not only under perfect orchestration but across the distribution of possible task assignments.

%% file: Text/6_compute_systems.tex
After having established \TMHP and its boundary cases, we next want to demonstrate how a cross-disciplinary principle can be applied to a new context where it can make specific predictions and act as a design blueprint for optimal systems-level design. Specifically we focus on large scale AI and computing systems, where the status quo looks very unlike what we would expect from our simulations, as generally speaking, large scale AI workloads are executed on high-performance computing clusters where the majority of operations is run on a single type of computing chip (GPUs). These systems hence show a much more homogeneous architecture, despite many of the cognitive skills that they are trying to achieve via large scale AI certainly being best captured by a wide workload distribution (arguments for a wide workload distribution are discussed in \cref{app:dirac_workload}). Here we model the specific reasons for this and show what we should expect to be possible if we build compute systems differently. In this context it is worth noting that following us first publicly sharing this manuscript on the \nth{26} February 2026, NVIDIA announced their Vera Rubin POD on the \nth{16} March 2026, containing multiple kinds of accelerators combined in a real heterogeneous AI computing system \cite{bhargava2026verarubin}. We see this as validation for our prediction that the benefit of heterogeneity will fundamentally transform how we best run large heterogeneous workloads on computing systems.

\subsection{Breaking \textit{The Hardware Lottery}}
Why do large scale computing systems seem to behave very differently from what is predicted by our principles, with relatively homogeneous scaling observed in current systems? The reason for this has been coined `The Hardware Lottery' \cite{hooker2021hardware} and is in fact predicted by our model. Specifically The Hardware Lottery discusses that any compute system is strongly constrained by the available hardware, as any function needs to be executed by chips and circuits. The chips we currently have which can execute operations at a large scale in parallel do so specifically focused on matrix multiplication, so that all systems are strongly constrained by having to go through a narrow set of functional representations. Our simulations show that such second order constrains severely limit innovations that could otherwise result in increased performance of computing systems. However, \TMHP can also go beyond these previous observations \cite{hooker2021hardware} by making specific predictions in how these constraining second order pressures can best be minimised. Specifically, the recursive property suggests (\cref{sec:MHP_recursive}) that to avoid any constrain, the level of heterogeneity should permeate through the individual layers of the system. Hence, in a currently second order constrained system, the ideal new hardware pieces are place in a way that mirror the heterogeneity of the upper level.

We can specifically show this in our model by running simulations where we use a high stiffness parameter on the second layer constraint that limits the densities  agents can converge on (see \cref{model_second_order} for details). Having a strong second order constraint is equivalent to having a strong hardware lottery force, where agents are strongly limited in which function they can execute well. We run a set of four optimisation experiments using our agent-based model. We consider a system of $N=6$ agents interacting via $J_6$ facing a mixture of Gaussian type workload, with a resource constraint of twice the number of agents. We optimise agent parameters using two control variables: the stiffness of the second-order constraint (the strength of the hardware lottery force) and the hardware prior $\mu_{\mathrm{hw}}$ that anchors each agent's preferred function. Hence the loss becomes: 
$$\mathcal{L}(\mu,\sigma) = \mathcal{L}_m(\mu,\sigma) + \mathcal{L}_s(\mu,\sigma;\mu_{\mathrm{hw}}),$$
where only the hardware parameter dependency is written.
\begin{enumerate}
    \item Homogeneous hardware.
We impose a strong second-order constraint (high stiffness) with a
homogeneous hardware prior
$\mu_{\mathrm{hw}}=\pi\,\mathbf{1}_N$, so that every agent is
pulled toward the same operating point. This mirrors the situation described
in \cite{hooker2021hardware}: all processors are optimised for the same
narrow class of operations (e.g.\ matrix multiplication), leaving the system
unable to cover the full breadth of the workload.
    \item Ideal (no hardware constraint).
We set the stiffness to zero, removing the second-order constraint entirely.
Agents are free to specialise without any hardware-imposed penalty. This
serves as the theoretical upper bound on system performance.
    \item Perfectly aligned heterogeneous hardware.
We take the optimal agent positions $\mu^{\star}$ obtained in Experiment~2) and use them as the
hardware prior, $\mu_{\mathrm{hw}}=\mu^{\star}$, while re-imposing a strong stiffness. This simulates a hypothetical hardware
stack whose heterogeneity is perfectly tailored to the workload, each chip is specialised for a different part of the computation space, and together
they mirror the optimal allocation of compute resources.
    \item Quasi-perfectly aligned heterogeneous hardware. We perturb the ideal hardware prior with a small random misalignment,
$\mu_{\mathrm{hw}}=\mu^{\star}+\varepsilon$ where each $\varepsilon_i$ is drawn uniformly from $\{-0.2,\,+0.2\}$, and again optimise under strong stiffness. This represents a more realistic scenario in which heterogeneous hardware is available but not perfectly
matched to the workload.
\end{enumerate}

\begin{figure}[H]
    \centering
    \includegraphics[width=0.9\linewidth]{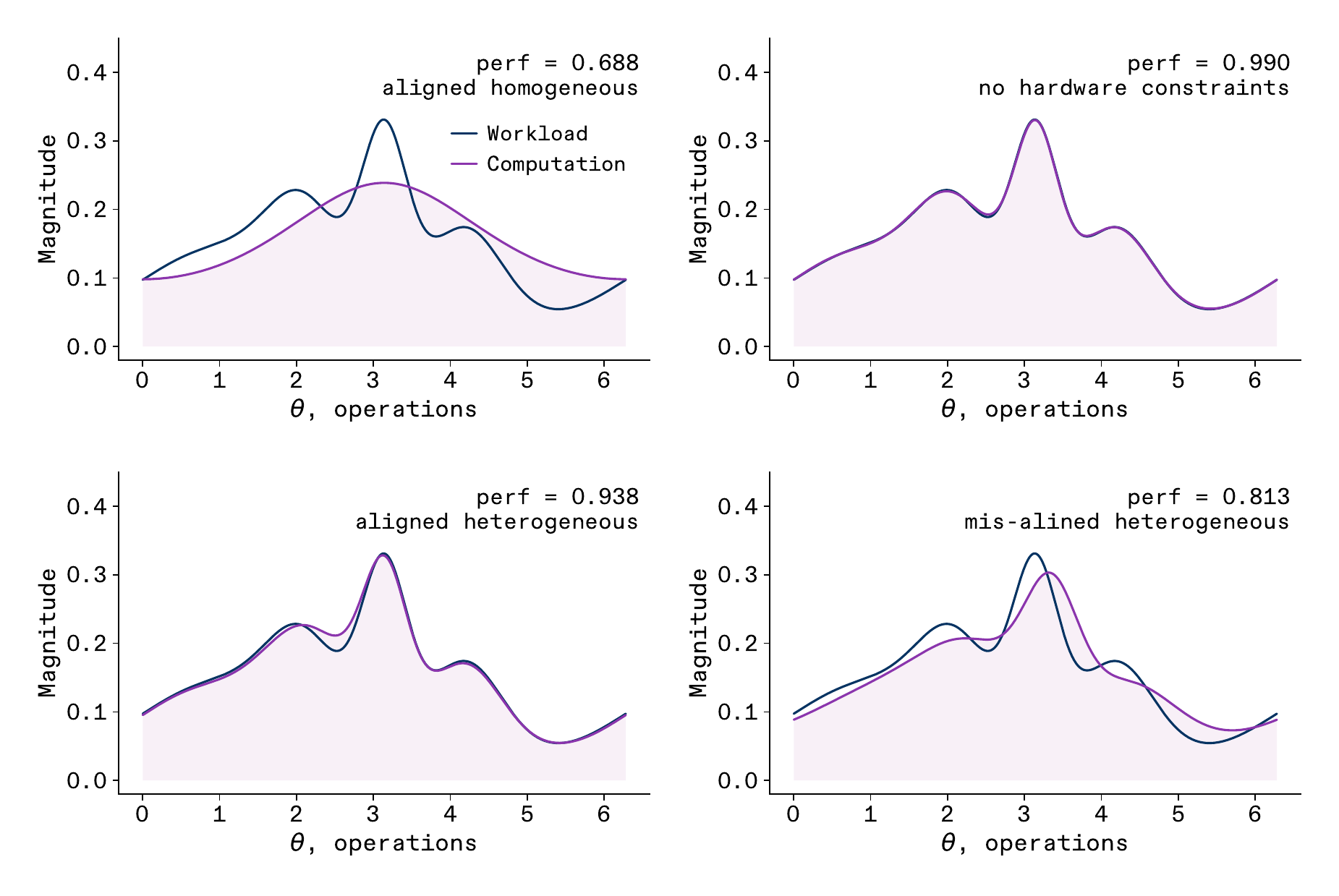}
    \caption{\textbf{Workload-production comparison across different hardware instantiations.}
        Workload is represented in blue and production functions are coloured in purple. Top-left, homogeneous hardware: The
        production function is narrow and poorly matched to the workload
        (perf.\ $\approx 0.69$). Top-right, no hardware constraint:
        agents freely specialise and nearly perfectly cover the workload
        (perf.\ $\approx 0.99$). Bottom-left, perfectly aligned
        heterogeneous hardware:strong constraints are present but aligned
        with the optimal specialisation, largely restoring performance
        (perf.\ $\approx 0.94$). Bottom-right, quasi-perfectly
        aligned heterogeneous hardware: small misalignment degrades
        performance but remains substantially better than the homogeneous
        case (perf.\ $\approx 0.81$).}
    \label{fig:lottery_grid}
\end{figure}

\Cref{fig:lottery_grid} displays the workload alongside the resulting production function for each experiment. Under homogeneous hardware, the production function collapses into a narrow peak that leaves most of the workload uncovered. Removing all hardware constraints allows the agents to distribute themselves optimally across the operational space, yielding near-perfect coverage. Crucially, the third demonstrates that strong hardware constraints need not be
detrimental: when the hardware heterogeneity is aligned with the workload optimal structure, the system recovers most of the ideal performance despite the
high stiffness. Experiment 4) shows that even approximate alignment yields a substantial improvement over the homogeneous baseline, although the residual
misalignment incurs a measurable performance loss. \Cref{fig:lottery_bar} quantifies these differences. The gap between task
performance and full performance reflects the cost of the hardware mismatch
itself. For aligned heterogeneous hardware this gap is small, confirming
that well-chosen specialised components introduce little additional penalty.
By contrast, homogeneous hardware both limits the achievable task coverage
and imposes a large mismatch cost, compounding the performance loss. These
results provide a concrete, quantitative illustration of the prediction
derived from the recursive property: the most effective strategy to mitigate
the hardware lottery is to introduce heterogeneous components whose
specialisation mirrors the structure of the workload at the level above.

\begin{figure}[H]
    \centering
    \includegraphics[width=0.8\linewidth]{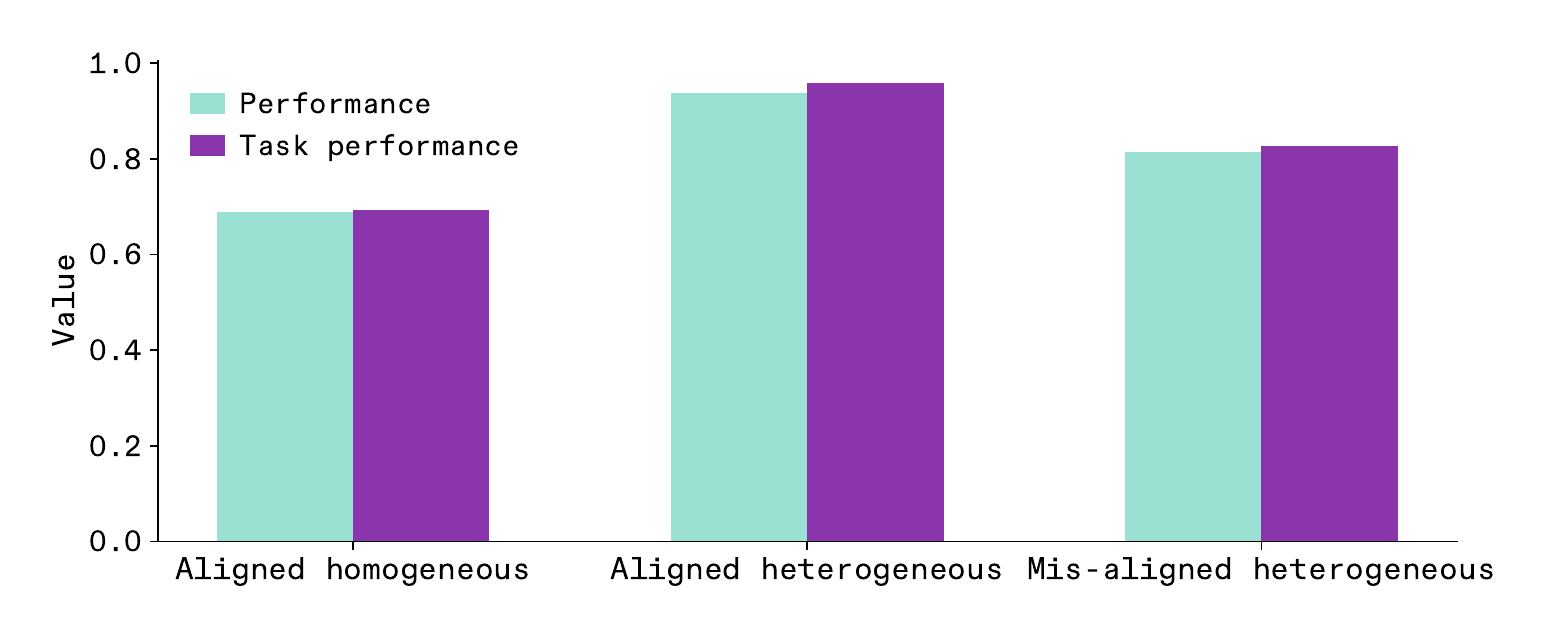}
    \caption{\textbf{Summary of system performance across hardware configurations.}
        Performance denotes the full objective $1/(1+\mathcal{L})$, task performance isolates the workload-coverage component $1/(1+\mathcal{L}_m)$,
        excluding the hardware mismatch penalty. Aligned heterogeneous hardware nearly matches the task performance of the
        unconstrained case, while homogeneous hardware suffers a large deficit. Misalignment degrades both metrics but remains preferable to homogeneity.}
    \label{fig:lottery_bar}
\end{figure}

Consequently, \TMHP predicts that new hardware pieces will unlock new level of performance in our compute systems, as long as they play complimentary roles with regards to the overall workload function. It is now an increasingly interesting time to track the outcome of this prediction, as new kinds of accelerators are being developed which can complement the existing hardware stack \cite{silvano2025survey, sun2025algorithm, orchard2021efficient, jouppi2023tpu, modha2023neural, kagan2025cl1}.

\subsection{Redefining \textit{Scaling Laws}}
As already discussed in prior sections, scaling laws have been observed across many disciplines \cite{west1997general} and have also played a particular role in the scaling of large language models because of the observed neural language model scaling laws \cite{kaplan2020scaling} which observe increasingly improved performance of models as a function of compute, data, and number of model parameters (discussed in more detail in \cref{sec:findings_scaling_laws}). Critically, it has given us the license to scale up our compute systems as targeting increased parameter numbers and amounts of compute in FLOPs were predicted to have a direct effect on the performance of the resulting model. Our simulations predict that as we move towards models which ought to be truly domain general \cite{achterberg2023building}, we need not only to consider the FLOPs but also the kind of operations that are realisable through these FLOPs as we otherwise risk imposing a strong second-order constraint on the system-level design which will continue to degrade the efficiency of the overall system-level design (\cref{sec:principles_efficiency}). This is not to say that one could never achieve domain-general cognition without heterogeneity, but that doing so at acceptable and economical compute budgets might remain an elusive goal under homogeneous scaling. Adding the correct kind of heterogeneity promises to continue not only scaling the computational capacity of our systems but also doing so while keeping the marginal benefit per addition in compute high.

To illustrate this point concretely, we run a scaling experiment in which team size $N$ grows from 2 to 25 agents on a fixed bimodal Gaussian workload ($\mu = \{\pi/2,\, 3\pi/2\}$, $\sigma = 0.5$) under a per-agent resource constraint $(2 + \log N)/N$ modelling better access to resources with scale, and a high hardware stiffness ($\lambda = 10$). We compare two hardware regimes: a homogeneous configuration in which every agent starts with its hardware prior centred at $\pi$ (equidistant from both task modes), and a heterogeneous configuration in which agents alternate between priors centred at $\pi/2$ and $3\pi/2$, thus pre-aligned with the two task modes. For each $N$ we run 20 independent Monte Carlo optimisations and measure productivity ($\Pi = 1/(1+\mathcal{L})$ where $\mathcal{L}$ is the total loss including the hardware mismatch penalty). The results (\cref{fig:h_vs_e_scaling}) show that the heterogeneous system reaches near-optimal productivity already at $N = 4$ and remains flat thereafter, while the homogeneous system climbs steadily but does not match the heterogeneous curve until $N \approx 11$. Beyond that crossover both regimes converge to the same plateau, suggesting that a sufficiently large homogeneous team can eventually compensate for its workload mismatch. This is precisely the inefficiency highlighted above: scaling FLOPs without considering whether those FLOPs are of the right kind leads to a second-order constraint that degrades the marginal return of each additional unit of compute. Adding the appropriate heterogeneity removes that constraint early and keeps the marginal benefit per agent high across the entire scaling curve.
\begin{figure}[H]
    \centering
    \includegraphics[width=0.5\linewidth]{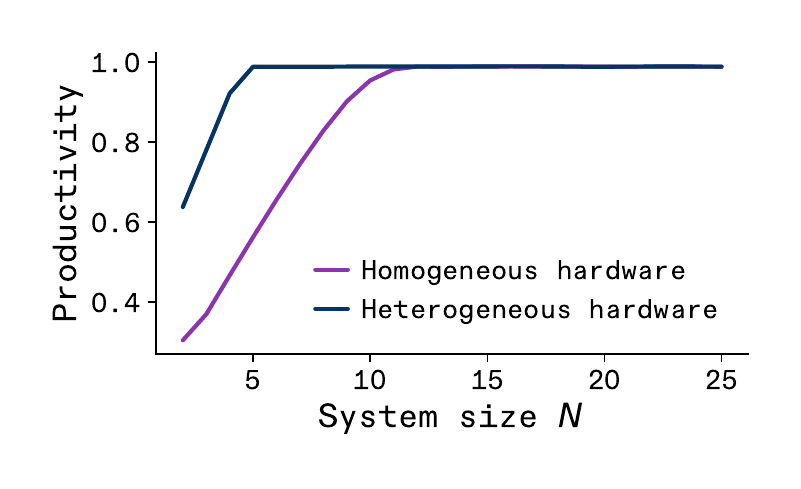}
    \caption{\textbf{Heterogeneous hardware reaches peak productivity at smaller system sizes.} Smoothed mean productivity ($\pm 1$\,SE, 20 runs) versus system size $N$. The heterogeneous system (blue) saturates near its maximum already at $N \approx 4$, whereas the homogeneous system (purple) requires roughly $N \approx 11$ to reach the same level.}
    \label{fig:h_vs_e_scaling}
\end{figure}

Of course the units representing system size in \cref{fig:h_vs_e_scaling} are arbitrary and do not map onto actual real-life units. As a result, our results do not suggest that from system size $\approx 11$ a homogeneous system is always as performant as a heterogeneous one. Instead they point to the more general heterogeneous scaling principle which is stating that in the limit with infinite resources both systems are as performant, but that for most resource-constrained systems, heterogeneous systems will be more performant. \cref{fig:h_vs_e_scaling_schematic} provides a schematic illustration of heterogeneous versus homogeneous scaling performance in the general case.

\begin{figure}[H]
    \centering
    \includegraphics[width=0.6\linewidth]{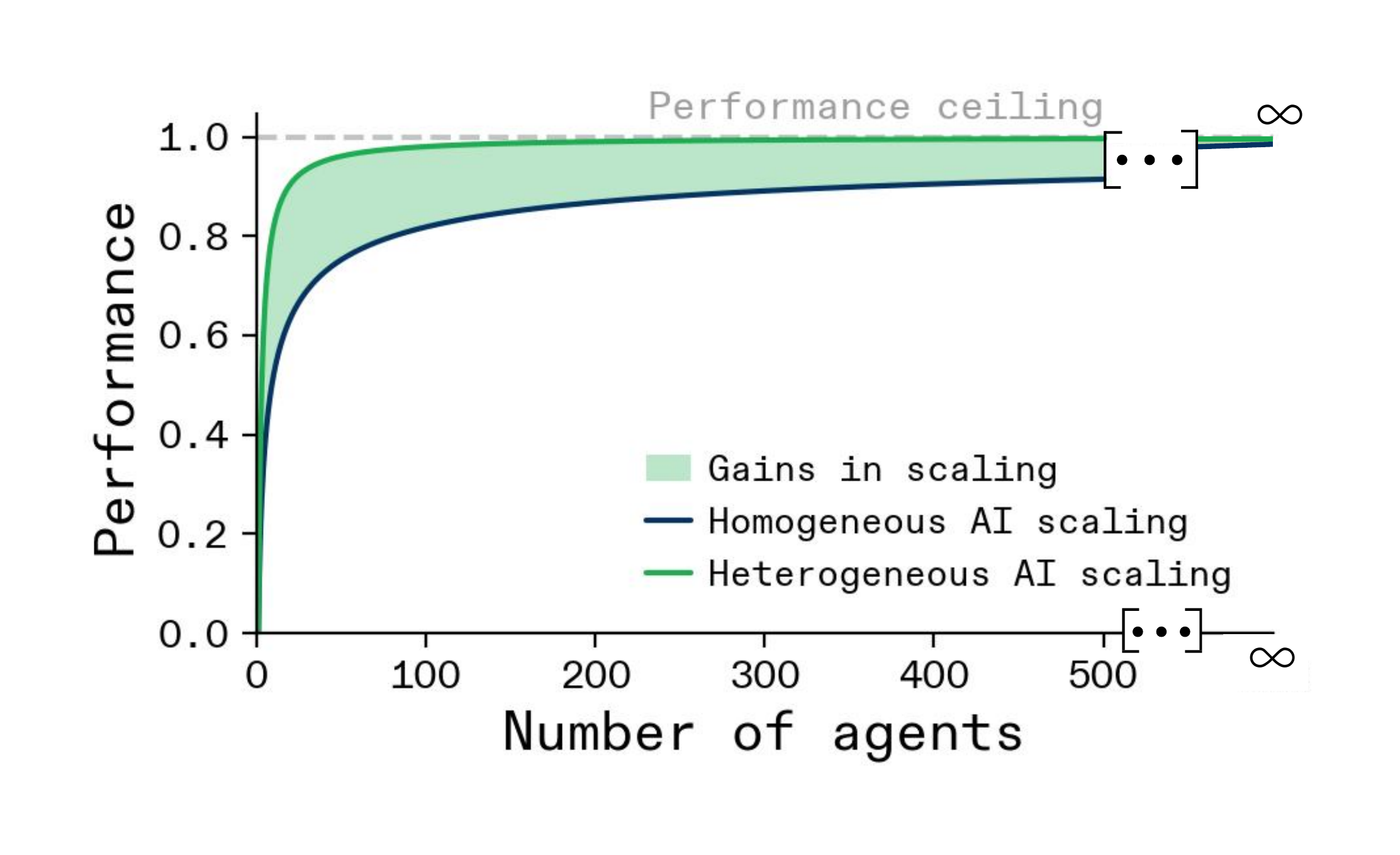}
    \caption{\textbf{Heterogeneous hardware is more performant than homogeneous hardware with performance meeting in the limit of infinite resources.} The purple curve represents homogeneous scaling and the green curve heterogeneous scaling. Figure shows a schematic and does not show an actual analysis based on model simulations.}
    \label{fig:h_vs_e_scaling_schematic}
\end{figure}

\subsection{Expanding \textit{The Bitter Lesson}}
As achieving increasing performance of models via a pure scaling strategy resonated with Sutton's `Bitter Lesson' \cite{sutton2019bitter}, stating that models which rely on scalable strategies to acquire their skills, so through learning and search, would always end up outperforming methods that relied on `cleverly hand-designed systems invoking smart priors'. The heterogeneous scaling laws that follow from \TMHP are not in violation of that but instead highlight how important it is that we expanding the scalable system-optimisation strategies to cover not only well-established model parameter training via gradient descent, but instead move to a world where learnable algorithm-hardware co-design is possible at scale \cite{plagge2022athena, cardwell2024device, mirhoseini2021graph, sun2025algorithm, lai2023chipformer, jiang2024circuitnet, tschand2026genai}.

%% file: Text/7_discussion.tex
The main contribution of this work is the discovery of the \textit{Maximum Heterogeneity Principle}. We do so by developing a new theoretical model which generalises across scientific disciplines and is tractable enough to allow us to derive what laws it follows to converge on optimal system's level design. We summarise this set of laws in \textit{The Maximum Heterogeneity Principle} of distributed production systems, which says: when a system develops towards optimal performance it needs to optimise towards heterogeneity. Optimal performance here means both maximum efficacy and efficiency, as well as robustness. Fixed environmental demands can provide an upper bound on the need for heterogeneity, and the locality and spread of heterogeneity across a system is defined by the system's internal communication structure, more specifically the radius of effective communication between agents defines the spread of heterogeneity throughout the system. In this, \textit{The Maximum Heterogeneity Principle} has a recursive property. This means that when one production systems underlies another production system, as can be seen in computational problems being solved by algorithms on the top level which in turn are being executed by computer chips on the bottom level, that the heterogeneity of the highest level system needs to permeate across all layers of the system for optimal performance. Hence heterogeneity optimises performance of a system across all layers of a system. The discovery of \textit{The Maximum Heterogeneity Principle} has important implications both for understanding large distributed systems, as well as constructing new ones.

Capturing the complex dynamics and relationships that underlie phenomena observed across nature is the fundamental goal of the sciences and with \textit{The Maximum Heterogeneity Principle} we reveal a new attractor point of any distributed production system, making it to converge on a heterogeneous system level setup as a product of optimising its own productivity. Admittedly it seems that only very rarely systems across scale seem to agree on their inner workings, such as previously describe by the $x^{\frac{3}{4}}$ scaling law \citep{west1997general}, and \textit{The Maximum Heterogeneity Principle} seems to be one of these rare but clear cases. We are far from the first to spot the importance of heterogeneity across systems \cite{dahmen2025heterogeneity, joshi2025neuronal, cook2025brainlike, achterberg2023building, Freund2025}, and through our model we were able to pin down why it is so prevalent in systems all around us.

A large part of the scientific value of `general principles' comes from their applicability to new situations. Game theory as example has not only answered scientific questions but has also been a helpful tool for navigating real world conflicts \cite{matsumoto2016game}. The \textit{Maximum Heterogeneity Principle} promises the same applicability to new systems. In our work we discussed its applicability to the design of radically different computing systems already, predicting that new levels of performance could be achieved if we scaled them in a heterogeneous than a homogeneous way. Specifically, our work proposes new heterogeneous scaling laws which would break the hardware lottery that we currently face in computing systems \cite{hooker2021hardware}. This quantitatively supports perspectives on new computing systems designs which have been described by other \cite{cardwell2020truly, hagleitner2021heterogeneous, siegel1995heterogeneous}. This would open up a whole new world of large scale computing that would be unachievable today and \textit{The Maximum Heterogeneity Principle} can act as a blueprint towards constructing these new computing systems.

The question remains how broadly the \textit{Maximum Heterogeneity Principle} generalises beyond computing. We believe it applies to any distributed production system, which we define as any system in which agents jointly produce an output through collaborative and competitive interactions, such that appropriate coordination can increase total output. This encompasses the full range of systems studied here: ecosystems producing biomass, neural circuits producing cognition, and economies producing GDP. The boundary of the principle's applicability is therefore the boundary of this class. In systems that are not distributed production systems, increasing heterogeneity need not be beneficial. A useful contrasting case can be found in physical processes aimed at reducing entropy, such as cooling or crystallisation, where energy is actively expended to decrease heterogeneity. That such reductions require an energy cost is itself consistent with the second law of thermodynamics, and underscores that the drift towards heterogeneity we describe is a property of systems optimising joint output, not of systems in general.

\subsection{Limitations}
\label{sec:limitations}
Our model treats workloads as flat distributions over tasks, but in practice workloads have internal dependency structure: operations must often be executed in a specific order, forming graphs rather than independent draws. This limitation is relevant across the systems we study. In computing, workload graphs are standard; in economics, production chains impose ordering constraints on inputs; in neuroscience, perceptual processing decomposes complex features in a fixed hierarchical sequence. Extending the workload representation to graphs is a natural next step. We include a first investigation of the more complex representation of the workload in \cref{app:workload_as_graph}, showing how the principles uncovered in this work still hold true for this case, but we a future version of the model could handle this workload representation more directly. A related point is that in real computing systems some operations can be hidden through parallelisation, effectively reducing their weight in the distribution; this is consequential primarily in peak-heavy workload distributions, where it may shift the optimal agent skill profile.

On the agent side, our model does not currently account for resource depletion at the lower level, nor does it handle cases where agents differ substantially in overall skill level rather than in the shape of their skill distribution. Both are natural extensions. More broadly, we also observe that the principle does not straightforwardly account for cases of strong within-feature homogeneity, such as the uniform long neck across giraffes, where a single trait converges to a fixed optimum across all agents in a system.

\subsection{Conclusion}

\textit{The Maximum Heterogeneity Principle} describes how in their goal of optimising their performance, all distributed production systems ought to drift towards increasingly heterogeneous system's level designs. We uncover this principle by integrating scientific studies across 6 disciplines and at least 80 years of scientific findings and it can now guide us to construct better large scale distributed systems.

%% file: Text/heterogeneity_measure.tex
In this appendix we give insights, details, and motivation for the framework we use to measure heterogeneity in a way that enables discrimination between production systems of different nature: economic production systems (firm of workers, countries trading together), neural networks or biological production systems (communities of plants producing biomass).

\subsection{Plenty heterogeneity measures} We need a tangible way to assess the heterogeneity. The literature provides a vast amount of different measures, each emphasising different forms of difference within a system they tend to capture. We provide a few illustrative examples. The simplest measure is the richness, which is the count of distinct types in a sampling unit,
\[
S = \sum_{i=1}^{K} \mathds{1}_{\{n_i>0\}},
\]
where $n_i$ is the abundance of type $i$ in the sample. From information theory, the Shannon entropy provides an interesting way to measure heterogeneity as the spread of a probability distribution (the average amount of surprise when sampling from this distribution),
\[
H = -\sum_{i=1}^{K} p_i \log p_i,
\]
where $p_i = n_i / \sum_j n_j$ is the proportion of type $i$ in the sample. There are also numerous statistics, such as the sample variance, the coefficient of variation, the Simpson index \cite{Simpson1949},
\[
D = 1- \sum_{i=1}^{K} p_i^2,
\]
the Hill numbers \cite{Hill1973} which correspond to the effective number of elements in a system (Hill numbers of order $q$),
\[
{D}^{q}(p) = \left(\sum_{i=1}^{K} p_i^{q}\right)^{1/(1-q)}.
\]
Statistical tests have also been used to assess heterogeneity. E.g. Cochran's $\chi^2$ test or the $Q$-test assess clinical heterogeneity which describes the differences between studies' results addressing a common question \cite{higgins2003measuring, Cochran1954, Whitehead1991}.

\subsection{Toward a better measure of heterogeneity}
Many of these measures have issues \cite{Jost2006, Jost2007, Jost2008, HardyJost2008} as they do not verify core principles or axioms that a heterogeneity measure function should verify $\mathcal{H}$. We reformulate them from \cite{Jost2009} more generally. Let us consider a system containing a set of different elements endowed with an abundance distribution $p$ that measures the proportion of each element of the set in the system. Note that an element in the set can have zero abundance and is consequently not visible within the system. The axioms are formulated as follows:
\begin{enumerate}
    \item \textit{Symmetry.} The heterogeneity function is symmetric in its arguments, i.e. for any permutation $\sigma:\{ 1, \dots, n \} \to \{1, \dots, n \},$
    $$\mathcal{H}(p_1, \dots, p_n) = \mathcal{H}(p_{\sigma(1)}, \dots, p_{\sigma(n)}).$$
    \item \textit{Expandability (Zero Output Independence).} Adding another element in the set that has zero abundance does not change the heterogeneity of the system.
    \item \textit{Output Transfers Principle.} Transferring a unit of abundance from a common element to a rarer element should not decrease heterogeneity.
    \item \textit{Homogeneity.} The heterogeneity depends only on species' relative frequencies and not on their absolute abundances.
    \item \textit{Replication Principle.} Suppose $K$ systems have identical abundance $p$, but no elements are shared between any of the systems. All $K$ systems necessarily have the same heterogeneity $\mathcal{H}(p)$. Suppose we pool all $K$ systems. If the heterogeneity measure obeys the replication principle, then the heterogeneity of the $K$ pooled equally diverse, equally large, completely distinct communities must be $K \mathcal{H}(p)$.
    \item \textit{Normalization.} If the heterogeneity measure is applied to $N$ equally common elements, its value is $N$.
\end{enumerate}

Verifying them thus provides a criterion to discriminate whether a measure is acceptable or not. From the examples listed above, the Hill numbers \cite{Hill1973} verify this set of properties \cite{Nunes2020, Jost2009}. There exist other measures that verify all of these axioms \cite{Hoffmann2008}, but the Hill numbers, from our observation, have the greatest support \cite{Nunes2020} to be the reference to quantify heterogeneity. 

However, we argue that this measure is slightly incomplete because it only takes into account the abundance distribution and not the pairwise dissimilarity between the elements. Indeed, take two boxes, one with one red shoe, one green shoe, and one blue shoe, and the other with one red shoe, one green shoe, and one blue ice cream. Intuitively, we would like to say that the second box is more heterogeneous than the first (unless the colour is the only discriminating factor we care about). None of the Hill numbers would enable one to discriminate the heterogeneity of these two systems as they do not take into account the differences between elements \cite{Jost2009}. That is to say, as soon as we have a grasp on the dissimilarity between the elements within the system, it should be taken into account in the measure. At the axiom level we argue that the \textit{output transfers principle} is too imprecise; it works only if the classes of elements within the system are equally different. As such, we support the claim of Tom Leinster and Christina Cobbold \cite{LeinsterCobbold2012} and use the formalism developed by Tom Leinster and co-authors \cite{Leinster2021, LeinsterRoff2021, LeinsterMeckes2016} to define heterogeneity.

\subsection{Heterogeneity measure}
This appendix subpart is largely inspired by Tom Leinster's book \cite{Leinster2021}. We give here the intuition and definition of heterogeneity and adapt it to our context.

\subsubsection{Heterogeneous system} First of all, we need to define the system on which we want to measure heterogeneity. The different systems stemming from the literature are finite sets with elements and an abundance measure that counts the number of elements within the system. Dividing the abundance measure by the total number of elements within the set, we get a probability distribution. So if we are not taking into account between-element differences, the system is formally a finite probability space $(S, p)$ where $p$ is the probability distribution. To model the between-element differences, or conversely, if we have some sense of similarities between elements we can define a square matrix $Z$ that encodes in each $(i,j)$-entry the similarity between the $i$-th and $j$-th elements. The formalism of a system becomes a finite probability space $(S, p, Z)$ endowed with a square similarity matrix.
\begin{example}
    Let us consider a set $S = \{x_1,x_2,x_3\}$, and a pairwise distance $d$ between elements. $(S,d)$ can be viewed as a metric space or as a graph with vertices $x_1,x_2,x_3$ and undirected edges having values $d(x_1,x_2)$, $d(x_1,x_3)$ and $d(x_2,x_3)$. From the distance $d$ we can construct a $3 \times 3$ similarity matrix $Z \in [0,1]^{ 3 \times 3}$ as follows
    $$\forall\; i,j \in \{1,2,3\} \quad Z_{ij} = e^{-d(x_i,x_j)} \in [0,1].$$
    We denote $p = (p_1, p_2, p_3)$ the probability distribution over $S$.
    The heterogeneous system is then defined as the triplet $(S,p,Z)$ or, as the triplet $(S,p,d)$ if one can construct a distance. 
\end{example}

\subsubsection{System heterogeneity} Building on Hill's numbers \cite{Hill1973} we define the heterogeneity measure of a finite (heterogeneous) system. The following definition corresponds to the diversity measure introduced in \cite{Leinster2021}.
\begin{definition}[System's heterogeneity measure]
    Let $S = \{x_1, \dots, x_n\}$ be a finite system with similarity matrix $Z \in \R_+^{n \times n}$ and a probability distribution $p = (p_1, \dots, p_n)^T \in (\R_+^*)^{n}.$ For any $q \in [0, \infty]$ we define the $q$-th heterogeneity measure of $S$ as,
    $$\mathcal{H}^q(S) := \left(\sum_{i=1}^n p_i(Zp)_i^{q-1}\right)^{1/(1-q)}.$$
\end{definition}

\begin{proposition}
    Let us fix a system $S$ with notations as in the previous definition. For any $q \in [0, \infty]$, the $q$-th heterogeneity number of $S$ is well defined and we have,
    $$\mathcal{H}^1(S) = \prod_{i =1}^n (Zp)_i^{-p_i}, \quad \mathcal{H}^{\infty}(S)= \frac{1}{\max_{1 \leq i \leq n}(Zp)_i}$$
\end{proposition}
  
\begin{proof} For $q \neq 1$ the associated heterogeneity number is well defined and the function of $q$ is smooth on $\R_+ \setminus \{1\}.$
\begin{enumerate}
\item There exist equal left and right finite limits at point 1. To show it, we perform a limited expansion around 1.
    \begin{align*}
        \mathcal{H}^q(S) &=  \left(\sum_{i=1}^n p_i(Zp)_i^{q-1}\right)^{1/(1-q)} \\
        &=  \left(\sum_{i=1}^n p_i e^{(q-1)\log(Zp)_i}\right)^{1/(1-q)} \\
        &= \left(\sum_{i=1}^n p_i( 1 + (q-1)\log(Zp)_i + o(q-1))  \right)^{1/(1-q)} \\
        &= \left(\sum_{i =1}^n p_i + (q-1)\sum_{i =1}^n p_i\log(Zp)_i + o(q-1)  \right)^{1/(1-q)} \\
        &= \exp\left( \frac{\log\left( 1 + (q-1)\sum_{i =1}^n p_i\log(Zp)_i + o(q-1) \right)}{1-q}\right) \\
        &= \exp\left( \frac{(q-1)\sum_{i =1}^n p_i\log(Zp)_i + o(q-1)}{1-q}\right) \\
        &= \exp\left( -\sum_{i =1}^n p_i \log(Zp)_i + o(1)\right) \\
        &= \prod_{i =1}^n (Zp)_i^{-p_i} + o(1)
    \end{align*}
\item When $q$ tends to $\infty$ we have also a finite limit. To prove this claim we introduce some notations. $m = \max_{1 \leq i \leq n}(Zp)_i$, $I = \{1 \leq i \leq n \; | \; (Zp)_i=m \}$, $I^c = \{1, \dots, n\}\setminus I$, $\widetilde{p_i} = p_i / \sum_{j \in I} p_j$ and $x = 1 /(q-1) \in \R_+^*$ that tends to $0^+$ when $q$ tends to infinity. We use the variable $x$ in the heterogeneity expression,
$$\mathcal{H}^q(S) =  \left(\sum_{i=1}^n p_i(Zp)_i^{q-1}\right)^{1/(1-q)} = \left(\sum_{i \in I} p_i(Zp)_i^{1/x} + \sum_{i \in I^c} p_i(Zp)_i^{1/x}\right)^{-x}.$$
We factorise by the first term within the parentheses $\sum_{i \in I} p_i(Zp)_i^{1/x} = m^{1/x}\sum_{i \in I} p_i$,
$$\mathcal{H}^q(S) =m^{-1}\left(\sum_{i \in I} p_i\right)^{-x} \left(1 + \sum_{i \in I^c} \widetilde{p_i} \left( \frac{(Zp)_i}{m}\right)^{1/x}\right)^{-x}.$$
If $I^c =\emptyset$, since $p$ is a probability distribution $\sum_{i \in I} p_i = 1$ and for any $q \geq 0$, $\mathcal{H}^q(S) = m^{-1}.$ Otherwise, for all $i \in I^c$, $0 < (Zp)_i /m < 1$ and since $1/x \to + \infty$, $\left((Zp)_i /m\right)^{1/x} \to 0.$ Consequently, $1 + \sum_{i \in I^c} \widetilde{p_i} \left( \frac{(Zp)_i}{m}\right)^{1/x}$ tends to 1 as $x$ tends to $0^+$. Besides, for any strictly positive real number $a$, $x \in \R_+^* \mapsto a^{-x}$ is continuous and tends to one as $x$ tends to $0$. So first of all, $\left(\sum_{i \in I} p_i\right)^{-x} \to 1$ and by composition of continuous functions, $\left(1 + \sum_{i \in I^c} \widetilde{p_i} \left( \frac{(Zp)_i}{m}\right)^{1/x}\right)^{-x} \to 1$ and we obtain $\lim_{q \to \infty} \mathcal{H}^q(S) = m^{-1}.$ Therefore, we can define, for any system $S$, a continuous function on $[0, \infty]$ endowed with the usual (extended real line) topology.
\end{enumerate}
\end{proof}
The system's heterogeneity measure extends the family of Hill numbers on the system. Indeed, replace $Z$ by the identity matrix $I$ and the $q$-th heterogeneity measure equals the $q$-th Hill number. Consequently, this measure verifies the set of axioms listed above. The parameter $q$ acts as a sensitivity modulator toward rare elements or dominant elements. When $q$ is small (approaching $0$), the measure gives relatively greater weight to rare elements. For example, when $Z = I$ the measure approaches the richness of the system which is simply the number of different elements (the cardinality of the set). As $q$ increases, the measure increasingly emphasises common species and reduces the influence of rare ones, making the index more sensitive to dominance. In the limit $q = \infty$ the heterogeneity measures the inverse frequency of the most dominant element in the system. Another interesting result is that, at $q = 1$, the index corresponds to the exponential of Shannon entropy, weighting species proportionally to their frequencies.
\begin{remark}
In this study we have been using the system heterogeneity measure of order $2$ extensively, which has the following simplified expression:
$$\mathcal{H}^2(Z,p) = \frac{1}{p^TZp}.$$
It resembles the squared inverse of a Mahalanobis norm of $p$, if the similarity matrix were positive semi-definite, which it need not be.
In the particular case where $p$ is the uniform distribution, it becomes:
$$\frac{n}{1 + \frac{2}{n} \sum_{1 \leq i < j \leq n}Z_{ij}}$$
\end{remark}
We give a short proof of the remark.
\begin{proof}
    For the simplified expression, in the case of $q = 2$, it suffices to note that $\sum_{1 \leq i \leq n} p_i(Zp)_i = p^T Z p$. When $p$ is the uniform distribution we then have,
    $$\mathcal{H}^2(Z,p) = \frac{n^2}{1^TZ1} = \frac{n^2}{\sum_{ij} Z_{ij}},$$
    and since $Z_{ii} = 1$ for all diagonal coefficients, and since $Z$ is symmetric,
    $$\mathcal{H}^2(Z,p) = \frac{n^2}{n + 2\sum_{1 \leq i < j \leq n} Z_{ij}} = \frac{n}{1 + \frac{2}{n}\sum_{1 \leq i < j \leq n} Z_{ij}}.$$
\end{proof}
We have defined system heterogeneity; in the next subsection we extend it towards a more general framework in which we abstract away from the system to focus on the mathematical objects so as to define the heterogeneity measure.

\subsubsection{Heterogeneity measure}
We have previously seen heterogeneity as a measure on a system. Mathematically it only depends on two objects: the probability distribution $p$ and the similarity matrix $Z$. It suggests we can forget about the elements within the systems (their true nature), and only consider what is of interest when accounting for heterogeneity: correspondingly, the relative proportions $p$ and the dissimilarity (or conversely, the similarity) between the elements. Hence, we can define the heterogeneity as a function over the class of finite probability distributions times the family of similarity matrices.

\begin{definition}[Heterogeneity measure on finite systems]
Let $\mathcal{Z}_n$ denote the class of similarity matrices of size $n$, and $\mathcal{P}_n$ the family of probability distributions on a set of size $n$. We define the $q$-th heterogeneity measure function as follows:
\[
\left[
\begin{array}{ccl}
\mathcal{H}^{q} : \bigcup_{n \in \mathbb{N}}& \mathcal{Z}_n \times \mathcal{P}_n &\to \R_+\\
 &\left(Z,p \right) &\mapsto \left(\sum_{i=1}^n p_i(Zp)_i^{q-1}\right)^{1/(1-q)}
\end{array}
\right].
\]
\end{definition}
By forgetting about the nature of the elements, we obtain a general measure that can enable us to compare the heterogeneity between systems. It is not only defined for a particular system but for all systems defined with a similarity matrix and a distribution. This framework is suitable for understanding the effect of heterogeneity on production, and conversely how a demand can shape the heterogeneity of a system. The key takeaway here is that the prism from which we look at heterogeneity is the one of similarity and distribution only.

\subsubsection{Heterogeneity measure on infinite space}
A natural extension is to go from finite heterogeneous systems to infinite heterogeneous systems. Close to \cite{Leinster2021} formalism we develop here the concept of heterogeneous spaces and define the heterogeneity measure on such spaces. This measure is used in \cref{sec:principles_workloads} to compute the heterogeneity of the task, a feature of the task highly related to the heterogeneity of the system optimised to perform this task. To define a heterogeneous system it suffices to extend each of the elements of a finite heterogeneous system. The finite set $S$ becomes infinite, the similarity matrix becomes a symmetric kernel,
$Z: S^2 \to [0, 1],$
and the distribution of abundance proportions becomes a probability measure (provided a $\sigma$-algebra is realistically constructible). With these extensions we can define the concept of heterogeneous spaces:
\begin{definition}[Heterogeneous space]
Let $(S, \mathscr{F}, \mu)$ be a probability space, and $Z:S^2 \to [0,1]$ be a similarity kernel. A heterogeneous space is a quadruplet $(S,\mathscr{F}, \mu,Z)$ composed of a probability space $(S, \mathscr{F}, \mu)$ and a similarity kernel $Z$. When no confusion is possible we write $(S, \mu, Z)$, $(S,Z)$, or $S$ to refer to the heterogeneous space $(S,\mathscr{F}, \mu, Z).$
\end{definition}
Analogously to the finite system, if there is a natural distance on the space the similarity kernel can be defined using this distance. Consequently, one can formulate an analogous definition of heterogeneous infinite system for a metric space.
\begin{definition}[Heterogeneous space endowed with a metric]
    Let $(S,\mathscr{F},\mu,d)$ be a metric space. The distance $d$ endows a similarity kernel,
    \[
\left[
\begin{array}{ccl}
Z_d : & S^2 &\to \R_+\\
 & \left(x,y \right) & \mapsto e^{-d(x,y)}
\end{array}
\right]
\]
and the heterogeneous space endowed with the metric $d$ is defined as a pair $(S,Z_d)$ composed of a measured space $S$ and a similarity kernel $Z_d$ induced by the distance $d$.
\end{definition}
From the definition of heterogeneous space we can refer to the class of heterogeneous spaces and denote it $\mathcal{H}$ which is a subclass of the set of all measured spaces endowed with a similarity kernel. With this definition, we may now define the heterogeneity measure as a function defined on the class of heterogeneous spaces (c.f. \cite{Leinster2021} for complementary details).
\begin{definition}[Heterogeneity measure]
    Let $(S, Z) \in \mathcal{H}$ be a heterogeneous space and $(\mathscr{F},\mu)$ the measurable structure and the probability on $S$. We define the heterogeneity measure on $S$ as follows:
    $$\mathcal{H}^q(S,\mu, Z) = \left(\int_{x \in S} \left(\frac{1}{(Z \mu)(x)}\right)^{1-q}d \mu(x) \right)^{1/(1-q)},$$
    where the similarity measure $Z \mu$ is defined as:
    $$\forall x \in S \quad \left(Z\mu\right)(x) := \int_{y \in S} Z(x,y)d \mu(y).$$
\end{definition}

\subsection{Measuring heterogeneity in production systems}
We have introduced the heterogeneity measure and the classes of systems on which it is defined. We want now to apply this concept in our precise case of production systems.

\subsubsection{Heterogeneity of a production system: definition}
In the context of a production system, the set corresponds to the agents. There is a subtlety here. As the agents in the production system are organised in a network, they are all uniquely identified, therefore we always use the uniform distribution. We now construct the similarity matrix. The pairwise similarity between agents is key. The similarity here corresponds to a measure of the amount of skills shared by two agents. Because they are represented as wrapped Gaussian densities over the torus, we can use their mean and standard deviation to encode the amount of shared skills. On the torus, we have the usual distance $d_{\circ}$ which we use to define the similarity kernel. First we write,
$$\widetilde{d}^\mu_{ij} = d_{\circ}(\mu_i, \mu_j) =  \min\left(|\mu_i - \mu_j|,\ 2\pi - |\mu_i - \mu_j|\right).$$
We now use the standard deviations of the two agents to define a Fisher-type distance between two agents' skill densities:
$$d_{ij} = \frac{\widetilde{d}^\mu_{ij}}{\sqrt{\sigma_i^2 + \sigma_j^2}}.$$
The similarity kernel is then,
$$K_{ij} = \exp(- d_{ij}).$$
It is a proxy to measure the area of shared skills between agents $i$ and $j$.
For simplicity we fix $q = 2$. The heterogeneity
$$\mathcal{H}^2(Z,N) = \frac{N}{1 + \frac{2}{N}\sum_{1 \leq i < j \leq N} Z_{ij}},$$
yields a value of 1 for a completely homogeneous system (where all elements are identical). However, in many physical and economic contexts, it is more intuitive to define a perfectly homogeneous system as having zero heterogeneity. We therefore subtract one, and obtain the heterogeneity measure in our precise case:
$$\mathcal{H}(\mu, \sigma) = \frac{N}{1 + \dfrac{2}{N}\displaystyle\sum_{1 \leq i < j \leq N} \exp\left(-\frac{\widetilde{d}^\mu_{ij}}{\sqrt{\sigma_i^2 + \sigma_j^2}}\right)} - 1.$$

\subsubsection{Scaling behaviour and redundancy limits}
\label{app:metric_scaling_behaviour}

A critical requirement for our study is that the heterogeneity measure remains robust to changes in $N$ and reflects the overall dissimilarity skills. We verify that our metric does not present pathological scaling behaviours.

\paragraph{Homogeneous Scaling}
A pathological behaviour to avoid is increasing heterogeneity when adding similar agents. We consider a system that undergoes homogeneous scaling, where all $N$ agents are identical ($\mu_i = \mu_j$ and $\sigma_i = \sigma_j$ for all $i,j$). In this case, the distance $d_{ij} = 0$ for all pairs, and the similarity matrix $Z$ becomes the all-ones matrix. The heterogeneity is then
$$ \mathcal{H}(N) = \frac{N}{1 + \frac{2}{N}\frac{N(N-1)}{2}} -1 = \frac{N}{1 + N - 1} - 1 = 0.$$
This justifies that simply increasing the number of agents without introducing new skills (i.e., scaling a uniform workforce) correctly maintains zero heterogeneity.

\paragraph{Convergence of Redundant Systems}
Starting from a heterogeneous system, we demonstrate that adding "more of the same" agents (redundancy) always decreases the heterogeneity, converging to zero in the limit. Consider a simple system of two dissimilar agents ($N=2$) with a pairwise similarity $z = \exp(-d_{12}) < 1$. The initial heterogeneity is $\mathcal{H}(2) = \frac{2}{1+z} - 1  > 0$. Now, suppose we expand the system by adding $k$ clones of agent $1$. The new population size is $N = k+1$. Without loss of generality we assume that the $k$-clones are indexed by $(2, \dots, k+1)$. The heterogeneity is then, 
$$\mathcal{H}(k+1) = \frac{k + 1}{1 + \frac{2}{k+1} \left( \sum_{j = 2}^{k+1} z + \sum_{2 \leq i < j \leq k+1} 1\right)} - 1.$$
Arranging the term and we obtain: 
$$\mathcal{H}(k+1) = \frac{1}{1- \frac{2k(1-z)}{(k+1)^2}} - 1.$$
The sequence $(\mathcal{H}(k+1))_{k \geq 1}$ is a decreasing sequence that tends to 0.  As the system becomes increasingly dominated by a single type of agent, the relative heterogeneity vanishes. This confirms that our metric properly captures how compositionality "dilutes" in the face of redundant scaling.

%% file: Text/model_design_discussion.tex
\subsection{Global optimiser}
\label{app:model_design_optimiser}

Notably our model relies on a global optimiser which is in contrast to many models of distributed systems which have a tendency to model the overall system dynamic (at least partially) as a function of the direct pairwise competition of actors \cite{nash1951, fakhar2025human}. Our choice for a global optimiser has both a practical mathematical and systems dynamics reason. 

On the \textbf{practical mathematical side}, using a global optimiser makes it feasible for us to include complex workload and interaction formulations while still keeping a model system which has a clearly converging overall gradient that makes the model a tractable system to study as it is robust to trivial changes of parameters or initial conditions. So while the model overall is not fully analytically solvable, the convergence characteristics under global optimisation allows for the model to be analysed nearly as if it as analytically tractable. Adding complex pairwise competition interactions as a function of individual strategies across complex network configurations would push the model far away from stable convergence characteristics.

The more interesting reason can be found in the specific \textbf{systems dynamics} we are interested in. Note that the question we care about is what principles best describe the trajectory of the overall systems level setup of distributed production systems, to see whether these generalise across disciplines. As such, we do not need to make any statements about which individual agent will be the overall most productive one given a starting position or its strategy. Our model does not claim that the exact local competition does not matter for the system or is not an interesting topic of study, it does however assume that all the complex pairwise collaborative and competitive interactions ultimately accumulate to an overall systems level dynamics which looks like, and can be described as, a global optimisation process. An example for how local interactions can jointly act like a global optimisation process can for example be observed in the context of learning in the brain \cite{richards2023study} or the discussion of resource allocation and market efficiency in economics \cite{Hurwicz1977}.

\subsection{Arguments against a absolutely homogeneous \textit{Dirac Delta function} workload}
\label{app:dirac_workload}

In \cref{sec:forcing_workloads} we mentioned the model assumption that workloads are not following a Dirac Delta function and argued that this was a realistic assumption for the work faced by any large distributed system, not because they do not exist at all, but because they are extremely rare. Here we analyse what a Dirac Delta workload would mean and why it would not be relevant for large distributed systems.

\subsubsection{A Physics Perspective on Dirac Delta workloads and larger dimensional workload characteristics}
A Dirac Delta workload would imply that a single irreducible function (that could not be further decomposed) needed to be performed with a fixed (never varying) rate of executions per unit time. As mentioned in the main text, such tasks do exist and an example can be found in a clock oscillator in CPUs or quartz crystals as they perform a single irreducible operation, toggling a bit, at a fixed frequency. However, they are extremely rare, especially for the conjoined case where both criteria are true. To make this clear, we explain both criteria in more detail.

To start with irreducibility, this is necessary because if a workload was still reducible or decomposable into subcomponents, then it would not be single-valued in our description of workloads, as it would instead cover the width of all its subcomponents in the space that workloads are defined in. This is the case because now the workload would be made of multiple sub-skills. There may be an agent that perfectly matches the distribution of that decomposable workload which might make it seem to not require decomposing, but once solved with a large distributed system, the standard system-level heterogeneity of \cref{sec:principles} would follow in order to find the optimal solution across agents.

The above effect gets amplified through the fixed rate criteria which has a very related effect. For the sake of argument, we may assume that non decomposable were common or that there existed a group of workload for which decomposing was not possible. To have this workload then appear as a Dirac Delta workload, its rate of execution would also need to be constant. Our current description of the spectrum of tasks does not currently take time into account but one could do so by describing all workloads in 2D space, where one axis is the task characteristic and another is the speed of execution. We know both from basic material science \cite{vig1994quartz} but also system level design such as compute systems \cite{li2025towards}, that unique execution speeds demand their own unique physical function realisation. As such, the varying execution speed is just another 'task characteristic' that needs to be realised through the system's level heterogeneity, and so again the standard system-level heterogeneity of \cref{sec:principles} would follow, just now with respect to execution speed of the underlying material. We represented workloads and agents' skill to be represented on one dimension but factors like execution speed simple represent additional dimensions in which \TMHP would apply. Additional physical factors such as stochasticity of execution would likely similarly expand the task space.

\subsubsection{Additional workload dimensions from the perspective of computer architecture}
In the physic perspective we have already seen how even if a workload was single-valued in one dimension, the cases in which this would continue across all task dimension that are to be considered to find an optimal systems level design are incredibly rare. This is true even on a very low level of analysis of Physics. Naturally this effect compounds into very large dimensional space of workloads for systems in a more complex design space. We can consider computing systems as a well-understood example. Workload descriptions for computing systems would include the description of the actual function to be executed and also the execution frequency, as discussed above, but would include various additional factors. An example for computing systems would be factors relating to memory: functions executed may be of the same basic type but process different volumes of input variables which in turn require different memory access patterns that the systems architecture needs to be optimised for through memory bandwidth optimisation \cite{li2025towards}, which in turn means that any variance over volume of input would invoke \TMHP. To show that this considering spans into many more dimension would could also refer to irregular and sparse memory access pattern \cite{xie2021spacea, ahn2015scalable}. That this is not just limited to the computer architecture but also the algorithm layer can be seen in the effects of network dimension on processing ability \cite{safran2017depth} and the benefit of delay-based architectures for temporal processing \cite{sun2025algorithm, sun2025exploiting}.

\subsubsection{Cognitive consideration for the dimensionality of workloads}
The prior two sections made clear that even for cases where the systems-level heterogeneity predicted by \TMHP still follows in cases where the workload itself may look relatively narrow, due additional workload modulation factors that expand the effective dimensionality of workloads. We should however highlight, that this level of nuance is not relevant for most real-world processes that most cases of science, computing, and engineering are concerned with. As example studies replicated in \cref{sec:findings} highlight, most large scale distributed systems are faced with complex trade relationships, exploration and exploitation of ecological niches, or producing the complex behaviour of mammals navigating the wild. While for these complex that deal with the physical world it may be less tractable to fully describe the 'skill space' that they live in, this solely because of their intrinsic complexity which directly imply heterogeneity. This complexity of operations can be observed both in analyses of task solving in cognitive science and neuroscience \cite{simpson2020neurocognitive, kievit2020sensitive, duncan2025construction, smith2022fluid} and in analyses of production processes in economics \cite{carvalho2019production}.

\subsection{Mapping between layers of the model}
\label{app:between_layer_mapping}

The model formalises the relationship between the resource space, the skill space, and the operations space as identity mappings,
\[
\T_{\mathrm{res}} \xrightarrow{\;\mathrm{id}\;} \T_{\mathrm{skills}} \xrightarrow{\;\mathrm{id}\;} \T_{\mathrm{ops}},
\]
so that a position $\theta \in \T$ in one layer corresponds directly to the same position in the adjacent layer. This implies that skills in one layer translate one-to-one into resources or operations at the layer below, which is unlikely to hold in any realistic system: the cognitive operations produced by a neural network, for instance, do not map positionally onto the transistor-level operations executing them.

The key observation is that the Maximum Heterogeneity Principle does not depend on this positional identity, but only on the ordering structure of the skill space. To see this, let $\phi: \T \to \T$ denote the true mapping from skills in one layer to resources in the next, and let $\Omega \subseteq \T$ be the region of skill space over which the system's heterogeneity is concentrated, i.e.\ the support of the inter-agent variance. The heterogeneity measure $\mathcal{H}_q$ (Definition~1.2) depends on agent positions only through the pairwise similarity matrix $Z$, with $Z_{ij} = e^{-d(\mu_i, \mu_j)}$, so what matters is whether $\phi$ preserves the ordering of inter-agent distances on $\Omega$. Any strictly monotone mapping on $\Omega$ does exactly this: it relabels the skill axis without collapsing the distinctions between agents, leaving $\mathcal{H}_q$ invariant up to that relabelling. This means that a wide class of inter-layer mappings, including inverted or non-linearly rescaled ones, leave the heterogeneity ordering of agent configurations intact, and the conclusion that the optimal configuration is the maximally heterogeneous one carries through unchanged.

This reasoning also tolerates partial bottlenecks. If $\phi$ collapses some region $\T \setminus \Omega$ to a point, but remains order-preserving on $\Omega$ where the inter-agent variance lies, the effect on $\mathcal{H}_q$ is negligible by construction. What would invalidate the argument is a mapping that collapses $\Omega$ itself, a complete bottleneck on the region carrying the system's heterogeneity. Short of that, the heterogeneity scaling conclusion is robust to the specifics of how layers relate to one another.

\subsection{Why universal function approximators still have non-uniform approximation skills}
\label{app:ufa}

A natural objection to modelling neural networks as agents with Gaussian skill densities over a bounded space is that neural networks are universal function approximators: given sufficient capacity, they can approximate any function arbitrarily well. This might seem to imply flat coverage of the skill space, contradicting the model's premise. The same objection could in principle be raised for any system that carries universal function approximation properties, such as kernel methods or certain classes of symbolic systems.

The objection conflates existence with cost. Universal approximation is an existence result: for any target function and any $\varepsilon > 0$, a sufficiently expressive system can approximate it to within $\varepsilon$. It makes no claim about how much capacity, compute, or data that approximation requires, and crucially, those costs are not uniform across function classes. For any fixed architecture or system design and parameter budget, approximation error varies substantially across targets: some functions are approximated cheaply, others only at disproportionate cost. We would expect this to be a general property of any universal function approximator, since the existence guarantee says nothing about the geometry of the approximation cost landscape, and there is no reason to expect that landscape to be flat. This is precisely what a skill density captures. The mean $\mu_i$ encodes where in function space a given system is most efficient; the standard deviation $\sigma_i$ encodes how rapidly that efficiency degrades away from that optimum. Universal approximation guarantees the tails of the density are never exactly zero, but does not prevent them from being negligibly small under realistic resource constraints. In our model, this is directly captured by the choice of wrapped Gaussian skill densities, 
which are strictly positive everywhere on $\T$: every agent assigns non-zero density to every point in skill space, meaning any agent can in principle approximate any skill, however distant from its specialisation. Universal approximation is thus preserved in the model, but the Gaussian shape ensures that the cost of doing so grows with distance from $\mu_i$, recovering the non-uniformity that the existence guarantee alone does not preclude.

In the case of neural networks, two empirical signatures confirm this non-uniformity directly. Depth and width do not improve approximation uniformly: compositional and hierarchical functions benefit strongly from depth in ways that width cannot easily replicate \cite{safran2017depth}, revealing an architecture-dependent bias over function space. Similarly, networks with temporal processing capabilities approximate functions over sequential inputs far more efficiently than standard architectures \cite{sun2025algorithm, sun2025exploiting}, again demonstrating that the effective skill profile is shaped by architectural choices. While analogous systematic evidence is less comprehensively documented for other universal approximators, the structural argument applies equally: any system operating under finite resource constraints will approximate some function classes more efficiently than others, and the existence guarantee of universal approximation does not remove that constraint.

The Gaussian parametrisation is therefore not a simplification in tension with universal approximation, but a model of the non-uniformity that universal approximation leaves unaddressed. Because approximation cost is heterogeneous across function classes, a system of agents collectively covering a broad workload faces exactly the trade-off the model describes: specialisation is efficient but brittle, and heterogeneous coverage is the robust solution. This holds for neural networks specifically, and we expect it to hold for any class of universal function approximators operating under realistic constraints.

\subsection{The Maximum Heterogeneity Principle under graph-structured workloads}
\label{sec:workload_as_graph}
\label{app:workload_as_graph}
In \cref{sec:limitations}, we note that our model treats the workload as a flat density over the operations space. However, real workloads often carry internal dependency structure: operations must be executed in a specific order, forming a directed acyclic graph (DAG). We argue that this does not invalidate the \emph{Maximum Heterogeneity Principle} but instead maps onto mechanisms the model already contains.

Any DAG admits a topological decomposition into a sequence of levels $L_0, L_1, \dots, L_K$, where the operations within each level $L_k$ share no mutual dependencies and can be executed concurrently; only the inter-level edges impose ordering. Within each level, the workload is therefore a distribution over independent operations --- precisely the flat density that the model takes as input. Between levels, the completion of $L_k$ enables $L_{k+1}$: the aggregate output of the agents serving level~$k$ becomes the effective workload that level~$k{+}1$ must meet. This is structurally identical to the multi-layer chain $\mathbb{T}_0 \xleftarrow{\phi_1} \mathbb{T}_1 \xleftarrow{\phi_2} \cdots \xleftarrow{\phi_K} \mathbb{T}_K$ formalised in \cref{sec:MHP_recursive}, where the production function of one layer serves as the demand for the next. The recursive property of \cref{sec:MHP_recursive} then applies directly: because the optimal solution at each level is heterogeneous, the output it produces is itself a broad, distributed function over the operations space, which is exactly the condition under which the principle propagates to the subsequent level. Heterogeneity at the top of the DAG therefore induces heterogeneity all the way down, just as it does across the $\mathbb{T}_{\mathrm{cog}} \leftarrow \mathbb{T}_{\mathrm{algo}} \leftarrow \mathbb{T}_{\mathrm{hardware}}$ hierarchy.

An additional observation reinforces this conclusion. Inter-operation dependencies act as a structural constraint on the system, restricting which agents can process which operations and in what order. This is formally analogous to the communication topology constraints studied in \cref{sec:principle_topology}: topology bounds the heterogeneity the system can express via the radius of collaboration, but it never eliminates the drive towards it. Workload dependencies similarly bound the scheduling freedom of agents without removing the fundamental incentive to partition the operation space efficiently. Only a fully serial, single-operation chain --- the graph-theoretic analogue of a Dirac Delta workload (\cref{app:dirac_workload}) --- would collapse all levels to singletons and suppress heterogeneity, and such a chain is the degenerate case already identified as unrealistic for any large distributed system. The \emph{Maximum Heterogeneity Principle} therefore not only survives the transition to graph-structured workloads but is, if anything, strengthened by the additional structural differentiation that dependencies introduce among operations.

%% file: Text/technical_verifications.tex
In this appendix section we discuss the well-posedness of the gradient descent. We focus our discussion on the special case where there is only one agent. In this case the production function writes:
$$W_1(\theta; \mu, \sigma) = \frac{1}{\sqrt{2 \pi} \sigma} \sum_{k \in \Z} \exp\left(-{\frac{(\theta - \mu + 2 k \pi)^2}{\sigma^2}}\right).$$
The last subsection of this appendix draws some comments stating that the special case is extensible to the general case of multi-agents setting with more complicated production function.

\subsection{Preliminaries on the Wrapped Gaussian}
For details on the wrapped Gaussian distributions we refer the reader to \citep{Mardia2000}. We give here all the technical results that will be useful in the next subsections to study the well-posedness and convergence of the gradient descent. 

\subsubsection{Analyticity of the wrapped Gaussian}
In this paragraph, we prove the analyticity of the wrapped Gaussian density. We need this property to ensure the domain $W_0 = W_1(\mu, \sigma)$ has no accumulating points. Let's fix $\mu \in \T$, $\sigma > 0$ and write
$f :=W_1(\cdot; \mu, \sigma)$ the $2\pi$-periodic function. The Fourier coefficients of $f$ are: 
$$c_n(f) = \frac{1}{2 \pi} \int_0^{2 \pi} f(\theta) e^{-in \theta} d \theta = \frac{1}{2 \pi} e^{-i n \mu - \frac{1}{2}\sigma^2 n^2} .$$

\begin{proof}
    Here we compute the Fourier coefficient of the $2\pi$-periodic function $f$.
    $$c_n(f) =  \frac{1}{2 \pi} \int_0^{2 \pi} f(\theta) e^{-in \theta} d \theta =  \frac{1}{2 \pi} \int_0^{2 \pi} \frac{1}{\sqrt{2 \pi} \sigma} \sum_{k \in \Z} \exp\left(-{\frac{(\theta - \mu + 2 k \pi)^2}{\sigma^2}}\right) e^{-in \theta} d \theta.$$
 By inverting the integral and series, 
    $$c_n(f) = \frac{1}{2 \pi} \frac{1}{\sqrt{2 \pi} \sigma} \sum_{k \in \Z} \int_0^{2 \pi}   \exp\left(-{\frac{(\theta - \mu + 2 k \pi)^2}{\sigma^2}}\right) e^{-in \theta} d \theta.$$
    Using an affine change of variables $x = \theta + 2k \pi$, the $n$-th Fourier coefficient becomes,
    $$c_n(f) =  \frac{1}{2 \pi} \frac{1}{\sqrt{2 \pi} \sigma} \sum_{k \in \Z} \int_{2 k \pi}^{(2k + 1) \pi}   \exp\left(-{\frac{(x - \mu)^2}{\sigma^2}}\right) e^{-in (x + 2 k \pi)} d x.$$
    Since $e^{-2ik \pi n} = 1$ for all $n,k$, we obtain,
    $$c_n(f) = \frac{1}{2 \pi}   \int_{-\infty}^{\infty}  \frac{1}{\sqrt{2 \pi} \sigma} \exp\left(-{\frac{(x - \mu)^2}{\sigma^2}}\right) e^{-in x} d x.$$
    Noticing that $\int_{-\infty}^{\infty}  \frac{1}{\sqrt{2 \pi} \sigma} \exp\left(-{\frac{(x - \mu)^2}{\sigma^2}}\right) e^{-in x} d x $ is the characteristic function of the normal distribution evaluated at $-n$ allows us to finalize the proof.
\end{proof}
The sequence of Fourier coefficients of $f$, $\left(c_n(f) = \frac{1}{2 \pi}  e^{-i n \mu - \frac{1}{2}\sigma^2 n^2}\right)_{n \in \Z}$ decays ultra-fast at infinty, 
$$\forall n \in \Z, \quad |c_n(f)| \leq  e^{- \frac{1}{2}\sigma^2 n^2} \leq e^{- \frac{1}{2}\sigma^2 |n|}.$$
The Paley-Wiener theorem asserts that $f$ is analytic. 

\subsubsection{Fourier series of the production function}
The previous demonstration of the analyticity of the production function $f$ via the fast-decay of the Fourier coefficients sequence enables us to write the production in the spectral domain:
$$\forall \theta \in \T, \quad f(\theta) = \frac{1}{2 \pi} \left( 1 + 2\sum_{n =1} ^{+ \infty} e^{-\frac{1}{2}\sigma^2 n^2} \cos(n (\theta - \mu))\right).$$

\begin{proof}
Since $f$ is a density, it equals its Fourier series: 
$$f(\theta) = \sum_{n \in \Z} c_n(f) e^{in \theta}.$$ 
Replacing the Fourier coefficients by their expressions,
$$f(\theta) = \sum_{n \in \Z} \frac{1}{2 \pi} e^{-i n \mu - \frac{1}{2}\sigma^2 n^2} e^{in \theta}  = \frac{1}{2 \pi} \sum_{n \in \Z} e^{i n (\theta - \mu) - \frac{1}{2}\sigma^2 n^2}.$$
We separate from the symbol sum the case $n = 0$.
$$f(\theta) =\frac{1}{2 \pi} \left( 1 + \sum_{n =1} ^{+ \infty} e^{\frac{1}{2}\sigma^2 n^2} \left( e^{i n (\theta - \mu)} +  e^{-i n (\theta - \mu)}\right)\right).$$
The use of Euler formulas finishes the proof.
\end{proof}

\subsubsection{On the strict unimodality of the Wrapped Gaussian}
A property of the wrapped Gaussian, on which we will rely on, is its strict unimodality. We do not prove this result here, we only state the result. The reader can refer to \cite{wintner1933shape, mardia2000directional} for a more analysis statistician approach and to \cite{levy1939addition, polya1930quelques} for a more mathematical physics approach. For any $\sigma > 0$ and $\mu \in \mathbb{T}$, the wrapped Gaussian density on $\mathbb{T}$, $W_1(\cdot; \mu, \sigma)$ possesses exactly two critical points: 
\begin{enumerate}
    \item A unique global maximum reached at $\theta = \mu$, and
    \item a unique global minimum reached at the antipode, $\theta = \mu + \pi$.
\end{enumerate}
Consequently, the wrapped Gaussian density is strictly decreasing on $[\mu, \mu + \pi]$, and strictly increasing on $[\mu + \pi, \mu + 2\pi]$.

With the expression of the density via its Fourier series, we can directly compute the exact values of the extrema of $W_1$. By evaluating the series at $\theta = \mu$ and $\theta = \mu + \pi$, we have: 
$$\max W_1 = W_1(\mu; \mu, \sigma) = \frac{1}{2 \pi} \left( 1 + 2\sum_{n=1}^{+\infty} e^{-\frac{n^2 \sigma^2}{2}} \right)$$
and
$$\min W_1 = W_1(\mu + \pi; \mu, \sigma) = \frac{1}{2 \pi} \left( 1 + 2\sum_{n=1}^{+\infty} (-1)^n e^{-\frac{n^2 \sigma^2}{2}} \right)$$

\subsection{Regularity of the loss}
The previous subsection has made clear some properties of the wrapped Gaussian densities, we are going to use them to prove the well definition of the loss and study its regularity.
We recall the loss definition here: 
$$\mathcal{L}: (\mu, \sigma) \in \T \times \R_+^* \mapsto  \left\lVert \frac{W_0 - W_1(\mu,\sigma)}{W_1(\mu,\sigma)}\,\mathbf{1}_{\{W_0 > W_1(\mu,\sigma)\}} \right\rVert_{p},$$
with $p > 1.$ First we prove that the loss is well defined and then we study the regularity of the loss. We assume that the demand $W_0$ is an analytical function on the torus.

\subsubsection{Definition of the loss}
Based on the analyticity and the strictly positivity of the the wrapped Gaussian density, for any parameters $(\mu, \sigma) \in \T \times \R_+^*$, the function
$$\theta \mapsto \frac{W_0(\theta) - W_1(\theta; \mu, \sigma)}{W_1(\theta; \mu, \sigma)},$$
is analytical. According the theorem of isolated zeros, the function,
$$\theta \mapsto W_0(\theta) - W(\theta; \mu, \sigma),$$
has a finite number of isolated zeros on the compact $\T$. Consequently, the indicator function has finitely many discontinuities, and the the integrand function,
$$ \theta \mapsto \left( \frac{W_0(\theta) - W_1(\theta; \mu, \sigma)}{W_1(\theta; \mu, \sigma)} \mathbf{1}_{\lbrace W_0(\theta) - W_1(\theta; \mu, \sigma) > 0 \rbrace } \right)^p,$$
is piecewise continuous.

To rigorously apply the dominated convergence theorem and ensure the integral is mathematically sound, we must show that the integrand is upper-bounded by an integrable function independent of the parameters $\mu$ and $\sigma$. To prevent the denominator from vanishing, we restrict our parameter space to a compact set $\mathcal{K} = \T \times [\sigma_{\min}, \sigma_{\max}]$ with $\sigma_{\min} > 0$. From our previous analysis of the strict unimodality, we know the minimum of $W_1$ with respect to $\theta$ is reached at the antipode $\mu + \pi$. Furthermore, as the variance $\sigma^2$ decreases, the density becomes more concentrated around $\mu$, causing the value at the antipode to drop. Thus, the absolute lowest value the denominator can take on our compact set $\mathcal{K}$ occurs precisely at the lower bound $\sigma = \sigma_{\min}$:
$$\forall (\mu, \sigma) \in \mathcal{K}, \forall \theta \in \T, \quad W_1(\theta; \mu, \sigma) \geq W_1(\mu + \pi; \mu, \sigma_{\min}) := \delta_{\min} > 0.$$
Therefore, $\delta_{\min}$ explicitly defines a global lower bound on $\mathcal{K}$ independent of $\mu$ and $\sigma$. By using this lowerbound on $W_1$ on $\mathcal{K}$, we can obtain the following upper bound: 
$$\left| \left( \frac{W_0(\theta) - W_1(\theta; \mu, \sigma)}{W_1(\theta; \mu, \sigma)} \right)^p \mathbf{1}_{\{W_0 > W_1(\mu,\sigma)\}} \right| \leq \left( 1 + \frac{W_0(\theta)}{\delta_{\min}} \right)^p,$$
which is an integrable function on the torus independent of $(\mu, \sigma) \in \mathcal{K}$. By the dominated convergence theorem, the loss $ \mathcal{L}$ is well-defined, strictly finite, and continuous on any compact $\mathcal{K} = \T \times [\sigma_{\min}, \sigma_{\max}]$ with $\sigma_{\min} > 0$, so on $\T \times \R_+^*$.

\subsubsection{Regularity of the loss}

We want to prove now that the loss function is globally $\mathcal{C}^1$ with respect to the parameters $(\mu, \sigma)$ on the compact space $\mathcal{K} = \T \times [\sigma_{\min}, \sigma_{\max}]$ to guarantee the stability of gradient descent methods. Equivalently, we study $\mathcal{L}(\mu, \sigma) := \mathcal{L}^p(\mu, \sigma)$. 

Let us write the loss as such,
$$\mathcal{F}(\mu, \sigma) = \int_{\T} \left( \max\left(0, \frac{W_0(\theta) - W_1(\theta; \mu, \sigma)}{W_1(\theta; \mu, \sigma)}\right) \right)^p d\theta,$$
and note $\eta \in \{\mu, \sigma\}$ and $E(\theta, \eta) = \frac{W_0(\theta) - W_1(\theta; \eta)}{W_1(\theta; \eta)}$ be the relative error. Our integrand is $g(\theta, \eta) = \max(0, E(\theta, \eta))^p$. Because $W_1 \geq \delta_{\min} > 0$ on $\mathcal{K}$, the function $E(\theta, \eta)$ is analytically differentiable with respect to $\eta$. Furthermore, because $p > 1$, the mapping $x \mapsto \max(0, x)^p$ is continuously differentiable on $\R$, with derivative $x \mapsto p \max(0, x)^{p-1} \mathbf{1}_{\{x>0\}}$. By composition, the integrand $g(\theta, \eta)$ is globally continuously differentiable ($\mathcal{C}^1$) with respect to $\eta$ on $\T \times \mathcal{K}$. Its partial derivative is:
$$\frac{\partial g}{\partial \eta}(\theta, \eta) = p \max(0, E(\theta, \eta))^{p-1} \mathbf{1}_{\{E(\theta, \eta) > 0\}} \frac{\partial E}{\partial \eta}(\theta, \eta).$$

Because the domain of integration $\T$ is fixed and compact, and the integrand's partial derivative $\frac{\partial g}{\partial \eta}$ is continuous and uniformly bounded on $\T \times \mathcal{K}$, we can rigorously apply the standard Leibniz integral rule (differentiation under the integral sign). We can pass the derivative directly inside the integral without requiring any transversality assumptions regarding how $W_0$ and $W_1$ intersect:
$$\nabla_\eta \mathcal{L}(\mu, \sigma) = \int_{\T} p \left( \frac{W_0(\theta) - W_1(\theta; \mu, \sigma)}{W_1(\theta; \mu, \sigma)} \right)^{p-1} \mathbf{1}_{\{W_0 > W_1\}} \nabla_\eta \left( \frac{W_0(\theta) - W_1(\theta; \mu, \sigma)}{W_1(\theta; \mu, \sigma)} \right) d\theta.$$

The gradient is an integral of a continuous, bounded function over a fixed compact domain, meaning the gradient itself is continuous. Therefore, the loss $\mathcal{L}^p$ is globally $\mathcal{C}^1$ on $\T \times [\sigma_{\min}, \sigma_{\max}]$. This strict regularity is the fundamental condition required to ensure that gradient descent remains well-defined and converges smoothly.

%% file: Text/experiment_parameters.tex
\subsection{Workloads used for testing of principles}
\label{app:principles_workload_overview}

For analyses in \cref{sec:principles} and in \cref{sec:findings}, we need to account for various different workload shapes. Figures \ref{app_fig:combined_workload_catalogues} and \ref{app_fig:workloads_eff}, show overviews of the different catalogues of workloads used.

\begin{figure}[H]
    \centering
    \begin{subfigure}{0.32\textwidth}
        \centering
        \includegraphics[width=\linewidth]{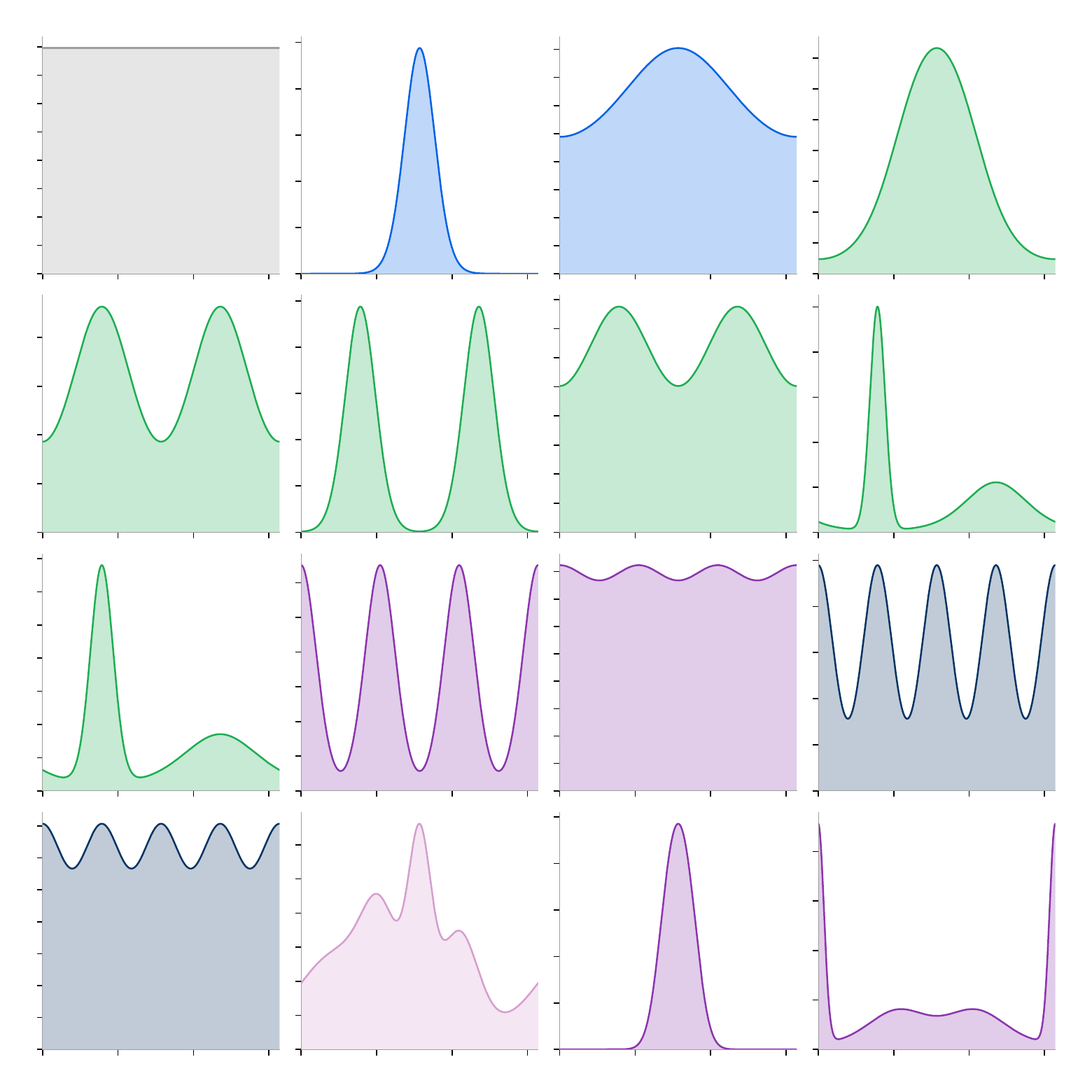}
        \caption{Principles workload overview.}
        \label{app_fig:workloads}
    \end{subfigure}
    \hfill
    \begin{subfigure}{0.32\textwidth}
        \centering
        \includegraphics[width=\linewidth]{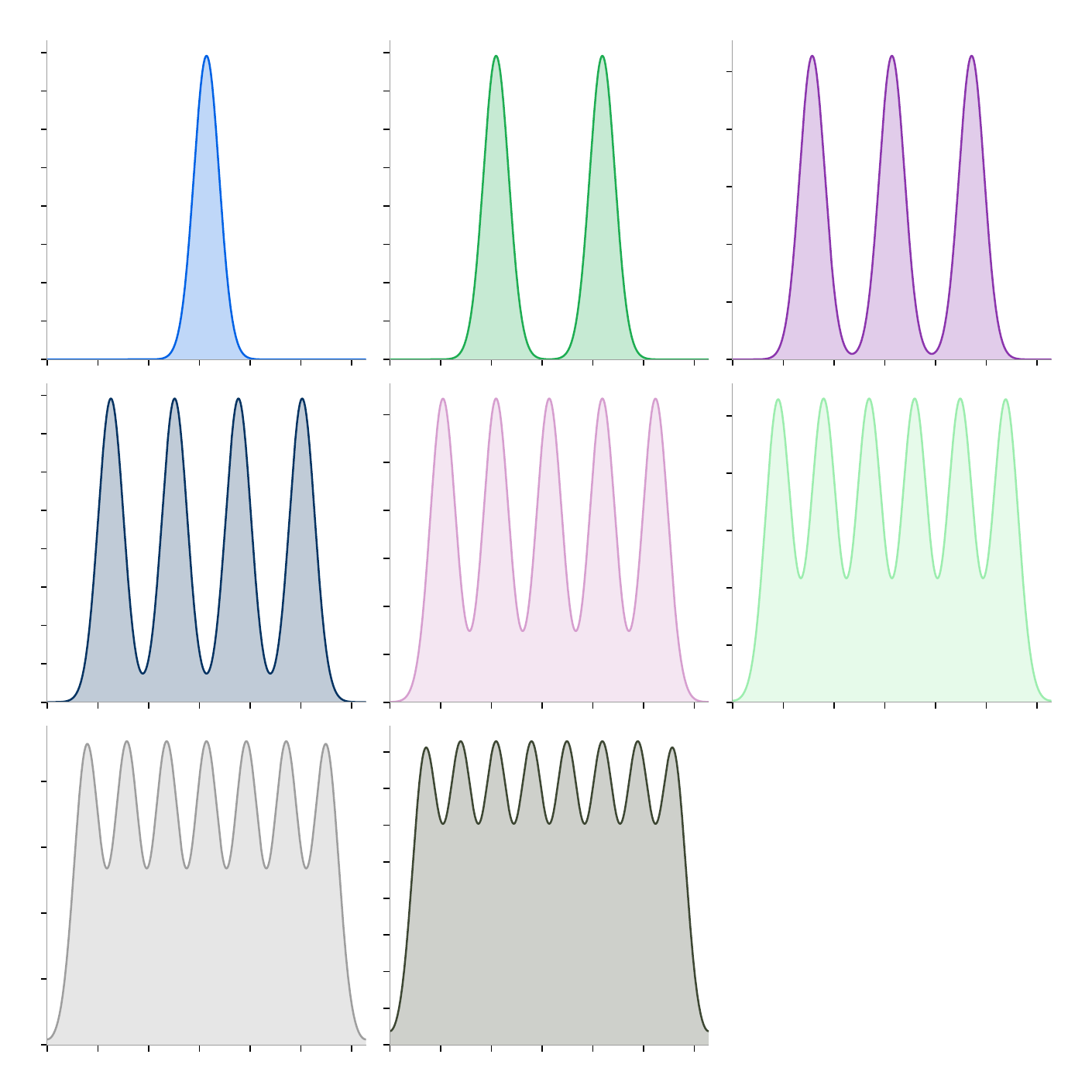}
        \caption{Multi-Gaussian series (1--9).}
        \label{app_fig:workloads_pics}
    \end{subfigure}
    \hfill
    \begin{subfigure}{0.32\textwidth}
        \centering
        \includegraphics[width=\linewidth]{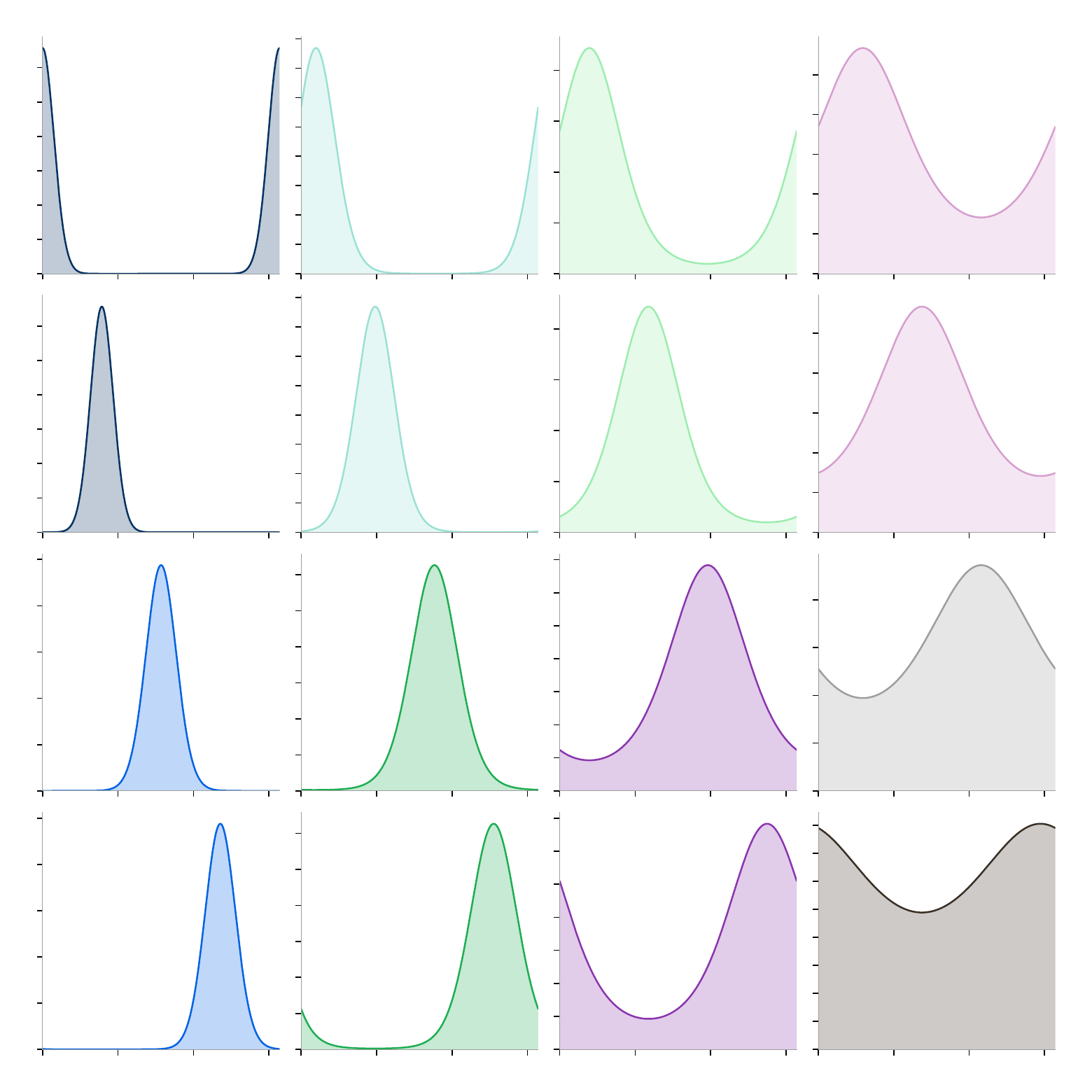}
        \caption{Unimodal robustness workloads.}
        \label{app_fig:workloads_unimodal}
    \end{subfigure}

    \caption{Comprehensive overview of tasks used in the model: (a) general principles, (b) multi-Gaussian scaling, and (c) robustness testing.}
    \label{app_fig:combined_workload_catalogues}
\end{figure}

\begin{figure}[H]
    \centering
    \includegraphics[width=0.8\textwidth]{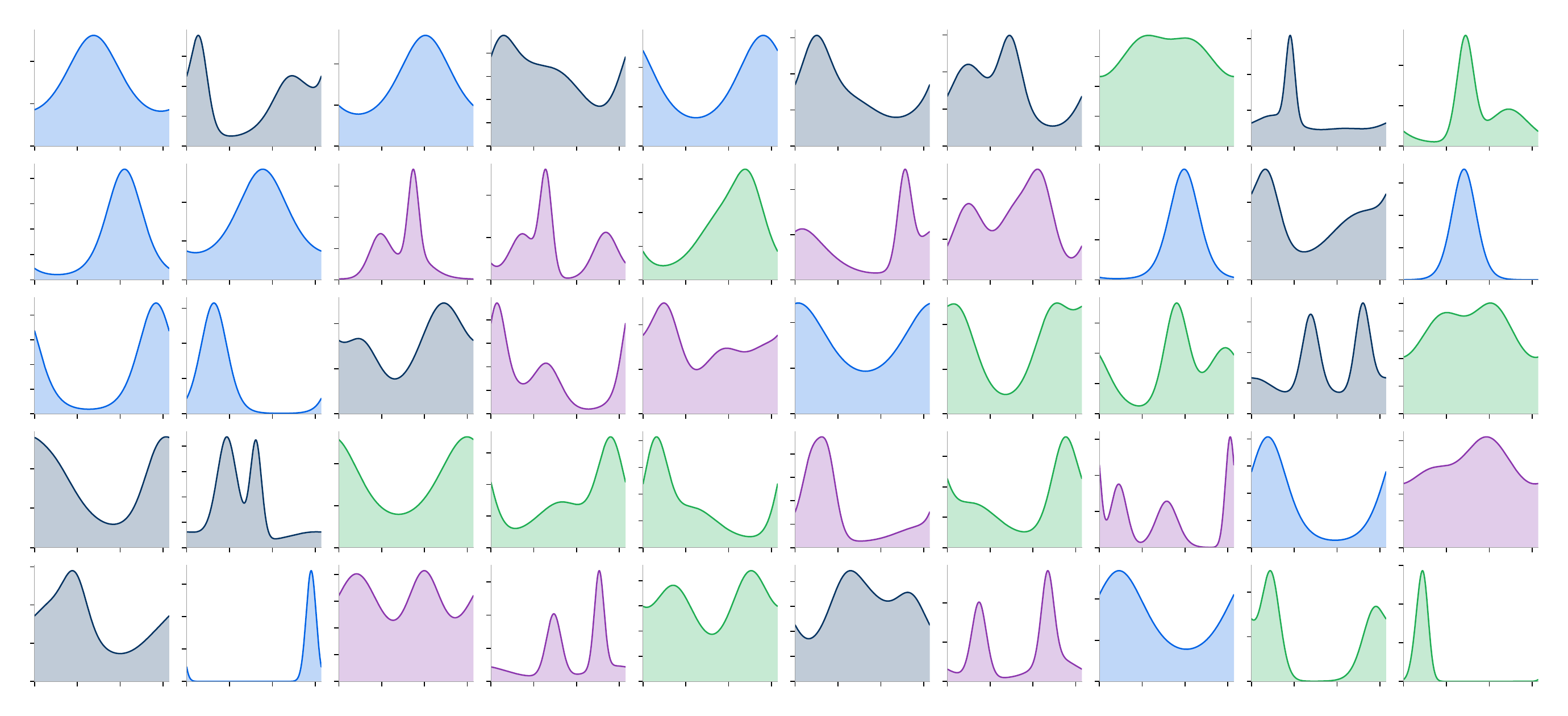}
    \caption{Overview of workloads used for extracting the efficiency principle.}
    \label{fig:efficiency_catalogue}
    \label{app_fig:workloads_eff}
\end{figure}

\subsection{Metrics to characterise workload during analyses of principles}
\label{app:principles_workload_metrics}

For analyses in \cref{sec:principles_workloads}, we have studied the effect of different features of the task onto the optimised system. We detail them in this appendix subsection. Let $W_0(\theta)$ be the task demand profile defined on the torus, discretised over a grid of $L = 2048$ points with spacing $\Delta\theta$. We define the normalised density
\[
  \tilde{W}_0(\theta) \;=\; \frac{W_0(\theta)}{\displaystyle\sum_{i=1}^{L} W_0(\theta_i)\,\Delta\theta}\,,
\]
so that $\sum_i \tilde{W}_0(\theta_i)\,\Delta\theta = 1$.

\medskip
\noindent The 12 features are grouped into four families.

\paragraph{Parametric features.}

\begin{enumerate}

\item \textit{Peaks number} ($n_{\text{peaks}}$).
The number of Gaussian components in the mixture defining $W_0(\theta)$.

\item \textit{Spread} ($s_{\max}$).
The largest pairwise circular distance between peak centres:
\[
  s_{\max} \;=\; \max_{i < j}\; \min\!\Big(\lvert \mu_i - \mu_j \rvert,\; 2\pi - \lvert \mu_i - \mu_j \rvert\Big).
\]
For unimodal tasks ($n_{\text{peaks}} = 1$), we set $s_{\max} = 0$.

\end{enumerate}

\paragraph{Information-theoretic features.}

\begin{enumerate}\setcounter{enumi}{2}

\item \textit{Shannon entropy} ($\mathcal{S}$).
The differential entropy of the normalised profile:
\[
  \mathcal{S} \;=\; -\sum_{i=1}^{L} \tilde{W}_0(\theta_i)\,\log\!\big[\tilde{W}_0(\theta_i)\big]\;\Delta\theta\,.
\]
Higher values indicate a more spread-out, uniform-like demand; lower values indicate concentration around few modes.

\item \textit{R\'enyi entropy} ($\mathcal{R}_2$).
The R\'enyi entropy of order 2:
\[
  \mathcal{R}_2 \;=\; -\log\!\left(\sum_{i=1}^{L} \tilde{W}_0(\theta_i)^{\,2}\;\Delta\theta\right).
\]
This is the logarithm of the inverse participation ratio and provides a robust measure of the effective support of $W_0$.

\end{enumerate}

\paragraph{Fourier features.}

Let $\bar{W}_0 = \frac{1}{L}\sum_i W_0(\theta_i)$ and define the centred profile $W_0^{c}(\theta) = W_0(\theta) - \bar{W}_0$. Denote by $\{\hat{c}_k\}_{k=0}^{\lfloor L/2 \rfloor}$ the discrete Fourier coefficients of $W_0^{c}$ and the spectral power $P_k = |\hat{c}_k|^2$.

\begin{enumerate}\setcounter{enumi}{4}

\item \textit{Fourier energy 1} ($F_1$).
The fraction of spectral power in the first harmonic:
\[
  F_1 \;=\; \frac{P_1}{\displaystyle\sum_{k} P_k}\,.
\]

\item \textit{Fourier energy 2} ($F_2$).
The fraction of spectral power in the second harmonic:
\[
  F_2 \;=\; \frac{P_2}{\displaystyle\sum_{k} P_k}\,.
\]

\item \textit{Fourier energy 3} ($F_3$).
The fraction of spectral power in the third harmonic:
\[
  F_3 \;=\; \frac{P_3}{\displaystyle\sum_{k} P_k}\,.
\]

\item \textit{Fourier concentration} ($F_{\text{conc}}$).
The cumulative fraction of spectral power in the first three harmonics:
\[
  F_{\text{conc}} \;=\; \frac{P_1 + P_2 + P_3}{\displaystyle\sum_{k} P_k}\,.
\]
Values close to 1 indicate that the task shape is well described by low-frequency components; values near 0 indicate fine-grained or high-frequency structure.

\end{enumerate}

\paragraph{Geometric features.}

\begin{enumerate}\setcounter{enumi}{8}

\item \textit{Heterogeneity low} ($\mathcal{H}_{\text{low}}$).
It is the heterogeneity of the task detailed in \cref{app:heterogeneity_metric}, which corresponds to a dissimilarity-based inverse concentration measure. Define the circular distance matrix, 
$$\forall \theta, \theta' \in \T, \quad d_{\circ}(\theta_i, \theta_j) := \min\!\big(|\theta - \theta'|,\; 2\pi - |\theta - \theta'|\big)$$
and the kernel $Z(\theta, \theta') = \exp(-d_{\circ}(\theta, \theta')/\pi)$. The heterogeneity of the task $W_0$ is defined as such, 
$$\mathcal{H}_{\text{low}}(W_0) = \left(\int_{\theta \in \T} (Z W_0)(\theta) d \theta \right)^{-1},$$
with the effective density $Z W_0$ defined as follows:
$$Z W_0 (\theta) = \int_{\theta \in \T} Z(\theta,\theta')W_0(\theta')d \theta'.$$
With the discretisation $(\theta_i)_{1 \leq i \leq L}$, we approximate the heterogeneity of the task with,
\[
  \mathcal{H}_{\text{low}} \;=\; \frac{1}{\displaystyle(\Delta\theta)^2\;\tilde{\mathbf{W}}_0^{\!\top} Z\, \tilde{\mathbf{W}}_0}.\,,
\]
where $\tilde{\mathbf{W}}_0$ is the vector of $\tilde{W}_0(\theta_i)$ values. This captures how dispersed the demand is: concentrated profiles yield small $\mathcal{H}_{\text{low}}$; spread profiles yield large values.

\item \textit{Peak height} ($W_{0,\max}$).
The maximum value of the unnormalised demand:
\[
  W_{0,\max} \;=\; \max_{\theta}\; W_0(\theta)\,.
\]

\item \textit{Flatness} ($\kappa$).
The normalised kurtosis of the profile, measuring how peaked or flat the shape is:
\[
  \kappa \;=\; \frac{\displaystyle\sum_{i=1}^{L} \big(W_0(\theta_i) - \bar{W}_0\big)^4\;\Delta\theta}{\left[\displaystyle\sum_{i=1}^{L} \big(W_0(\theta_i) - \bar{W}_0\big)^2\;\Delta\theta\right]^2}\,.
\]
High $\kappa$ indicates sharp, spiky profiles; low $\kappa$ indicates broad, plateau-like shapes.

\item \textit{Circular variance} ($V_{\text{circ}}$).
Measures the directional spread of the normalised profile on the unit circle:
\[
  V_{\text{circ}} \;=\; 1 \;-\; \left\lvert\, \sum_{j=1}^{L} \tilde{W}_0(\theta_j)\, e^{i\,\theta_j}\;\Delta\theta \,\right\rvert,
\]
where $i$ denotes the complex number such that $i^2 = -1$.
$V_{\text{circ}} = 0$ when $W_0$ is a perfect point mass; $V_{\text{circ}} \to 1$ when $W_0$ is uniform.

\end{enumerate}

%% file: Text/network_effect.tex
In this appendix subsection we justify the extensive use of the circular topology of the interaction graph, show implications of taking other geometries, as in \cref{fig:findings_MD}.

\subsection{Canonical network}
Throughout the study we used the circular topology as nearly the sole candidate for our studies. The reason is that its corresponding communication matrix (adjacency matrix of the network),
$$Q = J_N = \begin{pmatrix} 
    0 & 1 & 0 & \cdots & 0 & 1 \\
    1 & 0 & 1 & \ddots & \vdots & 0 \\
    0 & 1 & \ddots & \ddots & 0 & \vdots \\
    \vdots & \ddots & \ddots & \ddots & 1 & 0 \\
    0 & \cdots & 0 & 1 & 0 & 1 \\
    1 & 0 & \cdots & 0 & 1 & 0
\end{pmatrix},$$
can be considered as a canonical example of our study, which stems from the production function. To show it we fix a linear production function with respect to the vector of individual production functions, i.e.:
$$w \in \R^N \mapsto W_1(Q, w) \quad \text{linear}.$$
Depending on the topology, there exists a vector of weights $$\alpha(Q) = (\alpha_1(Q), \dots, \alpha_N(Q))^T \in \R^N$$ such that:
$$W_1(Q, w) = \alpha(Q)^T w = \sum_{1 \leq i \leq N} \alpha_i(Q)w_i.$$
We call $\alpha(Q)$ the agent weights. Hence, in the linear case the network effect is entirely encompassed within the weights. As the heterogeneity is the core variable of study, we want to move away, when it is not needed, any other variables. Therefore we do not want to have a different weighting on the agents as it would influence the heterogeneity of the agents' capacities. Therefore any network whose interaction matrix $Q$ leads to a vector of weights lying on the principal diagonal can be a candidate for canonicity. 
Hence, we fix $J_N$ as it represents an undirected graph and has a one-norm of $2N$ thus accounting for the resources needed to fully fuel the production system in a natural way: one resource unit per agent and one per interacting edge. We acknowledge that other conventions could have been fixed, and we study the influence of $\alpha$ and the communication structure further.
\subsection{Weights' effects on optimal agents}
Before investigating the effects of $\alpha$ on the optimal agents, we first introduce some notations. The average and deviation of $\alpha$:
$$\bar{\alpha} = \frac{1}{N} \sum_{1 \leq i \leq N} \alpha_i, \quad \delta\alpha = (\delta\alpha_1, \dots, \delta\alpha_N)^T = (\bar{\alpha} - \alpha_1, \dots, \bar{\alpha} - \alpha_N)^T.$$
With these notations we can define the mean-field production function, 
$$\bar{W_1}(\alpha, w) := \bar{\alpha}^T w = W_1(\bar{\alpha}, w), $$
and the deviation production function, 
$$\delta W_1(\alpha, w) = \delta\alpha^T w = W_1(\delta\alpha, w).$$
The production function is the sum of the mean-field and the deviation production functions: 
$$W_1(\alpha, w) = \bar{W_1}(\alpha, w) + \delta W_1(\alpha, w).$$
\begin{remark}
    The previous equality directly stems from the linearity: 
    $$W_1(\alpha, w) = W_1(\bar{\alpha} + \delta \alpha, w) =  W_1(\bar{\alpha}, w) + W_1(\delta\alpha, w).$$
\end{remark}
We now empirically show in a simple case that for any agent $i$, the higher its weight deviation $\delta \alpha_i$ is from the mean field $\bar{\alpha}$, the more its specialisation deviates from what it would have had with a weight equal to the mean field. We first consider a simple setting of five agents and a fixed task with high resource constraints. We consider a weights vector $\alpha = (1, 2, 3, 4, 5)^T$, so the second agent contributes twice more than the first, the third contributes $3$ times more than the first, and so on until the fifth. The production function is:
$$W_1(\alpha, w) = \sum_{i = 1}^5iw_i.$$
In order to remove the effect of the stochasticity of the gradient descent procedure, we perform Monte Carlo estimates of each agent's specialisation levels and we output the Monte Carlo averages alongside their standard deviation intervals. \cref{fig:alpha_to_specialisation} clearly shows that each agent's specialisation level decreases with its corresponding weight. In this particular setup, the first agent is nearly twice as specialised as the last agent, compared to the mean field, in \cref{fig:alpha_mf_to_specialisation}, where specialisation is drastically more uniform across agents. 
\begin{figure}[H]
    \centering
    \begin{subfigure}{0.48\textwidth}
        \centering
        \includegraphics[width=\linewidth]{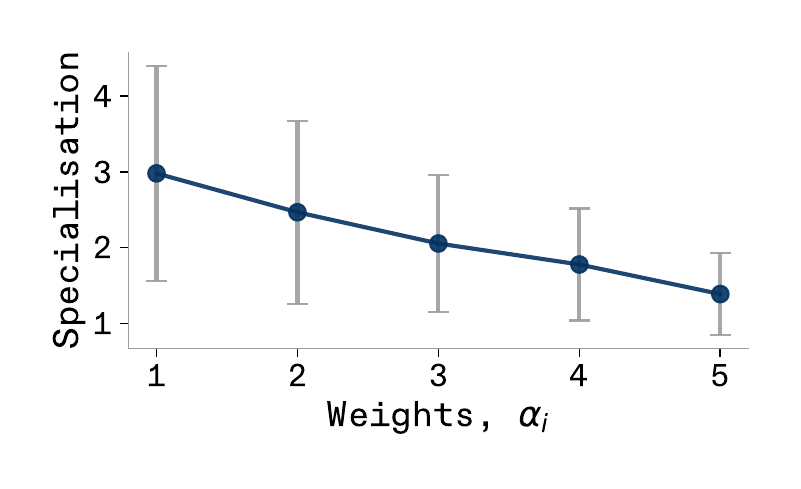}
        \caption{Specialisation levels with different agent weights.}
        \label{fig:alpha_to_specialisation}
    \end{subfigure}
    \hfill
    \begin{subfigure}{0.48\textwidth}
        \centering
        \includegraphics[width=\linewidth]{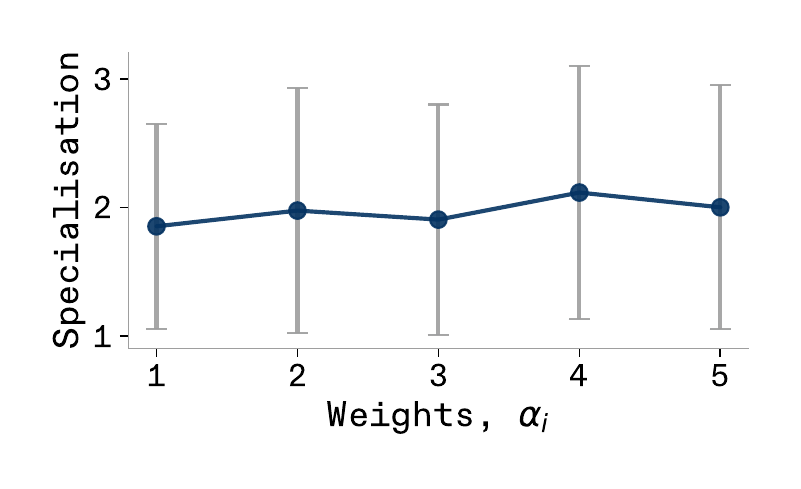}
        \caption{Specialisation levels with uniform agent weights.}
        \label{fig:alpha_mf_to_specialisation}
    \end{subfigure}
    \caption{Comparison of specialisation levels under different agent weighting schemes.}
    \label{fig:combined_specialisation}
\end{figure}
To observe the comparison we overlay the two preceding results to visualise the specialisation deviation as a function of the weight deviations: 
$$\delta S = S(\alpha) - S(\bar{\alpha}) = (S(\alpha_1) - S(\bar{\alpha}), \dots, S(\alpha_N) - S(\bar{\alpha}))$$
\begin{figure}[H]
    \centering
    \includegraphics[width=0.65\textwidth]{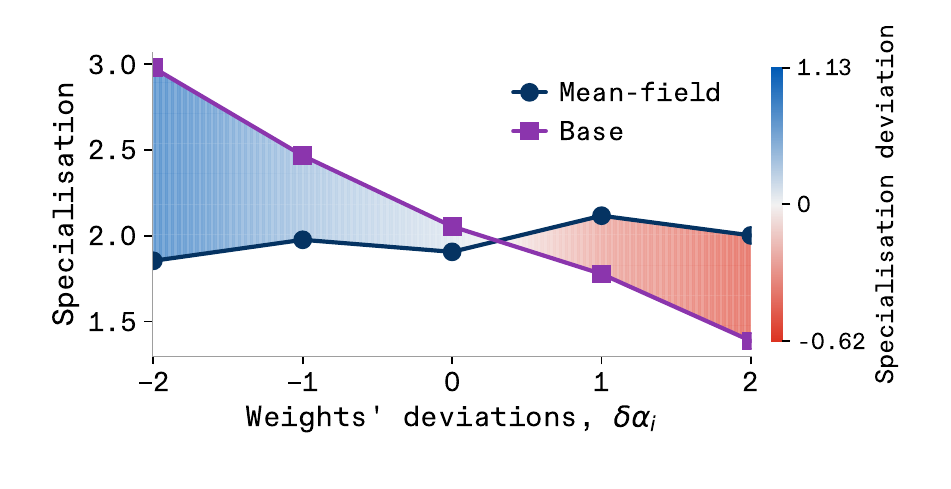}
    \caption{Specialisation deviations with deviations of agents' weights.}
    \label{fig:base_mf_specialisation_deviation}
\end{figure}
These observations illustrate, in a simple case, the specialisation deviation property with weight deviations from the mean field.

%% file: Text/supplementary_principles.tex
\subsection{Relation between productivity, efficiency and heterogeneity}
\label{app:prod_eff_het}
While the results in \cref{sec:principles_efficiency} establish a strong correlation between heterogeneity and performance under resource constraints, we further ask whether this relationship holds independently of the resource constraints themselves, that is, whether heterogeneity carries additional explanatory power for performance beyond what resource scarcity alone would predict. To test this, we conduct an incremental predictive analysis. We consider the same setting as in \cref{sec:principles_efficiency}, i.e.\ $N = 16$ agents interacting via $J_N$, subject to a varying stiffness parameter, jointly optimise to realise a bi-Gaussian workload under resource constraints. For each combination of resource constraint level $R$ and heterogeneity stiffness parameter $S$, we record the resulting heterogeneity $\mathcal{H}$ and performance $P$ of the corresponding optimal production system. We then compare two nested polynomial regression models: a full model $M_1: P \sim \mathrm{poly}(R, \mathcal{H})$ and a restricted model $M_2: P \sim \mathrm{poly}(R)$, using an $F$-test to assess whether including $\mathcal{H}$ significantly reduces prediction error beyond $R$ alone. 

To examine whether this relationship is regime-dependent, we partition the data into three resource-constraint regimes (terciles) and repeat the comparison with per-regime linear models. To ensure robust evaluation, the data within each regime is randomly divided into a training set (70\%) and a test set (30\%). The models are fitted exclusively on the training data, and their prediction errors ($E_1$ and $E_2$) are calculated based on their performance on the unseen test set. The results are shown in \cref{tab:granger_efficiency}. Globally, the inclusion of heterogeneity significantly improves predictive accuracy across the entire dataset ($p < 0.001$). When partitioned by regime, we observe that while heterogeneity offers a modest but significant predictive improvement even under abundant resources ($p = 0.009$), its explanatory power drastically intensifies as resources become scarce. Under severe resource constraints, the $F$-statistic surges to $99.81$ ($p < 0.001$) and the test set prediction error ratio $E_2/E_1$ leaps to $2.31$, meaning the model's accuracy on unseen data more than doubles when heterogeneity is considered. This confirms that the performance–heterogeneity relationship is not simply a reflection of resource scarcity: heterogeneity carries independent explanatory power for performance, and this predictive power becomes critically important as resources grow scarce. 
\begin{table}[H]
    \centering
    \begin{tabular}{lccc}
        \toprule
        \textbf{Regime} & $E_2 / E_1$ & $F$ & $p$-value \\
        \midrule
        Low $R$ (abundant) & 1.08 & 7.21 & 0.009 \\
        Mid $R$            & 1.07 & 7.61 & 0.007 \\
        High $R$ (scarce)  & 2.31 & 99.81 & $< 0.001$ \\
        \bottomrule
    \end{tabular}
    \caption{\textbf{Heterogeneity carries independent explanatory power for performance under resource constraints.} $E_2/E_1$ is the ratio of prediction errors of the restricted model ($M_2$, resource constraints only) to the full model ($M_1$, resource constraints and heterogeneity), and $F$ is the corresponding $F$-statistic. While heterogeneity provides a marginal predictive benefit under resource abundance, its explanatory power for performance becomes dominant and highly significant as resources grow scarce.}
    \label{tab:granger_efficiency}
\end{table}

\subsection{Additional details on the robustness principles of \cref{sec:principles_robustness}}
\label{app:add_principles_robustness}

We give here detailed results and illustrations supplementing \cref{sec:principles_robustness} by disaggregating robustness results by catalogues. 

\paragraph{Unimodal catalogue.}
\cref{fig:unimodal_example_timeseries} shows example time series for one
representative task under each regime.
Across all three regimes the heterogeneous system maintains a steadier, higher
production floor, whereas the homogeneous system exhibits large oscillations
that periodically collapse to near-zero output.
The summary statistics (\cref{fig:unimodal_summary_robustness}) confirm
this pattern across all $16$~tasks.
Under wave evolution, the heterogeneous system achieves a $134\%$ higher mean
production ($P_{\mathrm{heterogeneous}}=0.049$ vs.\
$P_{\mathrm{homogeneous}}=0.021$, $t_{15}=7.24$, $p<10^{-5}$) and a $56\%$
lower CV ($0.52$ vs.\ $1.17$, $p<10^{-6}$), with all $16/16$ tasks favouring
the heterogeneous system on every metric.
Under BM the advantage persists: mean minimum production is $7\times$ higher
for the heterogeneous system ($0.034$ vs.\ $0.005$, $p=0.001$), and under
extreme events the heterogeneous system's production floor is $16\times$ higher
($0.022$ vs.\ $0.001$, $p<10^{-6}$), with $16/16$ tasks showing a better
minimum.
\begin{figure}[H]
    \centering
    \includegraphics[width=\linewidth]{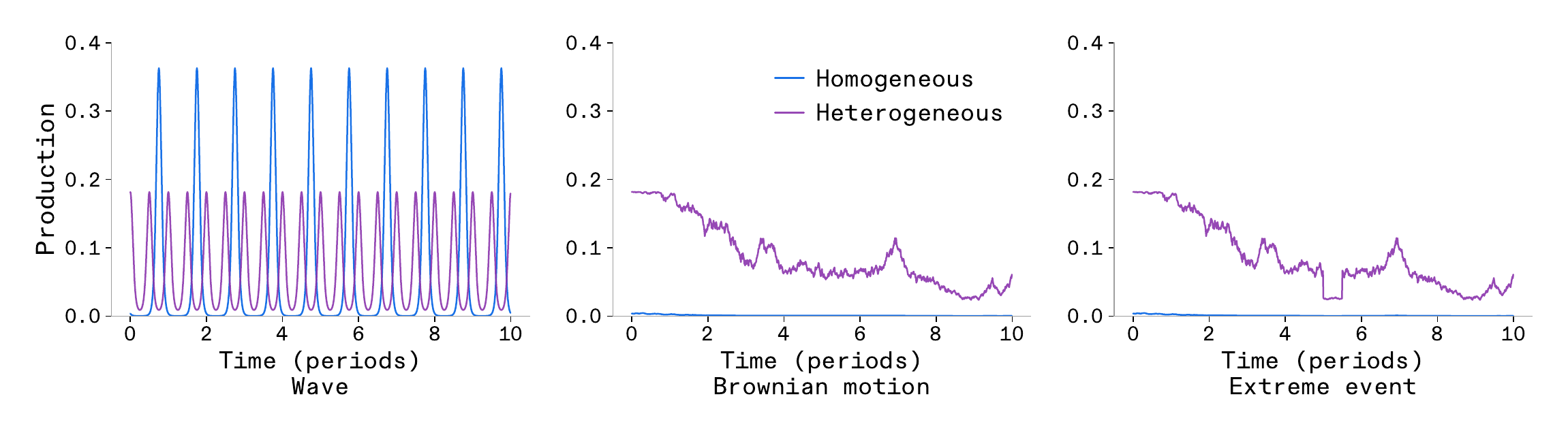}
    \caption{
        Example production time series for the unimodal catalogue (task~4)
        under wave, Brownian motion, and extreme-event regimes.
        Blue: homogeneous system; purple: heterogeneous system.
    }
    \label{fig:unimodal_example_timeseries}
\end{figure}
\begin{figure}[H]
    \centering
    \includegraphics[width=\linewidth]{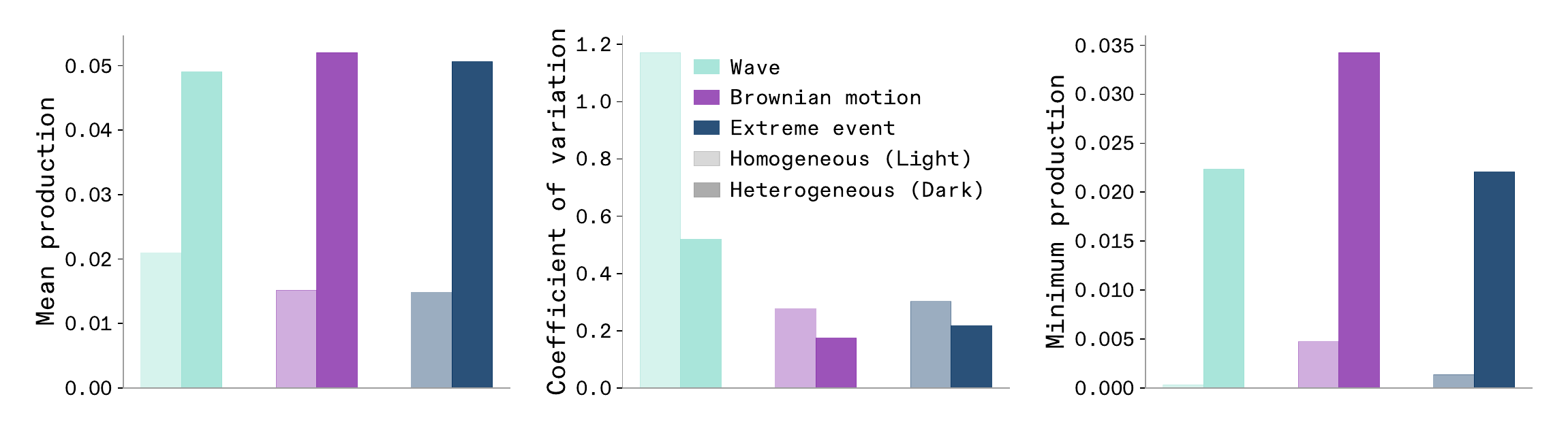}
    \caption{
        Summary robustness statistics (mean, CV, minimum) averaged across all
        $16$ unimodal tasks.
        Hatched bars: homogeneous; solid bars: heterogeneous.
        Colours indicate regime (blue: wave; purple: BM; dark blue: extreme).
    }
    \label{fig:unimodal_summary_robustness}
\end{figure}
\paragraph{Diverse catalogue.}
When the task catalogue includes structurally diverse demands—uniform
densities, bimodal and multimodal configurations with varying separations and
asymmetries—the heterogeneous advantage remains robust
(\cref{fig:diverse_summary_robustness}).
Under wave evolution, $16/16$ tasks favour the heterogeneous system on both
mean and minimum production ($p<10^{-4}$).
Under BM and extreme events, $14/16$ tasks show higher mean and minimum
production for the heterogeneous system.
The minimum-production advantage is highly significant across all regimes
($p<0.002$), confirming that the benefit is not an artifact of unimodal task
structure.
\begin{figure}[H]
    \centering
    \includegraphics[width=\linewidth]{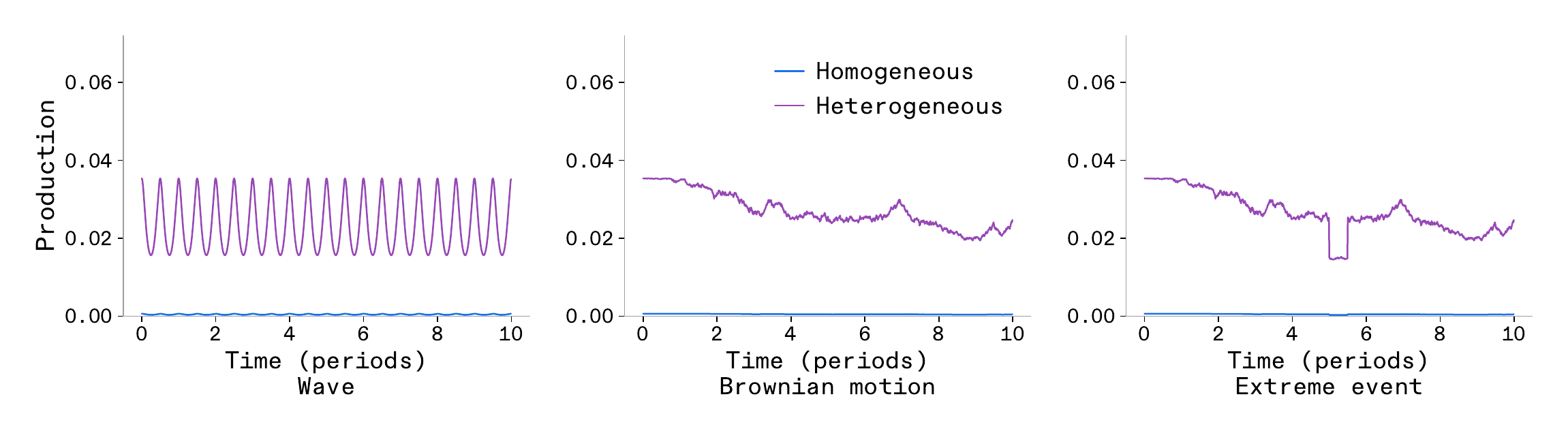}
    \caption{
        Example production time series for the diverse catalogue (task~4)
        under wave, Brownian motion, and extreme-event regimes.
    }
    \label{fig:diverse_example_timeseries}
\end{figure}
\begin{figure}[H]
    \centering
    \includegraphics[width=\linewidth]{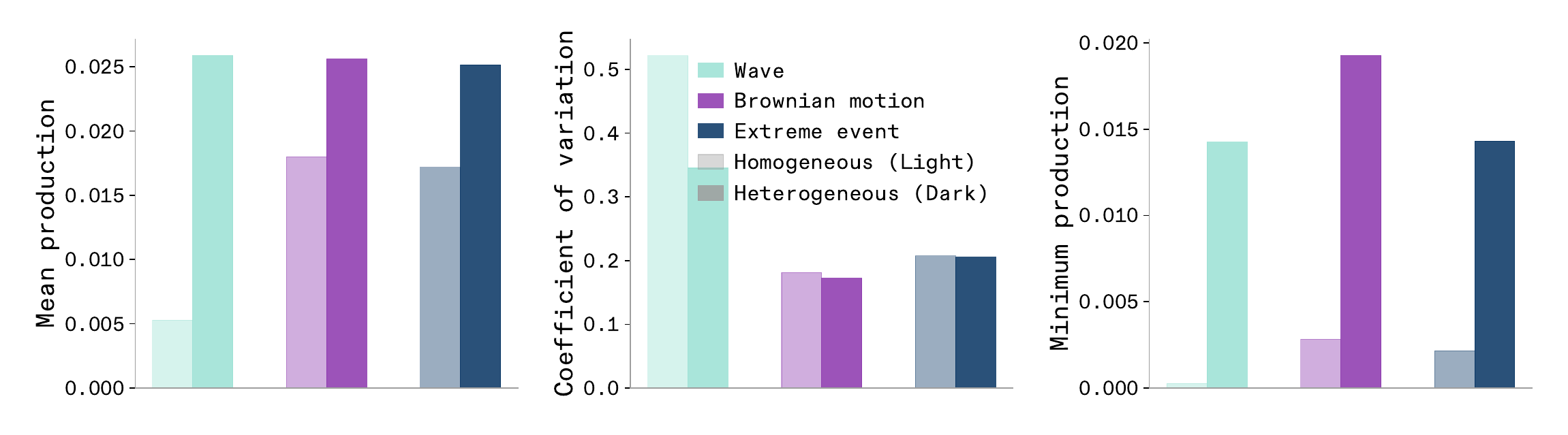}
    \caption{
        Summary robustness statistics for the diverse catalogue
        ($16$ tasks spanning uniform through quintmodal configurations).
    }
    \label{fig:diverse_summary_robustness}
\end{figure}

\paragraph{Homo-suited catalogue: robustness despite initial disadvantage.}
The most revealing experiment uses the homo-suited catalogue, where all tasks
are initially centred at $\mu=\pi$—the exact specialisation of the homogeneous
system—so that the homogeneous system begins with a large production advantage.
Under BM, the homogeneous system indeed achieves a $4\times$ higher mean
production ($0.069$ vs.\ $0.017$, $p=0.04$): the specialist excels when the
environment matches its narrow competence.
However, the heterogeneous system is significantly more \emph{stable}: its CV
is $41\%$ lower ($0.18$ vs.\ $0.31$, $p=0.014$), with $13/16$ tasks showing
lower variability for the heterogeneous system (Wilcoxon $p=0.009$).
Under extreme events, the pattern sharpens: the heterogeneous system maintains
a $44\%$ lower CV ($0.20$ vs.\ $0.36$, $p=0.004$) and a higher production
floor ($\min P=0.011$ vs.\ $0.008$), with $12/16$ tasks favouring the
heterogeneous system on CV (\cref{fig:homosuited_summary_robustness_2}).
The time series (\cref{fig:homosuited_example_timeseries}) illustrate this
trade-off vividly: the homogeneous system produces sharp peaks when the task
happens to align with its narrow skill, but collapses when it drifts away; the
heterogeneous system produces a lower but remarkably steadier output throughout.
\begin{figure}[H]
    \centering
    \includegraphics[width=\linewidth]{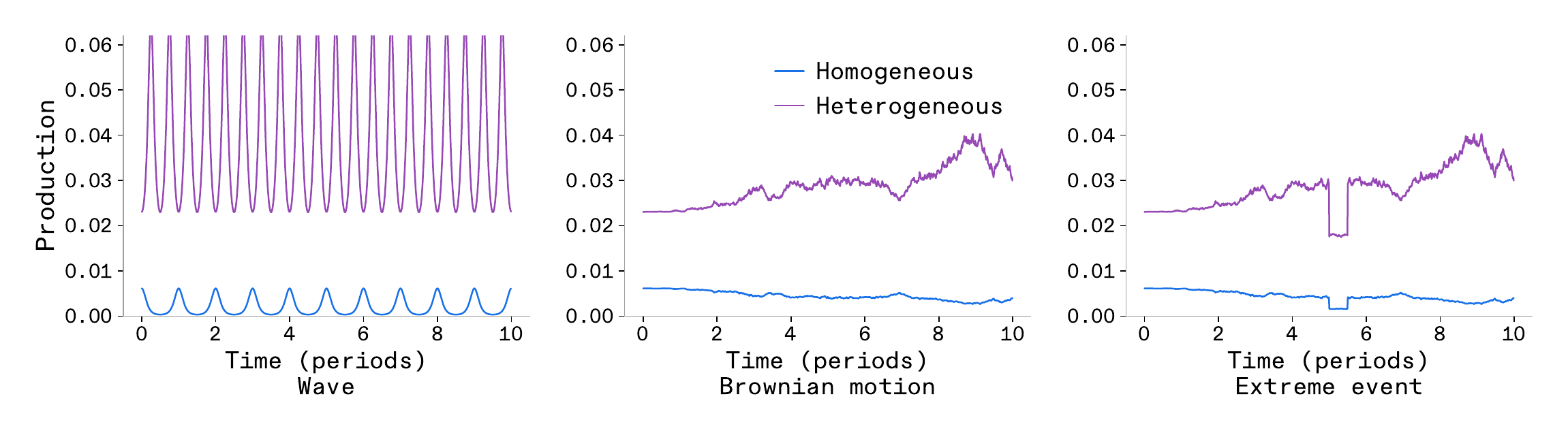}
    \caption{
        Example production time series for the homo-suited catalogue (task~4).
        The homogeneous system starts with high production (task initially at
        $\mu=\pi$) but suffers large swings; the heterogeneous system is
        steadier.
    }
    \label{fig:homosuited_example_timeseries}
\end{figure}
\begin{figure}[H]
    \centering
    \includegraphics[width=\linewidth]{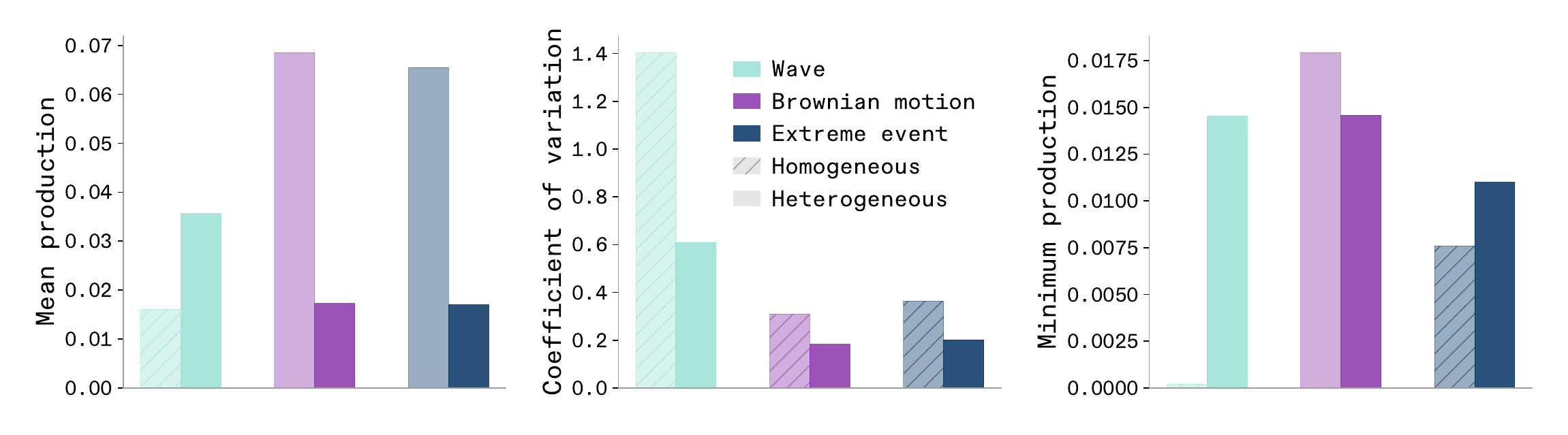}
    \caption{
        Summary robustness for the homo-suited catalogue.
        Despite lower mean production, the heterogeneous system exhibits
        significantly lower CV and higher minimum production under extreme
        events.
    }
    \label{fig:homosuited_summary_robustness_2}
\end{figure}